\documentclass{opt2020} 
\usepackage{url,balance}
\usepackage{multirow}
\usepackage{mathtools}
\usepackage{extarrows}
\usepackage{comment}
\usepackage{algorithm}
\usepackage{algpseudocode}
\usepackage{color,wrapfig,epsfig}

\input{mysymbol.sty}

\usepackage{amsmath,amsfonts,bm}

\def\1{\bm{1}}

\DeclareMathAlphabet{\mathsfit}{\encodingdefault}{\sfdefault}{m}{sl}
\SetMathAlphabet{\mathsfit}{bold}{\encodingdefault}{\sfdefault}{bx}{n}

\newcommand{\RR}{\mathbb{R}}
\newcommand{\EE}{\mathbb{E}}
\newcommand{\cL}{\mathcal{L}}

\newcommand{\ttheta}{\tilde{\theta}}
\newcommand{\tdelta}{\tilde{\delta}}

\newcommand{\dotp}[1]{\Big\langle #1 \Big\rangle}

\definecolor{myblue}{rgb}{0.158, 0.188, 0.478}
\definecolor{mygray}{cmyk}{0, 0.7808, 0.6429, 0.1412}

\newcommand{\red}[1]{{\color{mygray}#1}}
\newcommand{\blue}[1]{{\color{myblue}#1}}

\usepackage{hyperref}

\algrenewcommand{\algorithmiccomment}[1]{\hskip0em$\triangleright$ #1}
\algblock{ParFor}{EndParFor}
\algnewcommand\algorithmicparfor{\textbf{for}}
\algnewcommand\algorithmicpardo{\textbf{do in parallel}}
\algnewcommand\algorithmicendparfor{\textbf{end for}}
\algrenewtext{ParFor}[1]{\algorithmicparfor\ #1\ \algorithmicpardo}
\algrenewtext{EndParFor}{\algorithmicendparfor}

\title[CADA: Communication-Adaptive Distributed Adam]{CADA: Communication-Adaptive Distributed Adam}

\optauthor{\Name{Tianyi Chen} \Email{chent18@rpi.edu}\\
  \addr Rensselaer Polytechnic Institute\\
  \Name{Ziye Guo} \Email{guoz8@rpi.edu}\\
    \addr Rensselaer Polytechnic Institute\\
  \Name{Yuejiao Sun} \Email{sunyj@math.ucla.edu}\\
  \addr   University of California, Los Angeles\\
  \Name{Wotao Yin} \Email{	wotaoyin@math.ucla.edu}\\
  \addr   University of California, Los Angeles}

\begin{document}

\maketitle
 
\begin{abstract}%
Stochastic gradient descent (SGD) has taken the stage as the primary workhorse
for large-scale machine learning. It is often used with its adaptive variants such as AdaGrad, Adam, and AMSGrad. This paper proposes an adaptive stochastic gradient descent method for distributed machine learning, which can be viewed as the communication-adaptive counterpart of the celebrated Adam method --- justifying its name CADA. The key components of CADA are a set of new rules tailored for adaptive stochastic gradients that can be implemented to save communication upload. The new algorithms adaptively reuse the stale Adam gradients, thus saving communication, and still have convergence rates comparable to original Adam. In numerical experiments, CADA achieves impressive empirical performance in terms of total communication round reduction.
\end{abstract}


\section{Introduction}
Stochastic gradient descent (SGD) method \cite{robbins1951} is prevalent in solving large-scale machine learning problems during the last decades. 
Although simple to use, the plain-vanilla SGD is often sensitive to the choice of hyper-parameters and sometimes suffer from the slow convergence. 
Among various efforts to improve SGD, adaptive methods such as AdaGrad \cite{duchi2011adaptive}, Adam \cite{kingma2014adam} and AMSGrad \cite{reddi2019adam} have well-documented empirical performance, especially in  training deep neural networks. 

To achieve ``adaptivity," these algorithms adaptively adjust the \emph{update direction} or tune the \emph{learning rate}, or, the combination of both.  
While adaptive SGD methods have been mostly studied in the setting where data and computation are both centralized in a single node, their performance in the distributed learning setting is less understood. As this setting brings new challenges to machine learning, can we add an \emph{additional dimension of adaptivity} to Adam in this regime?

In this context, we aim to develop a fully adaptive SGD algorithm tailored for the distributed learning. 
We consider the setting composed of a central server and a set of $M$ workers in ${\cal M}:=\{1,\ldots,M\}$, where each worker $m$ has its local data $\xi_m$ from a distribution $\Xi_m$. Workers may have different data distributions $\{\Xi_m\}$, and they collaboratively solve the following problem
    \begin{align}\label{eqn: problem}
	\min_{\theta\in \mathbb{R}^p}~~~ {\cal L}(\theta)=\frac{1}{M}\sum_{m\in{\cal M}}{\cal L}_m(\theta) ~	~~{\rm with}~~~ {\cal L}_m(\theta):=\!\mathbb{E}_{\xi_m}\left[\ell(\theta;\xi_m)\right], ~m\in{\cal M} 
\end{align}
where $\theta\in \mathbb{R}^p$ is the sought variable and $\{{\cal L}_m, m\!\in\!{\cal M}\}$ are smooth (but not necessarily convex) functions.
We focus on the setting where local data $\xi_m$ at each worker $m$ can not be uploaded to the server, and collaboration is needed through communication between the server and workers. This setting often emerges due to the data privacy concerns, e.g., federated learning \cite{mcmahan2017,kairouz2019advances}.

To solve \eqref{eqn: problem}, we can in principle apply the single-node version of the adaptive SGD methods such as Adam \cite{kingma2014adam}: At iteration $k$, the server broadcasts $\theta^k$ to \emph{all} the workers; each worker $m$ computes {$\nabla  \ell(\theta^k; \xi_m^k)$} using a randomly selected sample or a minibatch of samples $\{\xi_m^k\}\sim\Xi_m$, and then uploads it to the server; and once receiving stochastic gradients from all workers, the server can simply use the aggregated stochastic gradient {$\bar{\bm\nabla}^k=\frac{1}{M}\!\sum_{m\in{\cal M}}\nabla  \ell(\theta^k; \xi_m^k)$} to update the parameter via the plain-vanilla single-node Adam. 
When $\nabla  \ell(\theta^k; \xi_m^k)$ is an unbiased gradient of $\cL_m(\theta)$, the convergence of this distributed implementation of Adam follows from the original ones \cite{reddi2019adam,chen2019adam}.
To implement this, however, \emph{all} the workers have to \emph{upload} the fresh {$\{\nabla \ell(\theta^k; \xi_m^k)\}$} at each iteration. 
This prevents the efficient implementation of Adam in scenarios where the communication uplink and downlink are not symmetric, and communication especially  upload from workers and the server is costly; e.g., cellular networks \cite{park2019wireless}.
Therefore, \emph{our goal} is to 
endow an additional dimension of adaptivity to Adam for solving the distributed problem \eqref{eqn: problem}. In short, on top of its adaptive learning rate and update direction, we want Adam to be communication-adaptive.  
\vspace{-0.15cm}

\subsection{Related work}
\vspace{-0.1cm}
To put our work in context, we review prior contributions that we group in two categories. 
 \vspace{-0.15cm}
 
\subsubsection{SGD with adaptive gradients}
\vspace{-0.15cm}
A variety of SGD variants have been developed recently, including momentum and acceleration \cite{polyak1964,nesterov1983method,ghadimi2016accelerated}. 
However, these methods are relatively sensitive to the hyper-parameters such as stepsizes, and require significant efforts on finding the optimal parameters.

\noindent\textbf{Adaptive learning rate.} 
One limitation of SGD is that it scales the gradient uniformly in all directions by a pre-determined constant or a sequence of constants (a.k.a. learning rates). 
This may lead to poor performance when the training data are sparse \cite{duchi2011adaptive}. To address this issue, adaptive learning rate methods have been developed that scale the gradient in an entry-wise manner by using past gradients, which include AdaGrad \cite{duchi2011adaptive,ward2019adagrad}, AdaDelta \cite{zeiler2012adadelta} and other variants \cite{li2019adapt}. 
This simple technique has improved the performance of SGD in some scenarios.

\noindent\textbf{Adaptive SGD.} 
Adaptive SGD methods achieve the best of both worlds, which update the search directions and the learning rates simultaneously using past gradients. 
Adam \cite{kingma2014adam} and AMSGrad \cite{reddi2019adam} are the representative ones in this category. While these methods are simple-to-use, analyzing their convergence is challenging \cite{reddi2019adam,wang2020sadam}. Their convergence in the nonconvex setting has been settled only recently \cite{chen2019adam,defossez2020convergence}. 
However, most adaptive SGD methods are studied in the single-node setting where data and computation are both centralized. Very recently, adaptive SGD has been studied in the shared memory setting \cite{xu2020adam}, where data is still centralized and communication is not adaptive. 
 

\subsubsection{Communication-efficient distributed optimization}
 
Popular communication-efficient distributed learning methods 
belong to two categories: c1) reduce the number of bits per communication round; and, c2) save the number of communication rounds.

For c1), methods are centered around the ideas of \emph{quantization} and \emph{sparsification}.\\
\noindent\textbf{Reducing communication bits.}
Quantization has been successfully applied to distributed machine learning. The 1-bit and multi-bits quantization methods have been developed in \cite{seide20141,alistarh2017qsgd,tang2018communication}. More recently, signSGD with majority vote has been developed in \cite{bernstein2018icml}.
Other advances of quantized gradient schemes include error compensation \cite{wu2018error,karimireddy2019icml}, variance-reduced quantization \cite{zhang2017zipml,horvath2019stochastic}, and quantization to a ternary vector \cite{wen2017terngrad,reisizadeh2019nips}. 
Sparsification amounts to transmitting only gradient coordinates with large enough magnitudes exceeding a certain threshold~\cite{strom2015scalable, aji2017sparse}.
To avoid losing information of skipping communication, small gradient components will be accumulated and transmitted when they are large enough \cite{lin2017deep,stich2018nips,alistarh2018,wangni2018gradient,jiang2018linear,tang2020apmsqueeze}. Other compression methods also include low-rank approximation \cite{vogels2019powersgd} and sketching \cite{ivkin2019communication}. 
However, all these methods aim to resolve c1).
In some cases, other latencies dominate the bandwidth-dependent transmission latency. This motivates c2).

\noindent\textbf{Reducing communication rounds.}
One of the most popular techniques in this category is the periodic averaging, e.g., elastic averaging
SGD \cite{zhang2015nips}, local SGD (a.k.a. FedAvg)  \cite{mcmahan2017communication,lin2018don,kamp2018,stich2019local,wang2018coop,karimireddy2019scaffold,haddadpour2019local} or local momentum SGD \cite{yu2019lcml,wang2020iclr}. 
In local SGD, workers perform local model updates independently and the models are averaged periodically. Therefore, communication frequency is reduced. 
However, except \cite{kamp2018,wang2018coop,haddadpour2019local}, most of local SGD methods follow a pre-determined communication schedule that is nonadaptive. Some of them are tailored for the \emph{homogeneous} settings, where the data are independent and identically distributed over all workers. 
To tackle the heterogeneous case, FedProx has been developed in \cite{li2018federated} by solving local subproblems. 
For learning tasks where the loss function is convex and its conjugate dual is
expressible, the dual coordinate ascent-based approaches have been demonstrated to yield impressive
empirical performance \cite{jaggi2014,ma2017}. 
Higher-order methods have also been considered \cite{shamir2014communication,zhang2015icml}.
Roughly speaking, algorithms in
\cite{li2018federated,jaggi2014,ma2017,shamir2014communication,zhang2015icml} reduce
communication by increasing local gradient computation.

The most related line of work to this paper is the lazily aggregated gradient (LAG) approach \cite{chen2018lag,sun2019}. 
In contrast to periodic communication, the communication in LAG is adaptive and tailored for the \emph{heterogeneous} settings.
Parameters in LAG are updated at the server, and workers only adaptively upload information that is determined to be informative enough. Unfortunately, while LAG has good performance in the deterministic settings (e.g., with full gradient), as shown in Section \ref{subsec.lag}, its performance will be significantly degraded in the stochastic settings \cite{chen2020lasg}. In contrast, our approach generalizes LAG to the regime of running adaptive SGD. 
Very recently, FedAvg with local adaptive SGD update has been proposed in \cite{reddi2020adaptive}, which sets a strong benchmark for communication-efficient learning. When the new algorithm in \cite{reddi2020adaptive} achieves the sweet spot between local SGD and adaptive momentum SGD, the proposed algorithm is very different from ours, and the \emph{averaging period} and the selection of \emph{participating workers} are nonadaptive. 

\subsection{Our approach}
 
We develop a new adaptive SGD algorithm for distributed learning, called \textbf{C}ommunication-\textbf{A}daptive \textbf{D}istributed \textbf{A}dam (\textbf{CADA}). 
Akin to the dynamic scaling of every gradient coordinate in Adam, the key motivation of adaptive communication is that during distributed learning, not all communication rounds between the server and workers are equally important. So a natural solution is to use a condition that decides whether the communication is important or not, and then adjust the frequency of communication between a worker and the server. If some workers are not communicating, the server uses their stale information instead of the fresh ones. We will show that this adaptive communication technique can reduce the less informative communication of distributed Adam.

Analogous to the original Adam \cite{kingma2014adam} and its modified version AMSGrad \cite{reddi2019adam}, our new CADA approach also uses the exponentially weighted stochastic gradient $h^{k+1}$ as the update direction of $\theta^{k+1}$, and leverages the weighted stochastic gradient magnitude $v^{k+1}$ to inversely scale the update direction $h^{k+1}$. 
Different from the direct distributed implementation of Adam that incorporates the fresh (thus unbiased) stochastic gradients {$\bar{\bm\nabla}^k=\frac{1}{M}\!\sum_{m\in{\cal M}}\!\!\nabla  \ell(\theta^k; \xi_m^k)$}, CADA exponentially combines the aggregated stale stochastic gradients {$\bm\nabla^k=\frac{1}{M}\!\sum_{m\in{\cal M}}\!\!\nabla  \ell(\hat\theta_m^k; \hat\xi_m^k)$}, where {$\nabla \ell(\hat\theta_m^k; \hat\xi_m^k)$} is either the fresh stochastic gradient {$\nabla \ell(\theta^k; \xi_m^k)$}, or an old copy when {$\hat\theta_m^k\neq \theta^k; \hat\xi_m^k\neq \xi_m^k$}. 
 Informally, with $\alpha_k>0$ denoting the stepsize at iteration $k$, CADA has the following update
\begin{subequations}\label{eqn:CADA1}
\begin{align}
\!\!\!h^{k\!+\!1}\!&=\!\beta_1 h^k\!+\!(1\!-\!\beta_1)\bm\nabla^k\!,~{\rm with}~
	\bm\nabla^k\!=\!\frac{1}{M}\!\sum_{m\in{\cal M}}\!\!\nabla  \ell(\hat\theta_m^k; \hat\xi_m^k)  \label{eq.CADA-1}\\
\!	v^{k+1}&=\beta_2 \hat v^k+(1-\beta_2) (\bm\nabla^k)^2\label{eq.CADA-2}\\
\!		\theta^{k+1}&=\theta^k-\alpha_k(\epsilon I+\hat V^{k+1})^{-\frac{1}{2}}h^{k+1} \label{eq.CADA-3}
\end{align}
\end{subequations} 
where {$\beta_1, \beta_2>0$} are the momentum weights, {$\hat V^{k+1}:=\diag(\hat v^{k+1})$} is a diagonal matrix whose diagonal vector is {$\hat v^{k+1}:=\max\{v^{k+1}, \hat v^k\}$}, the constant is {$\epsilon>0$}, and {$I$} is an identity matrix. To reduce the 
memory requirement of storing all the stale stochastic gradients {$\{\nabla \ell(\theta^k; \xi_m^k)\}$}, we can obtain $\bm\nabla^k$ by refining the previous aggregated stochastic gradients $\bm\nabla^{k-1}$ stored in the server via
\begin{equation}
	\bm\nabla^k=\bm\nabla^{k-1}+\frac{1}{M}\!\sum\limits_{m\in{\cal M}^k}\delta_m^k
\end{equation}
where {$\delta_m^k := \nabla \ell(\theta^k;\xi_m^k)-\nabla \ell(\hat\theta_m^k;\hat\xi_m^k)$} is the stochastic gradient innovation, and ${\cal M}^k$ is the set of workers that upload the stochastic gradient to the server at iteration $k$. Henceforth, {$\hat\theta_m^k=\theta^k; \hat\xi_m^k=\xi_m^k,\,\forall m\in {\cal M}^k$} and {$\hat\theta_m^k=\hat\theta_m^{k-1}; \hat\xi_m^k=\hat\xi_m^{k-1},\,\forall m\notin {\cal M}^k$}. See CADA's implementation in Figure \ref{fig:CADA-diag}.

Clearly, the selection of subset ${\cal M}^k$ is both critical and challenging. 
It is critical because it adaptively determines the number of communication rounds per iteration $|{\cal M}^k|$. 
However, it is challenging since 1) the staleness introduced in the Adam update will propagate not only through the momentum gradients but also the adaptive learning rate; 
2) the importance of each communication round is dynamic, thus a fixed or nonadaptive condition is ineffective; and 3) the condition needs to be checked efficiently without extra overhead. 
To overcome these challenges, we develop two adaptive conditions to select ${\cal M}^k$ in CADA. 

With details deferred to Section \ref{sec.adam}, the contributions of this paper are listed as follows.

{\bf c1)} We introduce a novel communication-adaptive distributed Adam (CADA) approach that
reuses stale stochastic gradients to reduce communication for distributed implementation of Adam.

{\bf c2)} We develop a new Lyapunov function to establish convergence of CADA under both the nonconvex and Polyak-{\L}ojasiewicz (PL) conditions even when the datasets are non-i.i.d. across workers. The convergence rate matches that of the original Adam. 

{\bf c3)} We confirm that our novel fully-adaptive CADA
algorithms achieve at least $60\%$ performance gains in terms of communication upload over some popular alternatives using numerical tests on logistic regression and neural network training.

\begin{figure}[t]
\hspace{-0.2cm}
\def\epsfsize#1#2{0.6#1}
\centerline{\epsffile{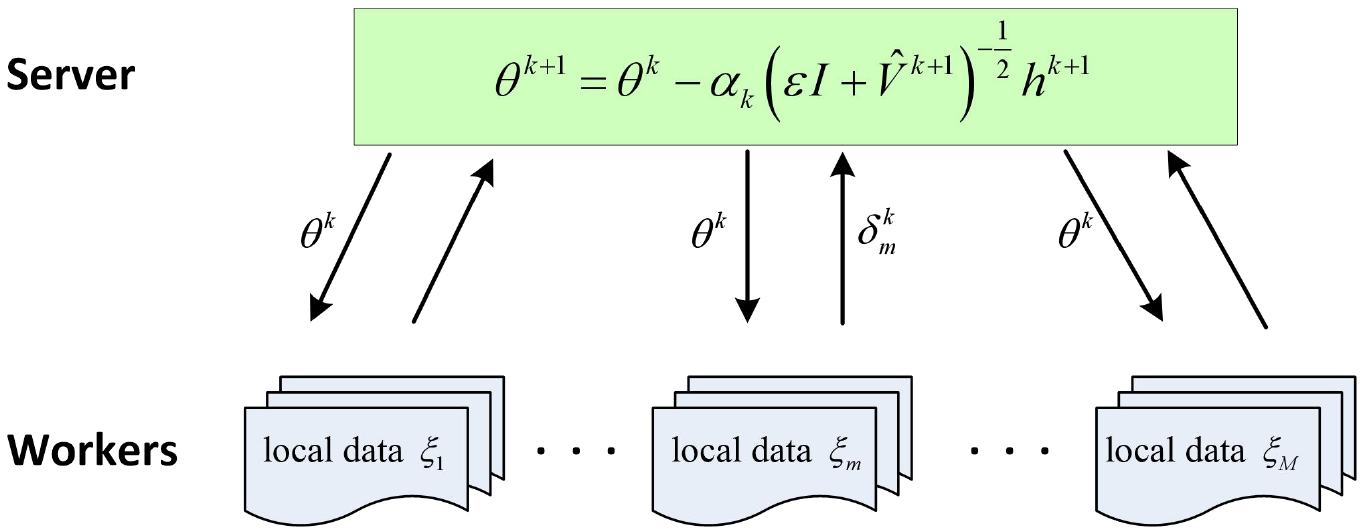}}
\vspace*{-16pt}
  \caption{The CADA implementation.}
\label{fig:CADA-diag}
\end{figure}

\section{CADA: Communication-Adaptive Distributed Adam}\label{sec.adam} 
In this section, we revisit the recent LAG method  \cite{chen2018lag} and provide insights why it does not work well in stochastic settings, and then develop our communication-adaptive distributed Adam approach.  
To be more precise in our notations, we henceforth use {$\tau_m^k\geq 0$} for the \emph{staleness or age of information} from worker $m$ used by the server at iteration $k$, e.g., {$\hat\theta_m^k=\theta^{k-\tau_m^k}$}. An age of $0$ means ``fresh."
\vspace{-0.1cm}

\subsection{The ineffectiveness of LAG with stochastic gradients}\label{subsec.lag}
\vspace{-0.1cm}

The LAG method \cite{chen2018lag} modifies the distributed gradient descent update.
Instead of communicating with all workers per iteration, LAG selects the subset of workers {${\cal M}^k$} to obtain \emph{fresh} full gradients and reuses stale full gradients from others, that is 
\begin{equation}
\theta^{k+1}=\theta^k-\frac{\eta_k}{M}\!\!\sum_{m\in{\cal M}\backslash{\cal M}^k}\!\!\!\!\nabla {\cal L}_m(\theta^{k-\tau_m^k})-\frac{\eta_k}{M}\!\sum_{m\in{\cal M}^k}\!\! \nabla{\cal L}_m(\theta^k)	
\end{equation}
where {${\cal M}^k$} is adaptively decided by comparing the gradient difference {$\|\nabla{\cal L}_m(\theta^k)-\nabla {\cal L}_m(\theta^{k-\tau_m^k})\|$}.
Following this principle, the direct (or ``naive") stochastic version of LAG selects the subset of workers {${\cal M}^k$} to obtain \emph{fresh} stochastic gradients {$\nabla{\cal L}_m\big(\theta^k;\xi_m^k\big)$, $m\in{\cal M}^k$}.
The \textbf{stochastic LAG} also follows the distributed SGD update, but it selects {${\cal M}^k$} by: if worker $m$ finds the innovation of the fresh stochastic gradient {$\nabla \ell(\theta^k;\xi_m^k)$} is small such that it satisfies 
\begin{align}\label{eqn:LAGrule}
\Big\|\nabla \ell(\theta^k;\xi_m^k)&-\nabla \ell(\theta^{k-\tau_m^k};\xi_m^{k-\tau_m^k})\Big\|^2  \leq \frac{c}{d_{\rm max}}\sum\limits_{d=1}^{d_{\rm max}}\big\|\theta^{k+1-d}-\theta^{k-d}\big\|^2\! 
\end{align}
where {$c\geq0$ and $d_{\rm max}$} are pre-fixed constants, then worker $m$ reuses the old gradient, {$m\in{\cal M}\backslash{\cal M}^k$}, and sets the staleness {$\tau_m^{k+1} = \tau_m^{k} + 1$}; otherwise, worker $m$ uploads the fresh gradient, and sets {$\tau_m^{k+1}=1$}.

If the stochastic gradients were full gradients, LAG condition \eqref{eqn:LAGrule} \emph{compares the error induced by using the stale gradients and the progress of the distributed gradient descent algorithm}, which has proved to be effective in skipping redundant communication \cite{chen2018lag}.
Nevertheless, the observation here is that the two stochastic gradients \eqref{eqn:LAGrule} are evaluated on not just two different iterates {($\theta^k$ and $\theta^{k-\tau_m^k}$)} but also two different samples {($\xi_m^k$ and $\xi_m^{k-\tau_m^k}$)} thus two different loss functions. 

This subtle difference leads to the ineffectiveness of \eqref{eqn:LAGrule}.
We can see this by expanding the left-hand-side (LHS) of \eqref{eqn:LAGrule} by (see details in supplemental material)
\begin{subequations}\label{eqn:variance-wk}
\begin{align}
\! \EE\left[\|\nabla\ell(\theta^k;\xi_m^k)-\nabla\ell(\theta^{k-\tau_m^k};\xi_m^{k-\tau_m^k})\|^2\right]   &\geq \frac{1}{2}\EE\Big[\big\|\nabla\ell(\theta^k;\xi_m^k)-\nabla{\cal L}_m(\theta^k)\big\|^2\Big]\label{eqn:wk-rq1}\\
\!&+ \frac{1}{2}\EE\Big[\big[\big\|\nabla\ell(\theta^{k-\tau_m^k};\xi_m^{k-\tau_m^k})-\nabla{\cal L}_m(\theta^{k-\tau_m^k})\big\|^2\big]\Big]\label{eqn:wk-rq2}\\
\!&- \EE[\|\nabla{\cal L}_m(\theta^k)-\nabla{\cal L}_m(\theta^{k-\tau_m^k})\|^2]. \label{eqn:wk-rq3} 
\end{align}
\end{subequations}
\hspace{-3pt}Even if $\theta^k$ converges, e.g., $\theta^k\rightarrow \theta^*$, and thus the right-hand-side (RHS) of \eqref{eqn:LAGrule} $\big\|\theta^{k+1-d}\!-\!\theta^{k-d}\big\|^2\!\rightarrow\! 0$, the LHS of \eqref{eqn:LAGrule} does not, because the variance inherited in \eqref{eqn:wk-rq1} and \eqref{eqn:wk-rq2} does not vanish yet the gradient difference at the same function \eqref{eqn:wk-rq3} diminishes. 
Therefore, the key insight here is that the non-diminishing variance of stochastic gradients makes the LAG rule \eqref{eqn:LAGrule} ineffective eventually. This will also be verified in our simulations when we compare CADA with stochastic LAG. 
\vspace{-0.1cm}

\subsection{Algorithm development of CADA}\label{sec.CADA}
\vspace{-0.1cm}
In this section, we formally develop our CADA method, and present the intuition
behind its design. 

The key of the CADA design is to \emph{reduce the variance of the innovation measure} in the adaptive condition. 
We introduce two CADA variants, both of which follow the update \eqref{eqn:CADA1}, but they differ in the variance-reduced communication rules.

The first one termed \textbf{CADA1} will calculate two stochastic gradient innovations with one {$\tdelta_m^k := \nabla \ell(\theta^k;\xi_m^k)-\nabla \ell(\ttheta;\xi_m^k)$} at the sample {$\xi_m^k$}, and one {$\tdelta_m^{k-\tau_m^k} := \nabla \ell(\theta^{k-\tau_m^k};\xi_m^{k-\tau_m^k})-\nabla \ell(\ttheta;\xi_m^{k-\tau_m^k})$} at the sample {$\xi_m^{k-\tau_m^k}$}, where $\ttheta$ is a snapshot of the previous iterate $\theta$ that will be updated every $D$ iterations. 
As we will show in \eqref{eq.diff_vr}, $\tdelta_m^k-\tdelta_m^{k-\tau_m^k}$ can be viewed as the difference of two variance-reduced gradients calculated at $\theta^k$ and $\theta^{k-\tau_m^k}$. 
Using $\tdelta_m^k-\tdelta_m^{k-\tau_m^k}$ as the error induced by using stale information, CADA1 will exclude worker $m$ from {${\cal M}^k$} if worker $m$ finds 
\begin{equation}\label{eqn:workerrule1}
\!\!\left\|\tdelta_m^k-\tdelta_m^{k-\tau_m^k}\right\|^2\leq \frac{c}{d_{\rm max}}\sum\limits_{d=1}^{d_{\rm max}}\left\|\theta^{k+1-d}-\theta^{k-d}\right\|^2.\!
\end{equation}
In \eqref{eqn:workerrule1}, we use the change of parameter $\theta^k$ averaged over the past $d_{\rm max}$ consecutive iterations to measure the progress of algorithm. Intuitively, if \eqref{eqn:workerrule1} is satisfied, the error induced by using stale information will not large affect the learning algorithm. 
In this case, worker $m$ does not upload, and the staleness of information from worker $m$ increases by {$\tau_m^{k+1} = \tau_m^{k} + 1$}; otherwise, worker $m$ belongs to {${\cal M}^k$}, uploads the stochastic gradient innovation $\delta_m^k$, and resets {$\tau_m^{k+1}=1$}.

\textbf{The rationale of CADA1.}
In contrast to the non-vanishing variance in LAG rule (see \eqref{eqn:variance-wk}),
the CADA1 rule \eqref{eqn:workerrule1} reduces its inherent variance.
To see this, we can decompose the LHS of \eqref{eqn:workerrule1} as the difference of two \emph{variance reduced} stochastic gradients at iteration $k$ and $k-\tau_m^k$.
Using the stochastic gradient in SVRG as an example \cite{johnson2013accelerating},
the innovation can be written as
\begin{align}\label{eq.diff_vr}
\tdelta_m^k-\tdelta_m^{k-\tau_m^k}&= \big(\nabla\ell(\theta^k;\xi_m^k)-\nabla\ell(\tilde{\theta};\xi_m^k)+\nabla{\cal L}_m(\tilde{\theta})\big)\nonumber\\
&-\left(\nabla\ell(\theta^{k-\tau_m^k};\xi_m^{k-\tau_m^k})-\nabla\ell(\tilde{\theta};\xi_m^{k-\tau_m^k})+\nabla{\cal L}_m(\tilde{\theta})\right). 
\end{align}
Define the minimizer of \eqref{eqn: problem} as $\theta^{\star}$. 
With derivations given in the supplementary document, the expectation of the LHS of \eqref{eqn:workerrule1} can be \emph{upper-bounded} by
\begin{align}\label{eqn:variance-wk1}
&\!\EE\left[\big\|\tdelta_m^k-\tdelta_m^{k-\tau_m^k}\big\|^2\right]={\cal O}\Big(\EE[{\cal L}(\theta^k)]-{\cal L}(\theta^{\star})+\EE[{\cal L}(\theta^{k-\tau_m^k})]-{\cal L}(\theta^{\star}) +\EE[{\cal L}(\ttheta)]-{\cal L}(\theta^{\star})\Big).
\end{align}
If $\theta^k$ converges, e.g., $\theta^k, \theta^{k-\tau_m^k}, \ttheta\rightarrow \theta^*$, the RHS of \eqref{eqn:variance-wk1} diminishes, and thus the LHS of \eqref{eqn:workerrule1} diminishes. This is in contrast to the LAG rule \eqref{eqn:variance-wk} \emph{lower-bounded} by a non-vanishing value. Notice that while enjoying the benefit of variance reduction, our communication rule does not need to repeatedly calculate the full gradient $\nabla{\cal L}_m(\tilde{\theta})$, which is only used for illustration purpose.

  \begin{algorithm}[t]
\caption{Pseudo-code of CADA; \colorbox{red!30}{red lines} are run only by \red{\bf CADA1}; \colorbox{blue!30}{blue lines} are implemented only by \blue{\bf CADA2}; not both at the same time.}\label{alg: CADA}
    \begin{algorithmic}[1]
        \State{\textbf{Input:} delay counter $\{\tau_m^0\}$, stepsize $\alpha_k$, constant threshold $c$, max delay $D$.}
        \For{$k=0,1,\ldots, K-1$}
            \State{Server broadcasts $\theta^k$ to all workers.}
            \State{ 
             \colorbox{red!30}{All workers set {\small$\ttheta=\theta^k$} if {\small$k\!\!\mod\! D\!=\!0$}.}}
            \ParFor{ Worker $m=1,2,\ldots, M$}
            \State{\colorbox{red!30}{Compute $\nabla \ell(\theta^k;\xi_m^k)$ and $\nabla \ell(\ttheta;\xi_m^k)$.}}
            \State{\colorbox{red!30}{Check condition \eqref{eqn:workerrule1} with stored $\tdelta_m^{k-\tau_m^{k}}$.}}
            \State{\colorbox{blue!30}{Compute $\nabla \ell(\theta^k;\xi_m^k)$ and $\nabla \ell(\theta^{k-\tau_m^k}_m;\xi_m^k)$.}}
            \State{\colorbox{blue!30}{Check condition \eqref{eqn:workerrule2}.}}
                \If{\eqref{eqn:workerrule1} or \eqref{eqn:workerrule2} is violated or $\tau_m^k\geq D$}
                \State{Upload $\delta_m^k$. \qquad\qquad\qquad\Comment{$\tau_m^{k+1}=1$}}
                \Else
                \State{Upload nothing.}\qquad\Comment{$\tau_m^{k+1} = \tau_m^{k} + 1$}
                \EndIf
            \EndParFor
            \State{Server updates $\{h^k, v^k\}$ via \eqref{eq.CADA-1}-\eqref{eq.CADA-2}.}
            \State{Server updates $\theta^k$ via \eqref{eq.CADA-3}.}
        \EndFor
    \end{algorithmic}
\end{algorithm}

In addition to \eqref{eqn:workerrule1}, the second rule is termed \textbf{CADA2}. The key difference relative to CADA1 is that CADA2 uses $\nabla \ell(\theta^k;\xi_m^k)-\nabla \ell(\theta^{k-\tau_m^k}_m;\xi_m^k)$ to estimate the error of using stale information. 
CADA2 will reuse the stale stochastic gradient {$\nabla \ell(\theta^{k-\tau_m^k}_m;\xi_m^{k-\tau_m^k})$} or exclude worker $m$ from {${\cal M}^k$} if worker $m$ finds 
\begin{align}\label{eqn:workerrule2}
 \big\|\nabla \ell(\theta^k;\xi_m^k) -\nabla \ell(\theta^{k-\tau_m^k}_m;\xi_m^k)\big\|^2   \leq \frac{c}{d_{\rm max}}\sum\limits_{d=1}^{d_{\rm max}}\big\|\theta^{k+1-d} -\theta^{k-d}\big\|^2.
\end{align}
If \eqref{eqn:workerrule2} is satisfied, then worker $m$ does not upload, and the staleness increases by $\tau_m^{k+1} = \tau_m^{k} + 1$; otherwise, worker $m$ uploads the stochastic gradient innovation $\delta_m^k$, and resets the staleness as $\tau_m^{k+1}=1$. Notice that different from the naive LAG \eqref{eqn:LAGrule}, the CADA condition \eqref{eqn:workerrule2} is evaluated at two different iterates but on the same sample $\xi_m^k$.

\textbf{The rationale of CADA2.}
Similar to CADA1, the CADA2 rule \eqref{eqn:workerrule2} also reduces its inherent variance, since the LHS of \eqref{eqn:workerrule2} can be written as the difference between a \emph{variance reduced} stochastic gradient and a \emph{deterministic} gradient, that is
\begin{align}\label{eqn:variance-wk2}
 \nabla\ell(\theta^k;\xi_m^k)-\nabla\ell(\theta^{k-\tau_m^k};\xi_m^k)=\!\left(\nabla\ell(\theta^k;\xi_m^k)-\nabla\ell(\theta^{k-\tau_m^k};\xi_m^k)+\nabla{\cal L}_m(\theta^{k-\tau_m^k})\right) \nabla{\cal L}_m(\theta^{k-\tau_m^k}).
\end{align}
With derivations deferred to the supplementary document, similar to \eqref{eqn:variance-wk1} we also have that
{$\EE[\|\nabla\ell(\theta^k;\xi_m^k)-\nabla\ell(\theta^{k-\tau_m^k};\xi_m^k)\|^2]\rightarrow 0$} as the iterate $\theta^k\rightarrow \theta^{\star}$.

For either \eqref{eqn:workerrule1} or \eqref{eqn:workerrule2}, worker $m$ can check it locally with small memory cost by recursively updating the RHS of \eqref{eqn:workerrule1} or \eqref{eqn:workerrule2}. In addition, worker $m$ will update the stochastic gradient if the staleness satisfies {$\tau_m^k\geq D$}. We summarize CADA in Algorithm \ref{alg: CADA}. 

\textbf{Computational and memory cost of CADA.}
In CADA, checking \eqref{eqn:workerrule1} and \eqref{eqn:workerrule2} will double the computational cost (gradient evaluation) per iteration. 
Aware of this fact, we have compared the number of iterations and gradient evaluations in simulations (see Figures \ref{fig:covtype}-\ref{fig:NNcifar10}), which will demonstrate that CADA requires \emph{fewer} iterations and also \emph{fewer} gradient queries to achieve a target accuracy. Thus the extra computation is small. 
In addition, the extra memory for large $d_{\rm max}$ is low. To compute the RHS of \eqref{eqn:workerrule1} or \eqref{eqn:workerrule2}, each worker only stores the norm of model changes (\textbf{$d_{\rm max}$ scalars}).

\section{Convergence Analysis of CADA}
We present the convergence results of CADA. 
For all the results, we make some basic assumptions.\begin{assumption}\label{assump:smoothness}
The loss function $\cL(\theta)$ is smooth with the constant $L$.
\end{assumption}

\begin{assumption}\label{assump:gradientestimator}
Samples {$\xi_m^1,\xi_m^2,\ldots$} are independent, and the stochastic gradient {$\nabla\ell(\theta;\xi_m^k)$} satisfies {$\EE_{\xi_m^k} [\nabla\ell(\theta;\xi_m^k)] = \nabla\cL_m(\theta)$ and $\|\nabla\ell(\theta;\xi_m^k)\|\leq\sigma_m$}.
\end{assumption}
\vspace{-0.1cm}

Note that Assumptions \ref{assump:smoothness}-\ref{assump:gradientestimator} are standard in analyzing Adam and its variants  \cite{kingma2014adam,reddi2019adam,chen2019adam,xu2020adam}.
\vspace{-0.1cm}

\subsection{Key steps of Lyapunov analysis}\label{subsec.Lya-analysis}
\vspace{-0.1cm}
The convergence results of CADA critically builds on the subsequent Lyapunov analysis. 
We will start with analyzing the expected descent in terms of $\cL(\theta^k)$ by applying one step CADA update. 
\begin{lemma}\label{lemma:lossdescent}
Under Assumptions \ref{assump:smoothness} and \ref{assump:gradientestimator}, if $\alpha_{k+1}\leq \alpha_k$, then $\{\theta^k\}$ generated by CADA satisfy 
\begin{align}\label{eqn:lossdescent}
    \EE[\cL(\theta^{k+1})]-\EE[\cL(\theta^k)] 
   &\leq -\alpha_k(1-\beta_1)\EE\left[\dotp{\nabla\cL(\theta^k),(\epsilon I+\hat V^{k-D})^{-\frac{1}{2}}\bm\nabla^k}\right]\nonumber\\
     &-\alpha_k\beta_1\EE\left[ \dotp{\nabla\cL(\theta^{k-1}),(\epsilon I+\hat V^k)^{-\frac{1}{2}}h^k}\right]\nonumber\\
     &+\left(\frac{L}{2}+ \beta_1 L\right)\EE\left[\|\theta^{k+1}-\theta^{k}\|^2\right]\nonumber\\
     &+\alpha_k(2\!-\!\beta_1)\sigma^2\EE\Big[\sum_{i=1}^p\Big((\epsilon\!+\!\hat v_i^{k-D})^{\!-\!\frac{1}{2}}-(\epsilon\!+\!\hat v_i^{k+1})^{-\frac{1}{2}}\Big)\Big]
\end{align}
where $p$ is the dimension of $\theta$, $\sigma$ is defined as $\sigma:=\frac{1}{M}\sum_{m\in{\cal M}} \sigma_m$, and $\beta_1, \epsilon$ are parameters in \eqref{eqn:CADA1}.
\end{lemma}

Lemma \ref{lemma:lossdescent} contains four terms in the RHS of \eqref{eqn:lossdescent}: the first two terms quantify the correlations between the gradient direction $\nabla\cL(\theta^k)$ and the \emph{stale} stochastic gradient $\bm\nabla^k$ as well as the \emph{state momentum} stochastic gradient $h^k$; the third term captures the drift of two consecutive iterates; and, the last term estimates the maximum drift of the adaptive stepsizes over $D+1$ iterations. 
 
 From Lemma \ref{lemma:lossdescent}, analyzing the progress of $\cL(\theta^k)$ under CADA is challenging especially when the effects of staleness and the momentum couple with each other. 
Because the the state momentum gradient $h^k$ is recursively updated by $\bm\nabla^k$, we will first need the following lemma to characterize the regularity of the stale aggregated stochastic gradients $\bm\nabla^k$, which lays the theoretical foundation for incorporating the properly controlled staleness into the Adam's momentum update. 
\begin{lemma}\label{lemma6}
Under Assumptions \ref{assump:smoothness} and \ref{assump:gradientestimator}, if the stepsizes satisfy $\alpha_{k+1}\leq \alpha_k\leq 1/L$, then we have
\begin{align}\label{eqn:lemma6}
 - \alpha_k\EE\bigg[\dotp{\nabla\cL(\theta^k),(\epsilon I+\hat V^{k-D})^{-\frac{1}{2}}\bm\nabla^k}\bigg] &\leq  - \frac{\alpha_k}{2} \EE\left[\left\|\nabla\cL(\theta^k)\right\|^2_{(\epsilon I+\hat V^{k-D})^{-\frac{1}{2}}} \right]\!+\!  \frac{6 D L\alpha_k^2 \epsilon^{-\frac{1}{2}}}{M}\!\!\sum_{m\in{\cal 	M}}\!\sigma_m^2\nonumber\\
 & + \epsilon^{-\frac{1}{2}}\left(\frac{L}{12}+\frac{c }{2Ld_{\rm max}}\right) \sum\limits_{d=1}^{D} \EE\left[\|\theta^{k+1-d}-\theta^{k-d}\|^2\right].
\end{align}
\end{lemma}
Lemma \ref{lemma6} justifies the relevance of the stale yet properly selected stochastic gradients. 
Intuitively, the first term in the RHS of \eqref{eqn:lemma6} resembles the descent of using SGD with the unbiased stochastic gradient, and the second and third terms will diminish if the stepsizes are diminishing since {$\EE\left[\|\theta^k-\theta^{k-1}\|^2\right]={\cal O}(\alpha_k^2)$}. This is achieved by our designed communication rules.  

In view of Lemmas \ref{lemma:lossdescent} and \ref{lemma6}, we introduce the following \textbf{Lyapunov function}:
\begin{align}\label{eqn:Lyapunov}
 {\cal V}^k:= &\cL(\theta^k)-\cL(\theta^{\star}) -\sum\limits_{j=k}^{\infty}\alpha_j\beta_1^{j-k+1}\left\langle \nabla \cL(\theta^{k-1}), (\epsilon I+\hat V^k)^{-\frac{1}{2}}h^k\right\rangle\nonumber\\
    &+b_k\sum\limits_{d=0}^{D}\sum_{i=1}^p(\epsilon+\hat v_i^{k-d})^{-\frac{1}{2}}+\sum\limits_{d=1}^{D}\rho_d\|\theta^{k+1-d}-\theta^{k-d}\|^2	
\end{align}
where $\theta^{\star}$ is the solution of \eqref{eqn: problem}, {$\{b_k\}_{k=1}^K$} and {$\{\rho_d\}_{d=1}^D$} are constants that will be specified in the proof. 

The design of Lyapunov function in \eqref{eqn:Lyapunov} is motivated by the progress of $\cL(\theta^k)$ in Lemmas \ref{lemma:lossdescent}-\ref{lemma6}, and also coupled with our communication rules \eqref{eqn:workerrule1} and \eqref{eqn:workerrule2} that contain the parameter difference term. We find this new Lyapunov function can lead to a much simple proof of Adam and AMSGrad, which is of independent interest.  
The following lemma captures the progress of the Lyapunov function.

\begin{lemma}\label{lemma:lyapunovdescent}
Under Assumptions \ref{assump:smoothness}-\ref{assump:gradientestimator}, if $\{b_k\}_{k=1}^K$ and $\{\rho_d\}_{d=1}^D$ in \eqref{eqn:Lyapunov} are chosen properly, we have
\begin{align}\label{descent}
 \EE[ {\cal V}^{k+1}]-\EE[{\cal V}^k]   \leq  -\frac{\alpha_k(1-\beta_1)}{2}\!\left(\epsilon+\frac{\sigma^2}{1-\beta_2}\right)^{\!-\frac{1}{2}}\!\!\EE\left[\left\|\nabla\cL(\theta^k)\right\|^2 \right]+\alpha_k^2 C_0
\end{align}
where the constant $C_0$ depends on the CADA and problem parameters {$c, \beta_1, \beta_2, \epsilon, D$}, and {$L, \{\sigma_m^2\}$}.
\end{lemma}
\vspace{-0.1cm}
The first term in the RHS of \eqref{descent} is strictly negative, and the second term is positive but potentially small since it is ${\cal O}(\alpha_k^2)$ with $\alpha_k\rightarrow 0$. 
This implies that the function ${\cal V}^k$ will eventually converge if we choose the stepsizes appropriately. Lemma \ref{lemma:lyapunovdescent} is a generalization of SGD's descent lemma. 
If we set {$\beta_1=\beta_2=0$} in \eqref{eqn:CADA1} and {$b_k=0, \rho_d=0,\,\forall d, k$} in \eqref{eqn:Lyapunov}, then Lemma \ref{lemma:lyapunovdescent} reduces to that of SGD in terms of $\cL(\theta^k)$; see e.g., \cite[Lemma 4.4]{bottou2016}.

\subsection{Main convergence results}
\vspace{-0.1cm}
Building upon our Lyapunov analysis, we first present the convergence in nonconvex case.
\begin{theorem}[nonconvex]\label{thm:nonconvex}
Under Assumptions \ref{assump:smoothness}, \ref{assump:gradientestimator}, if we choose $\alpha_k=\alpha={\cal O}(\frac{1}{\sqrt{K}})$ and $\beta_1<\sqrt{\beta_2}<1$, then the iterates $\{\theta^k\}$ generated by CADA satisfy 
\begin{equation}\label{eq.rate-noncvx}
    \frac{1}{K}\sum\limits_{k=0}^{K-1}\EE\left[\|\nabla\cL(\theta^k)\|^2\right]={\cal O}\left(\frac{1}{\sqrt{K}}\right).
\end{equation}
\end{theorem}
\vspace{-0.2cm}


From Theorem \ref{thm:nonconvex}, the convergence rate of CADA in terms of the average gradient norms is ${\cal O}(1/\sqrt{K})$, which matches that of the plain-vanilla Adam \cite{reddi2019adam,chen2019adam}. 
Unfortunately, due to the complicated nature of Adam-type analysis, the bound in \eqref{eq.rate-noncvx} does not achieve the linear speed-up as analyzed for asynchronous nonadaptive SGD such as \cite{lian2016nips}. However, our analysis is tailored for adaptive SGD and does not make any assumption on the asynchrony, e.g., the set of uploading workers are independent from the past or even independent and identically distributed.

Next we present the convergence results under a slightly stronger assumption on the loss ${\cal L}(\theta)$.
\begin{assumption}\label{assump:strongconvexity}
The loss function $\cL(\theta)$ satisfies the Polyak-{\L}ojasiewicz (PL) condition with the constant $\mu>0$, that is 
{${\cal L}(\theta)-{\cal L}(\theta^*)\leq \frac{1}{2\mu}\left\|{\cal L}(\theta)\right\|^2$}.
\end{assumption}
\vspace{-0.2cm}

The PL condition is weaker than the strongly convexity, which does not even require convexity \cite{karimi2016}. And it is satisfied by a wider range of problems such as least squares for
underdetermined linear systems, logistic regression, and also certain types of neural networks.


We next establish the convergence of CADA under this condition.
\begin{theorem}[PL-condition]\label{thm:stronglyconvex}
Under Assumptions \ref{assump:smoothness}-\ref{assump:strongconvexity}, if we choose the stepsize as $\alpha_k=\frac{2}{\mu(k+K_0)}$ for a given constant $K_0$, then $\theta^K$ generated by Algorithm \ref{alg: CADA} satisfies
\begin{equation}
    \EE\left[\cL(\theta^K)\right]-\cL(\theta^{\star})={\cal O}\left(\frac{1}{K}\right).
\end{equation}
\end{theorem}
Theorem \ref{thm:stronglyconvex} implies that under the PL-condition of the loss function, the CADA algorithm can achieve the global convergence in terms of the loss function, with a fast rate ${\cal O}(1/K)$. 
Compared with the previous analysis for LAG \cite{chen2018lag}, as we highlighted in Section \ref{subsec.Lya-analysis}, the analysis for CADA is more involved, since it needs to deal with not only the outdated gradients but also the \emph{stochastic momentum} gradients and the \emph{adaptive} matrix learning rates. We tackle this issue by i) considering a new set of communication rules \eqref{eqn:workerrule1} and \eqref{eqn:workerrule2} with reduced variance; and, ii) incorporating the effect of momentum gradients and the drift of adaptive learning rates in the new Lyapunov function \eqref{eqn:Lyapunov}. 

\begin{figure*}[t]
 \vspace{-0.3cm}
    \includegraphics[width=.34\textwidth]{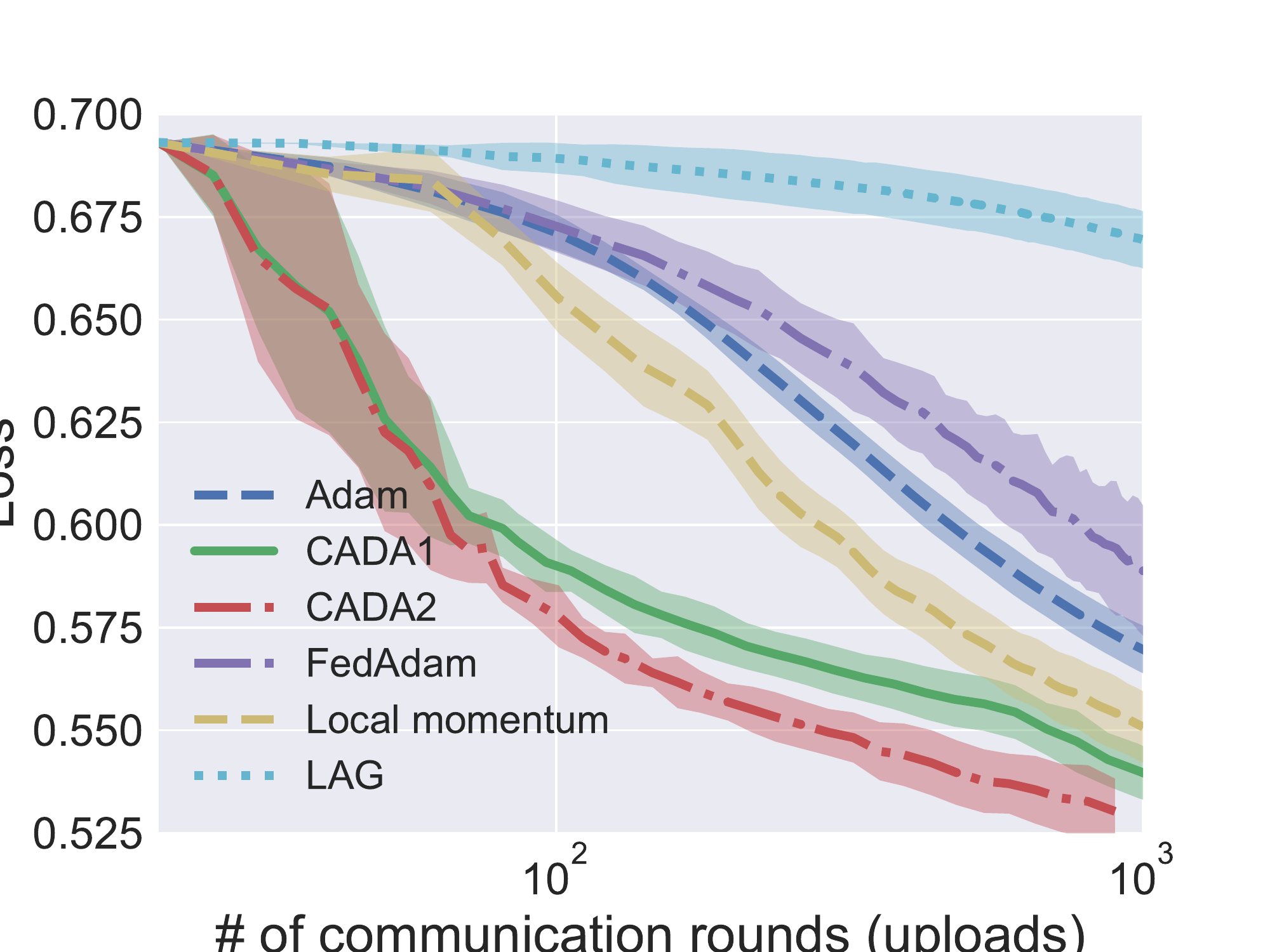}
    \hspace*{-2ex}
    \includegraphics[width=.34\textwidth]{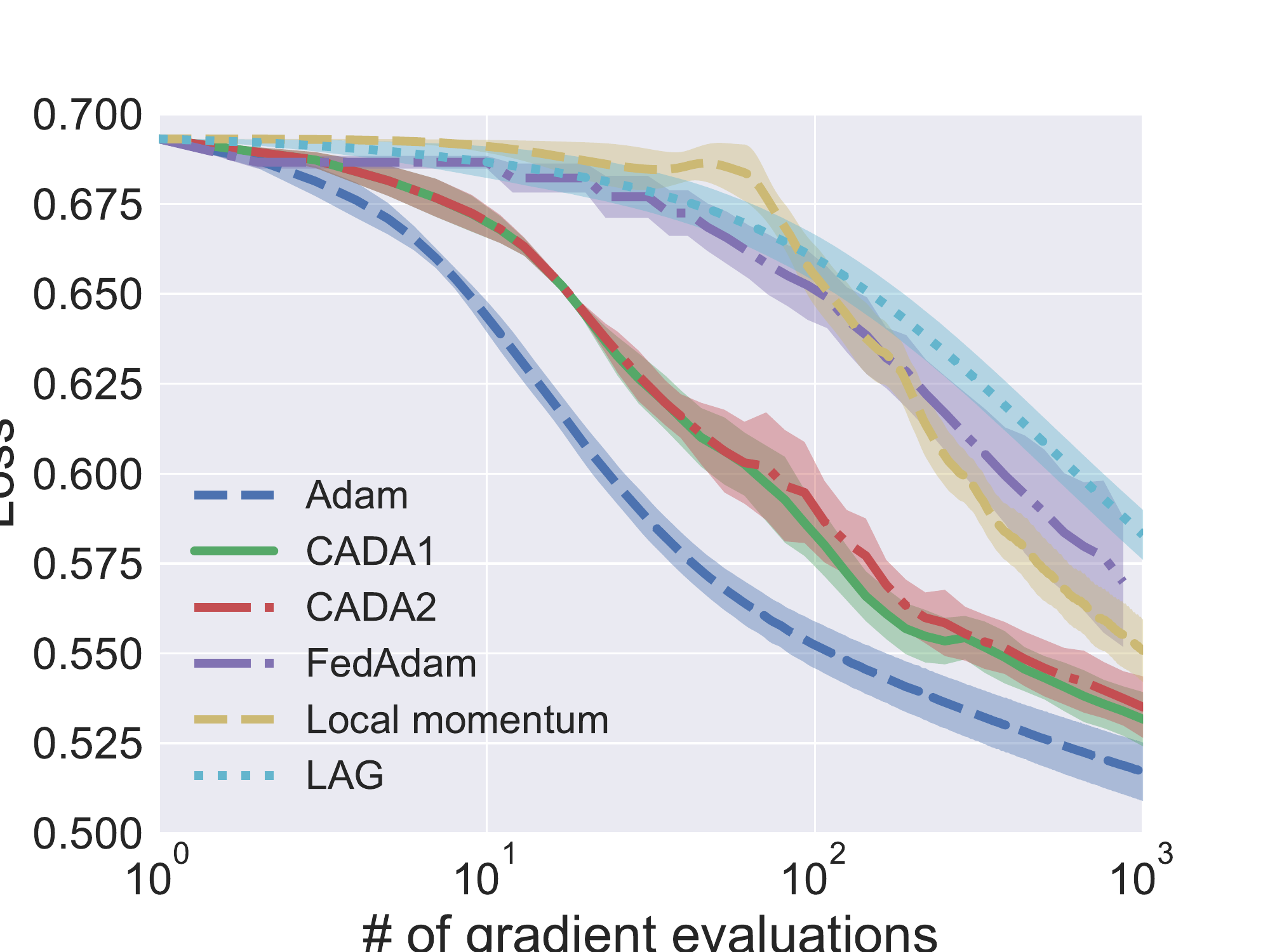}
    \hspace*{-2ex}    
    \includegraphics[width=.34\textwidth]{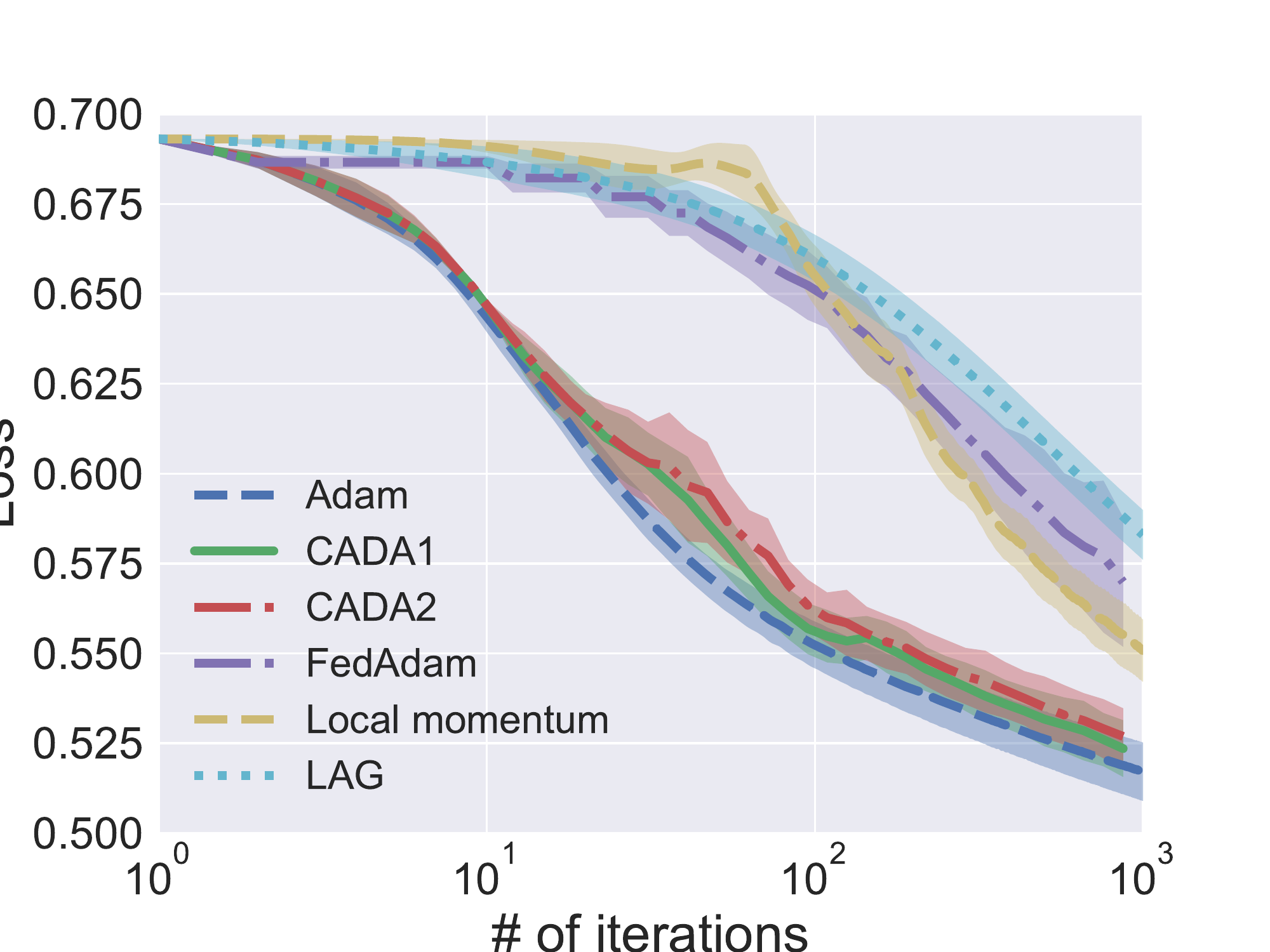}
    \hspace*{-3ex}
    \vspace{-0.2cm}
    \caption{Logistic regression training loss on \textit{covtype} dataset averaged over 10 Monte Carlo runs.}
    \label{fig:covtype}
\end{figure*}

\begin{figure*}[t]
\vspace{-0.2cm}
  \hspace*{-3ex}
    \includegraphics[width=.35\textwidth]{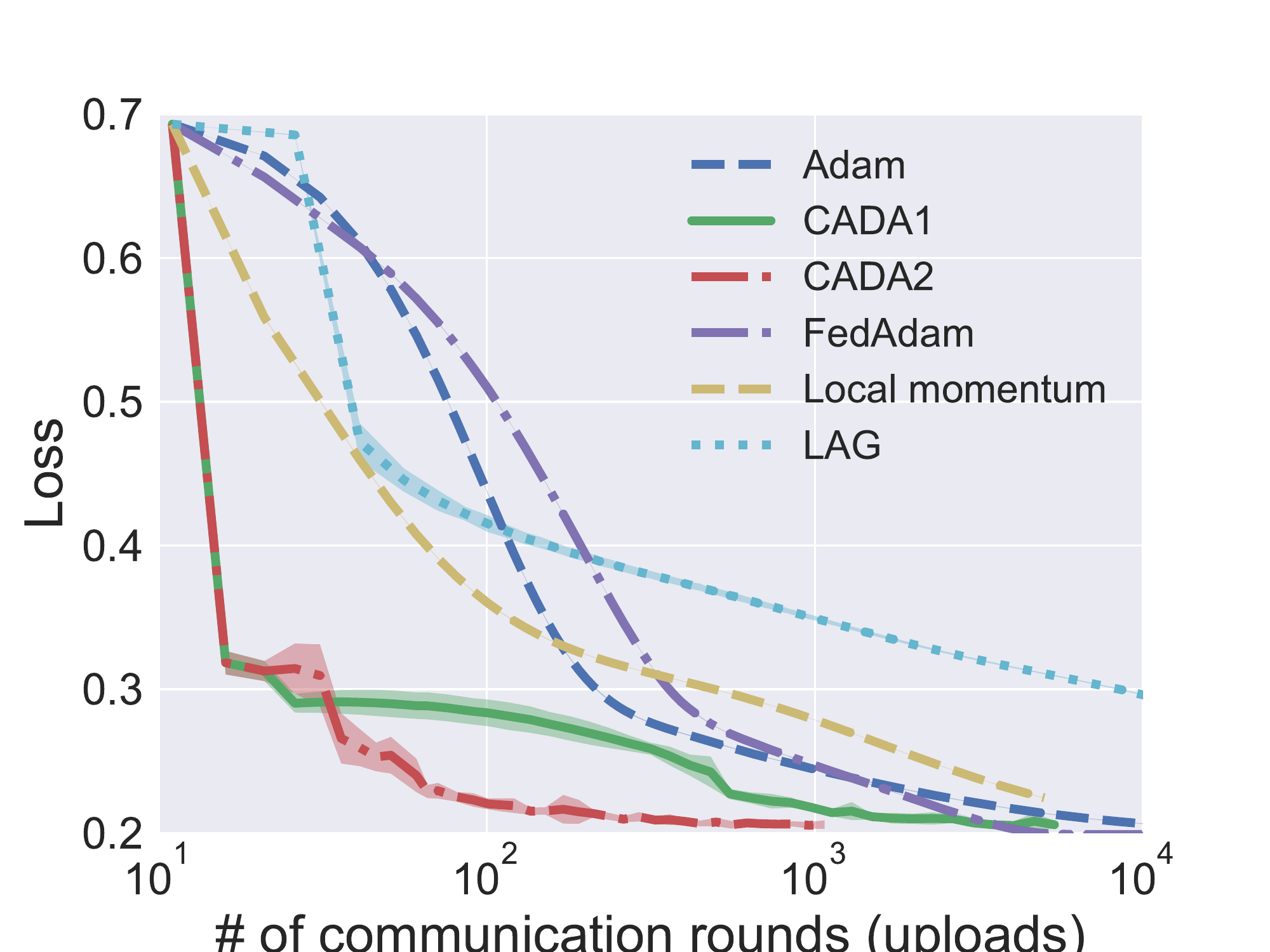}
    \hspace*{-2ex}
    \includegraphics[width=.35\textwidth]{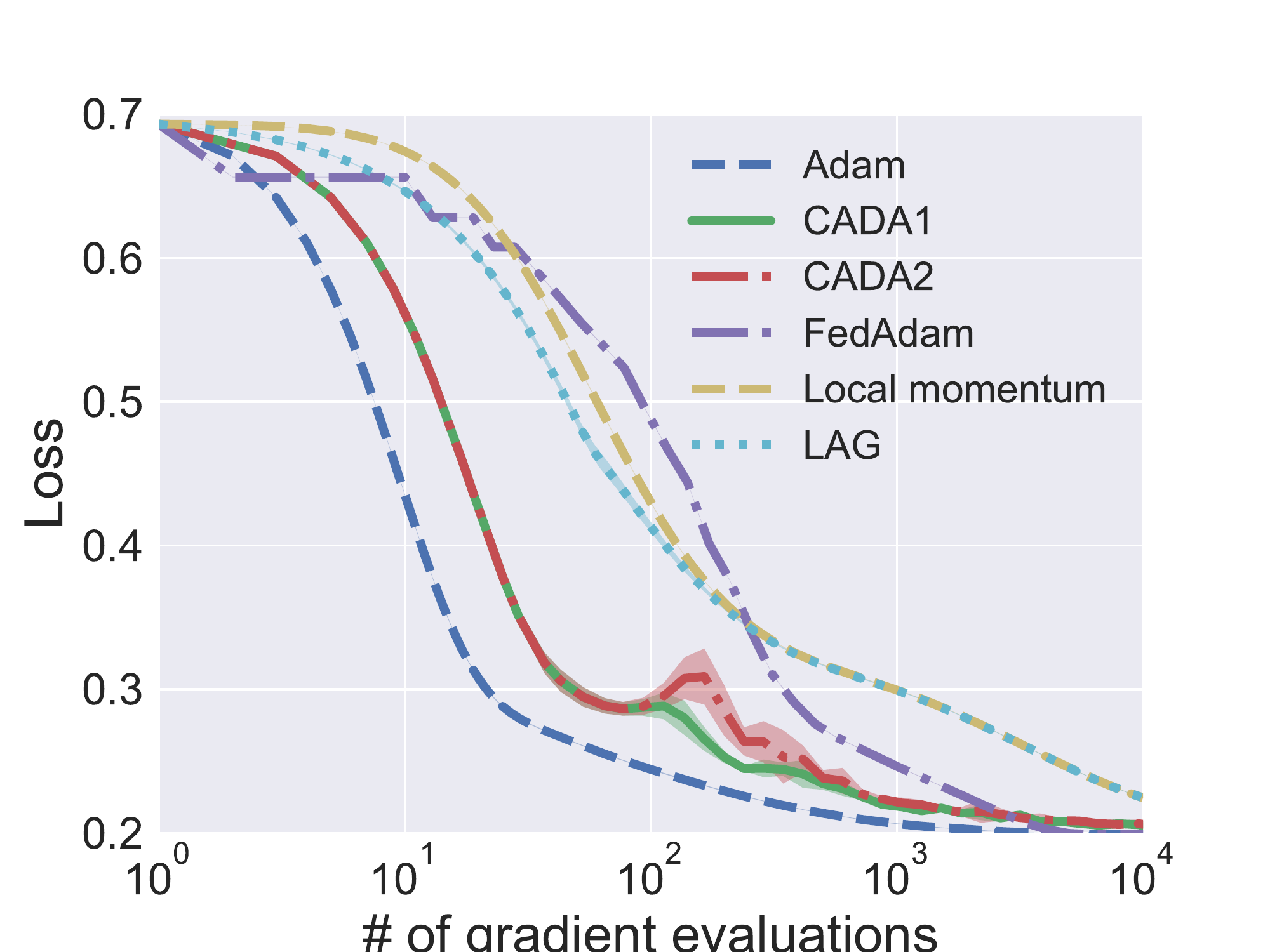}
     \hspace*{-2ex}
        \includegraphics[width=.35\textwidth]{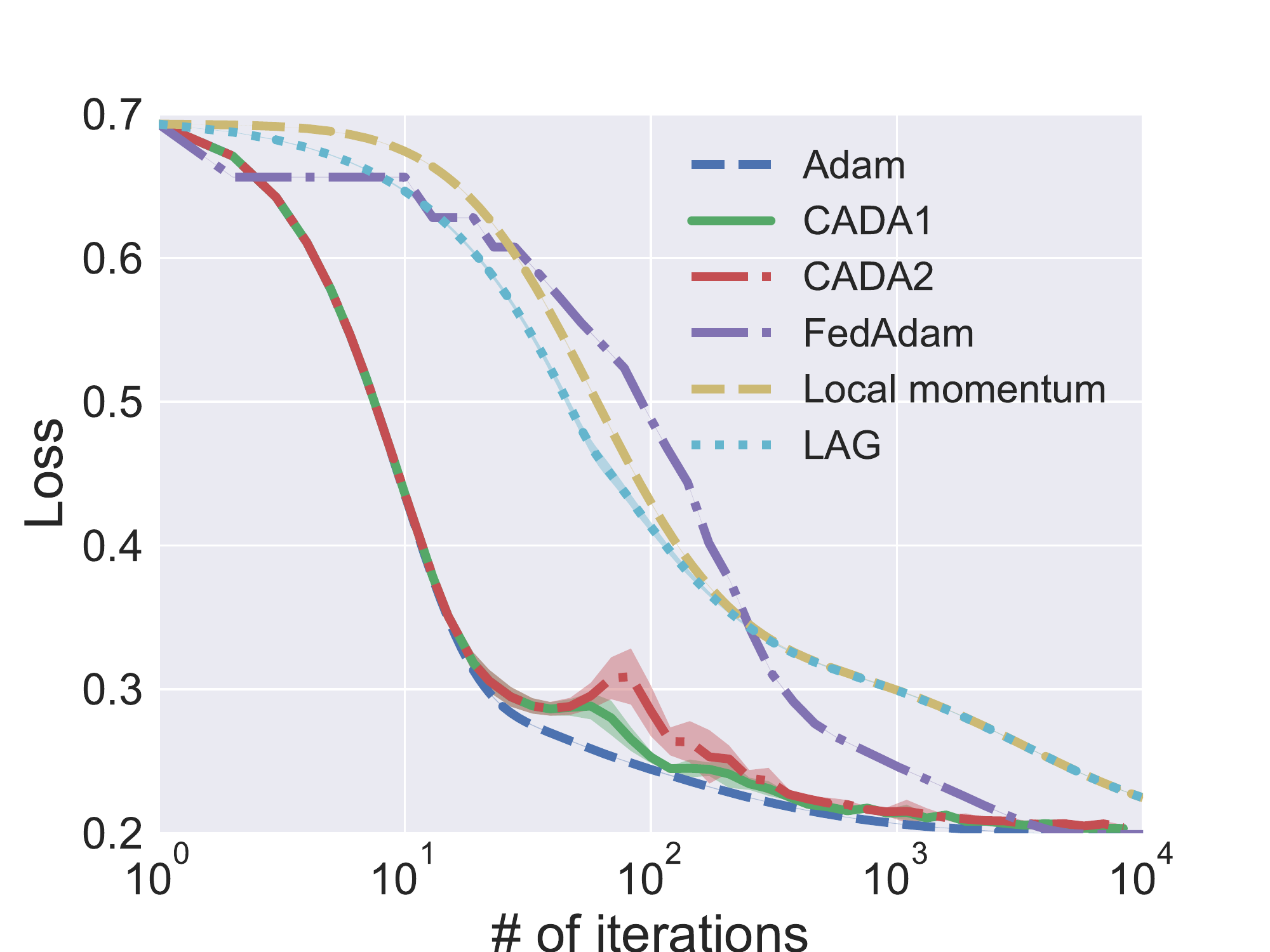}
    \hspace*{-3ex}
\vspace{-0.2cm}
    \caption{Logistic regression training loss on \textit{ijcnn1} dataset averaged over 10 Monte Carlo runs.}
    \label{fig:ijcnn}
\vspace{-0.2cm}
\end{figure*}

\begin{figure*}[t]
 \hspace*{-3ex}
    \includegraphics[width=.35\textwidth]{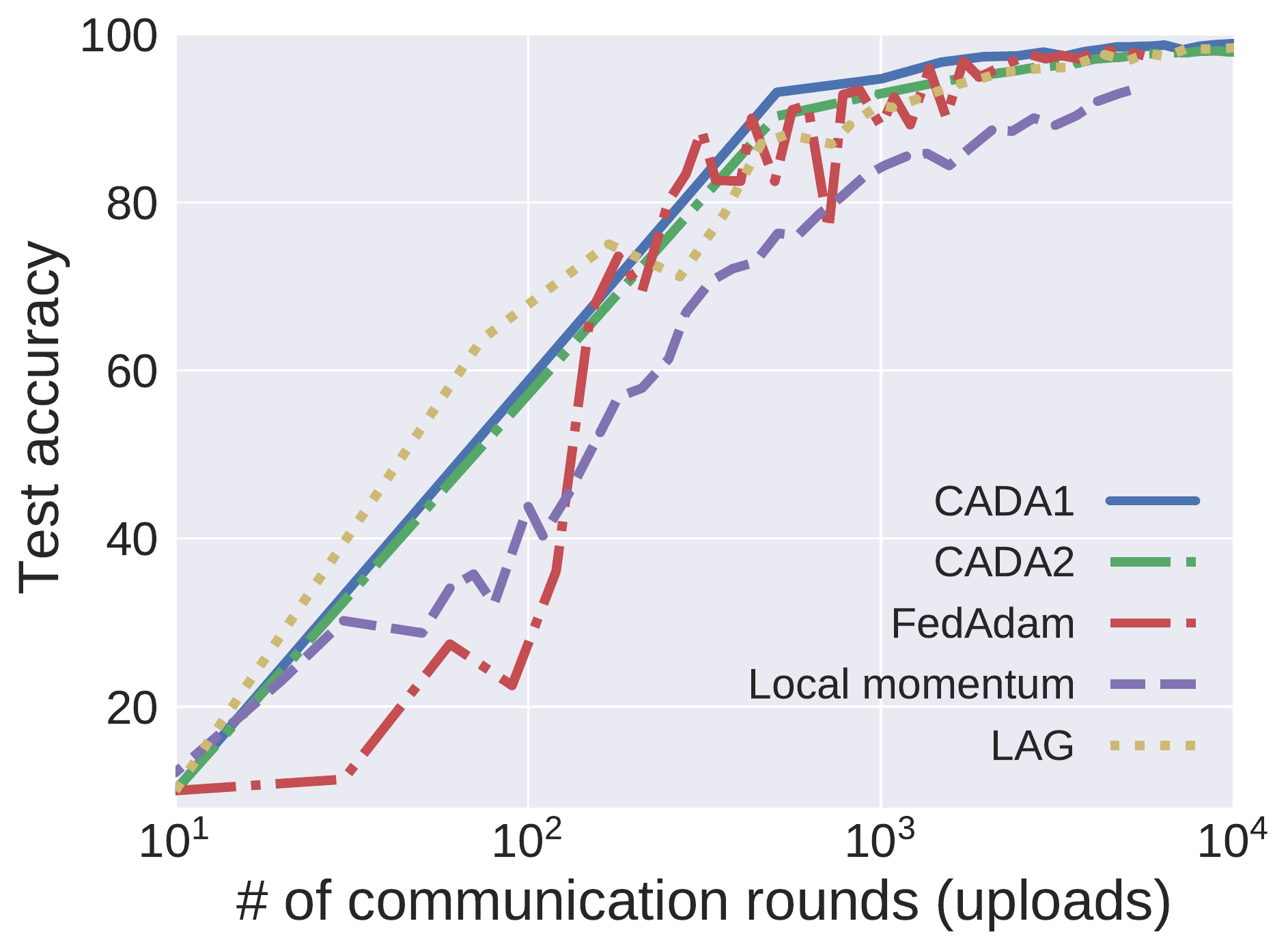}
    \hspace*{-2ex}
    \includegraphics[width=.35\textwidth]{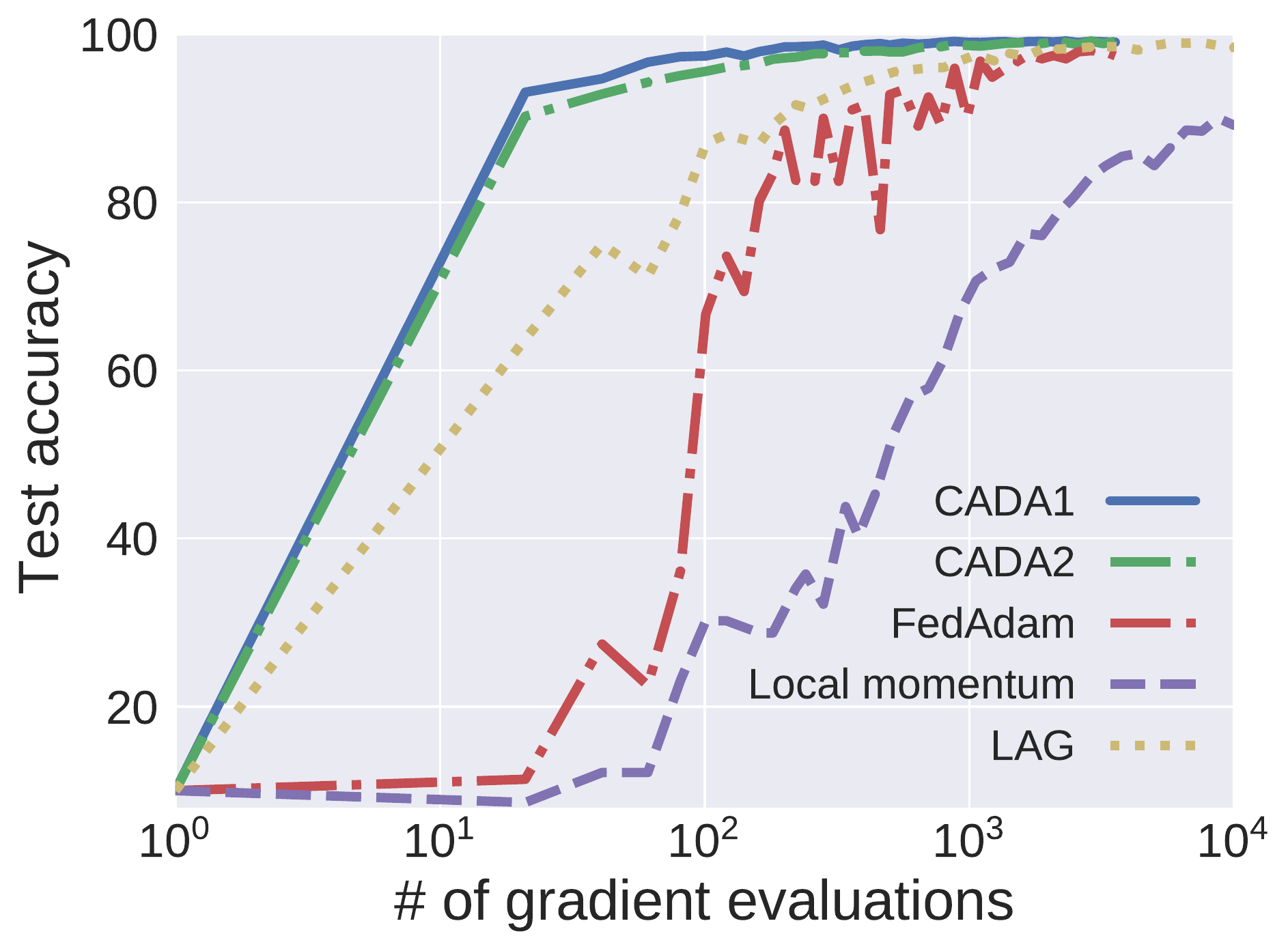}
        \hspace*{-2ex}
    \includegraphics[width=.35\textwidth]{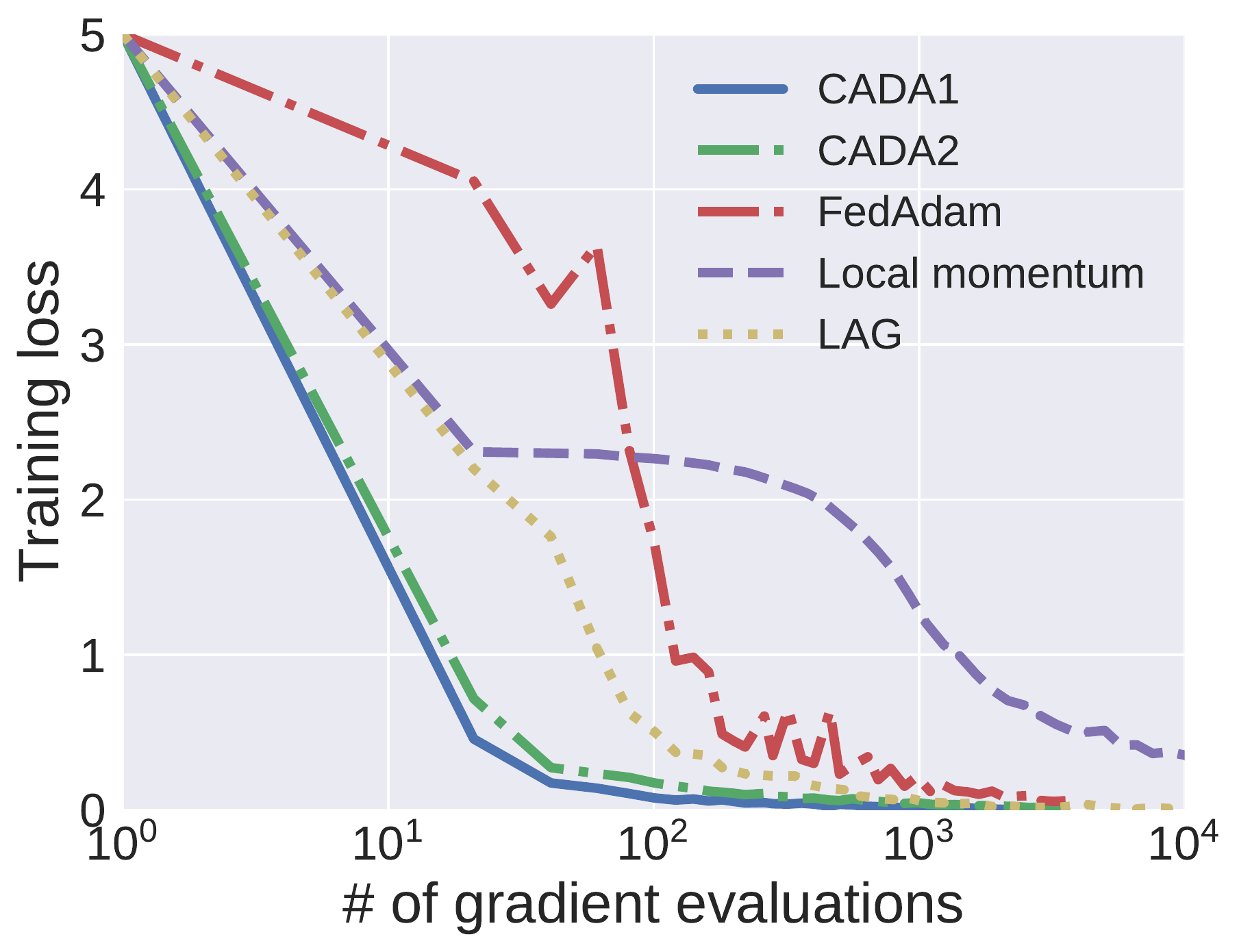}
            \hspace*{-3ex}
     \vspace{-0.2cm}
    \caption{Training Neural network for classification on \textit{mnist} dataset.}
    \label{fig:NNmnist}
\vspace{-0.2cm}
\end{figure*}


\begin{figure*}[t]
 \hspace*{-3ex}
    \includegraphics[width=.35\textwidth]{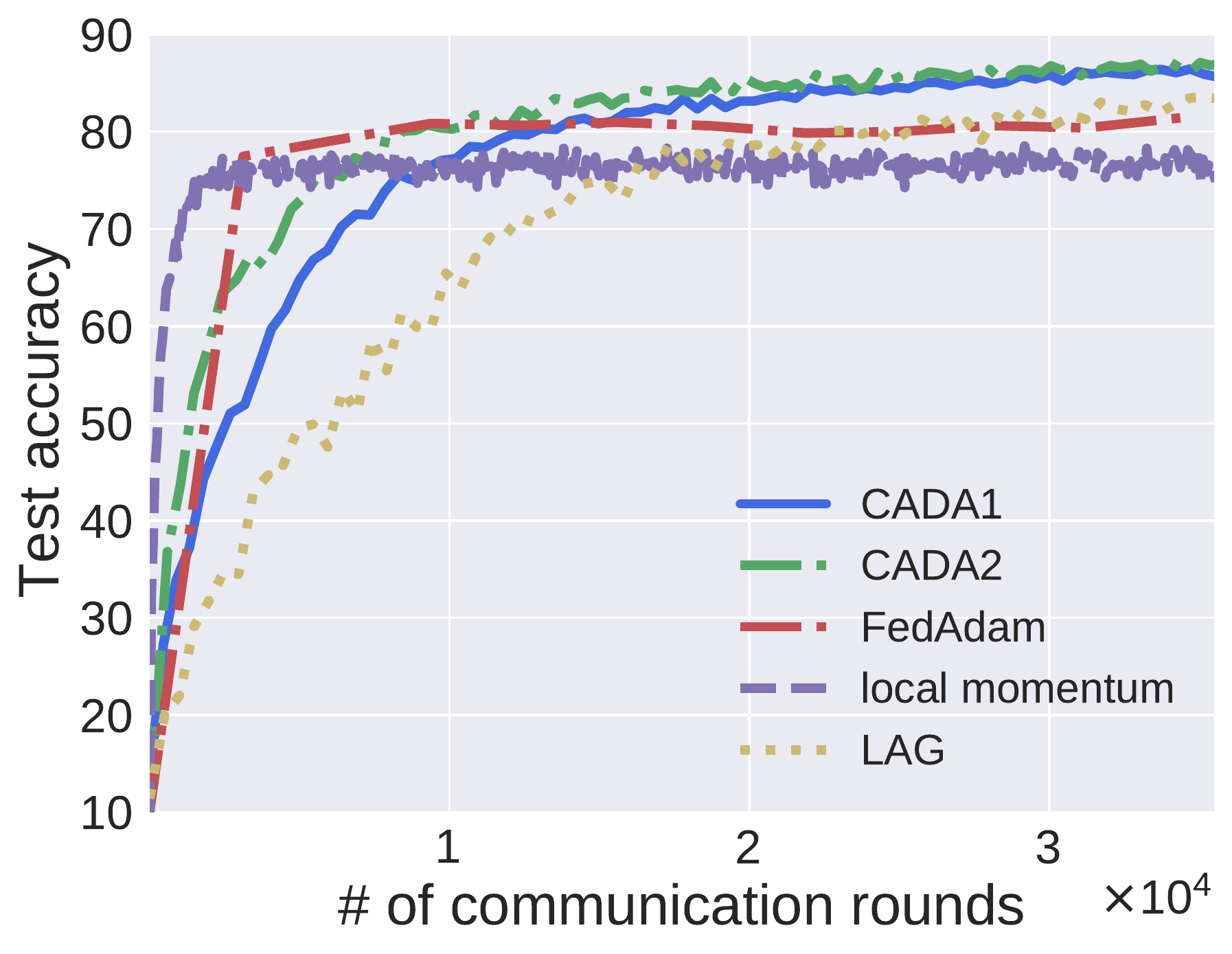}
    \hspace*{-2ex}
    \includegraphics[width=.35\textwidth]{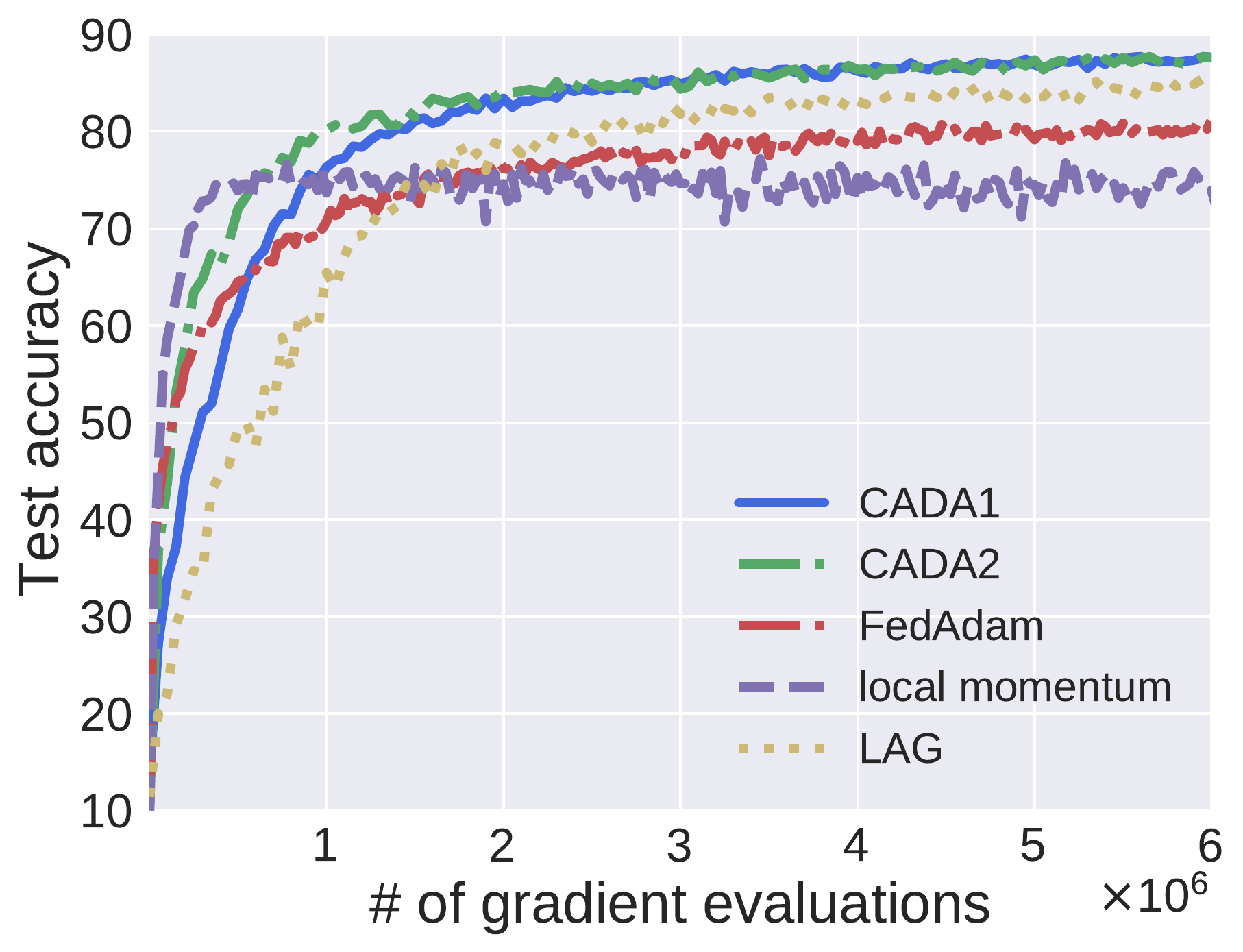}
        \hspace*{-2ex}
    \includegraphics[width=.35\textwidth]{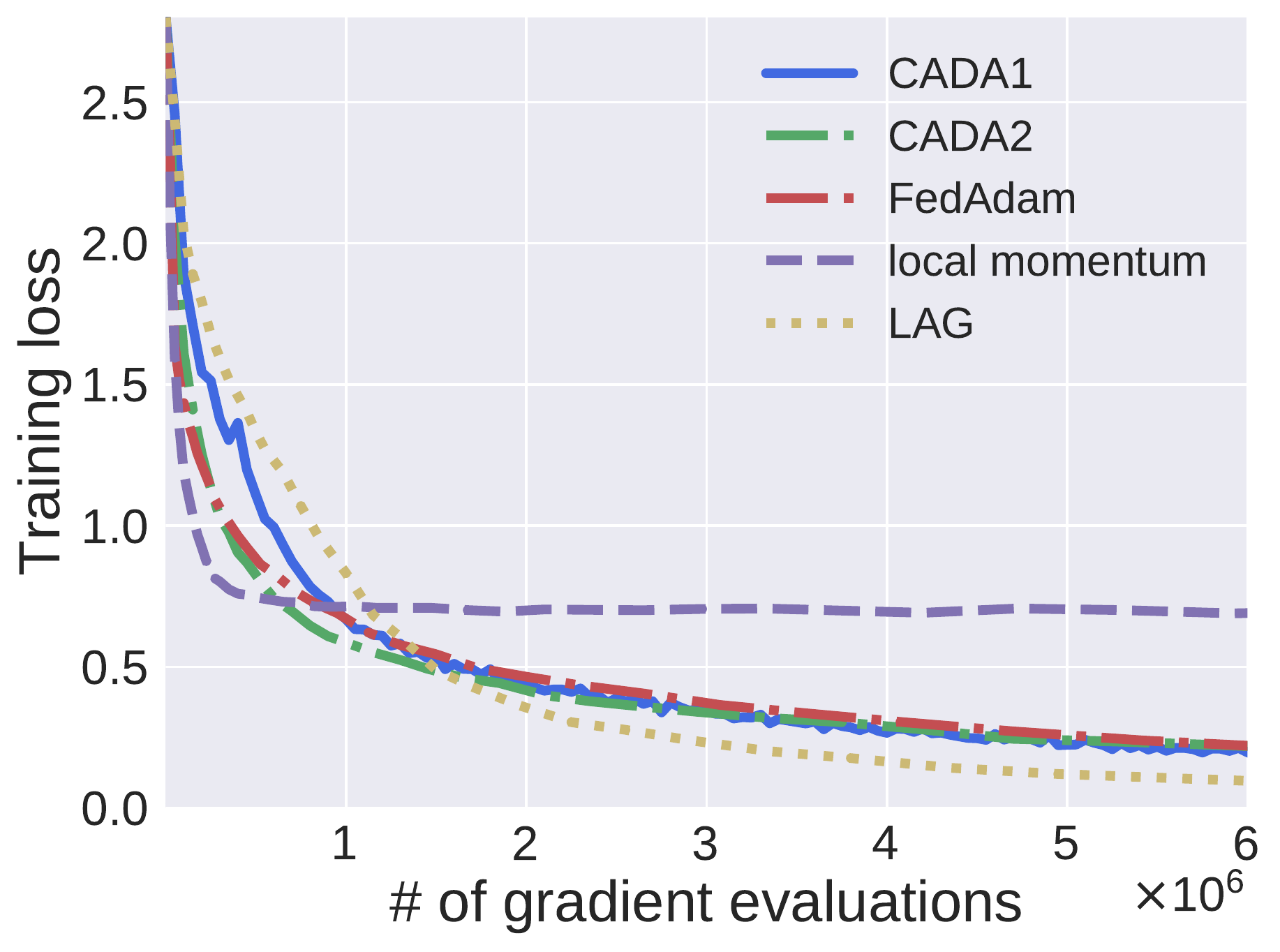}
            \hspace*{-3ex}
     \vspace{-0.4cm}
    \caption{Training Neural network for classification on \textit{cifar10} dataset.}
    \label{fig:NNcifar10}
\vspace{-0.2cm}
\end{figure*}

\section{Simulations}
\vspace{-0.1cm}
In order to verify our analysis and show the empirical performance of CADA, we conduct simulations using logistic regression and training neural networks. 
Data are distributed across $M=10$ workers during all tests. 
We benchmark CADA with some popular methods such as Adam \cite{kingma2014adam}, stochastic version of LAG \cite{chen2018lag}, local momentum \cite{yu2019lcml} and the state-of-the-art FedAdam \cite{reddi2020adaptive}. 
For local momentum and FedAdam, workers perform model update independently, which are averaged over all workers every $H$ iterations. In simulations, critical parameters are optimized for each algorithm by a grid-search.  
Due to space limitation, please see the detailed choice of parameters, and additional experiments in the \textbf{supplementary document}.


\noindent\textbf{Logistic regression.}
For CADA, the maximal delay is $D=100$ and $d_{\rm max}=10$. 
For local momentum and FedAdam, we manually optimize the averaging period as $H=10$ for \textit{ijcnn1} and $H=20$ for \textit{covtype}. Results are averaged over 10 Monte Carlo runs.

Tests on logistic regression are reported in Figures \ref{fig:covtype}-\ref{fig:ijcnn}. 
In our tests, two CADA variants achieve the similar iteration complexity as the original Adam and outperform all other baselines in most cases. Since our CADA requires two gradient evaluations per iteration, the gradient complexity (e.g., computational complexity) of CADA is higher than Adam, but still smaller than that of other baselines. 
For logistic regression task, CADA1 and CADA2 save the number of communication uploads by at least one order of magnitude. 

\noindent\textbf{Training neural networks.}
We train a neural network with two convolution-ELU-maxpooling layers followed by two fully-connected layers for 10 classes classification on \textit{mnist}. 
We use the popular \emph{ResNet20} model on \textit{CIFAR10} dataset, which has 20 and roughly 0.27 million parameters. 
We searched the best values of $H$ from the grid $\{1,4,6,8,16\}$ to optimize the testing accuracy vs communication rounds for each algorithm. 
In CADA, the maximum delay is $D=50$ and the average interval $d_{\rm max}=10$. 
See tests under different $H$ in the supplementary material. 

Tests on training neural networks are reported in Figures \ref{fig:NNmnist}-\ref{fig:NNcifar10}. 
In \textit{mnist}, CADA1 and CADA2 save the number of communication uploads by roughly 60\% than local momentum and slightly more than FedAdam.
In \textit{cifar10}, CADA1 and CADA2 achieve competitive performance relative to the state-of-the-art algorithms FedAdam and local momentum. We found that if we further enlarge $H$, FedAdam and local momentum converge fast at the beginning, but reached worse test accuracy (e.g., 5\%-15\%). 
It is also evident that the CADA1 and CADA2 rules achieve more communication reduction than the direct stochastic version of LAG, which verifies the intuition in Section \ref{subsec.lag}. 

\section{Conclusions} 
While Adaptive SGD methods have been widely applied in the single-node setting, their performance in the distributed learning setting is less understood. 
In this paper, we have developed a communication-adaptive distributed Adam method that we term CADA, which endows an additional dimension of adaptivity to Adam tailored for its distributed implementation. 
CADA method leverages a set of adaptive communication rules to detect and then skip less informative  communication rounds between the server and workers during distributed learning. All CADA variants are simple to implement, and have convergence rate comparable to the original Adam.

\onecolumn
\begin{center}
{\large \bf Supplementary materials for}
{\large \bf ``CADA: Communication-Adaptive Distributed Adam"}
\end{center}


In this supplementary document, we first present some basic inequalities that will be used frequently in this document, and then present the missing derivations of some claims, as well as the proofs of all the lemmas and theorems in the paper, which is followed by details on our experiments. The content of this supplementary document is summarized as follows. 


\section{Supporting Lemmas}

Define the $\sigma$-algebra $\Theta^k=\{\theta^l, 1\leq l\leq k\}$. For convenience, we also initialize parameters as $\theta^{-D},\theta^{-D+1},\ldots,\theta^{-1}=\theta^0$. 
Some basic facts used in the proof are reviewed as follows.

\noindent{\bf Fact 1.} Assume that $X_1, X_2, \ldots, X_n\in\RR^p$ are independent random variables, and $EX_1=\cdots=EX_n=0$. Then
\begin{align}\label{eqn:equation1}
    \EE\Bigg[\Big\|\sum\limits_{i=1}^nX_i\Big\|^2\Bigg]=\sum\limits_{i=1}^n\EE\left[\|X_i\|^2\right].
\end{align}

\noindent{\bf Fact 2.} (Young's inequality) For any $\theta_1,\theta_2\in\RR^p,\varepsilon>0$,
\begin{align}\label{eqn:young}
    \dotp{\theta_1,\theta_2}\leq\frac{\|\theta_1\|^2}{2\varepsilon}+\frac{\varepsilon \|\theta_2\|^2}{2}.
\end{align}

As a consequence, we have
\begin{align}\label{eqn:young1}
    \|\theta_1+\theta_2\|^2\leq \Big(1+\frac{1}{\varepsilon}\Big)\|\theta_1\|^2+(1+\varepsilon)\|\theta_2\|^2.
\end{align}

\noindent{\bf Fact 3.} (Cauchy-Schwarz inequality) For any $\theta_1,\theta_2,\ldots,\theta_n\in\RR^p$, we have
\begin{align}\label{eqn:cauchy}
    \Big\|\sum\limits_{i=1}^n\theta_i\Big\|^2\leq n\sum\limits_{i=1}^n\|\theta_i\|^2.
\end{align}

\begin{lemma}\label{lemma2}
For $k-\tau_{\rm max}\leq l \leq k-D$, if $\{\theta^k\}$ are the iterates generated by CADA, we have
\begin{align}\label{eqn:inequality1-2}
&\EE\left[\dotp{\nabla \cL(\theta^k),(\epsilon I+\hat V^{k-D})^{-\frac{1}{2}}\left(\nabla\ell(\theta^{l};\xi_m^k)-\nabla\ell(\theta^{l};\xi_m^{k-\tau_m^k})\right)}\right]\nonumber\\
\leq&\frac{L\epsilon^{-\frac{1}{2}}}{12\alpha_k} \sum\limits_{d=1}^{D} \EE\left[\|\theta^{k+1-d}-\theta^{k-d}\|^2\right]\!+\! 6 D L\alpha_k \epsilon^{-\frac{1}{2}}\sigma_m^2
\end{align}
and similarly, we have
\begin{align}\label{eqn:inequality2}
    &\EE\left[\dotp{\nabla\cL(\theta^k),(\epsilon I+\hat V^{k-D})^{-\frac{1}{2}}\left(\nabla\cL_m(\theta^l)-\nabla\ell(\theta^l;\theta^{k-\tau_m^k}\right)}\right]\nonumber\\
    \leq & \frac{L\epsilon^{-\frac{1}{2}}}{12\alpha_k}\sum\limits_{d=1}^D\EE\left[\|\theta^{k+1-d}-\theta^{k-d}\|^2\right]+ {3DL\alpha_k\epsilon^{-\frac{1}{2}}} \sigma_m^2.
\end{align}
\end{lemma}

\begin{proof}
We first show the following holds. 
\begin{align}\label{eq.pf-lm2-1}
& \EE\left[\dotp{\nabla\cL(\theta^{l}),(\epsilon I+\hat V^{k-D})^{-\frac{1}{2}}\left(\nabla\ell(\theta^l;\xi_m^k)-\nabla\ell(\theta^{l};\xi_m^{k-\tau_m^k})\right)}\right]\nonumber\\
     \stackrel{(a)}{=}&\EE\left[\EE\Big[\dotp{\nabla\cL(\theta^{l}),(\epsilon I+\hat V^{k-D})^{-\frac{1}{2}}\left(\nabla\ell(\theta^l;\xi_m^k)-\nabla\ell(\theta^{l};\xi_m^{k-\tau_m^k})\right)}\Big|\Theta^l\Big]\right]\nonumber\\
    \stackrel{(b)}{=}&\EE\left[\dotp{\nabla\cL(\theta^l), (\epsilon I+\hat V^{k-D})^{-\frac{1}{2}}\EE\Big[ \nabla\ell(\theta^l;\xi_m^k)-\nabla\ell(\theta^{l};\xi_m^{k-\tau_m^k})\big|\Theta^{l}\Big]}\right]\nonumber\\
    =&\EE\left[\dotp{\nabla\cL(\theta^l), (\epsilon I+\hat V^{k-D})^{-\frac{1}{2}}\left(\nabla\cL_m(\theta^l)-\nabla\cL_m(\theta^l)\right)}\right] =0
\end{align}
where (a) follows from the law of total probability, and (b) holds because $\hat V^{k-D}$ is deterministic conditioned on $\Theta^l$ when $k-D\leq l$.

We first prove \eqref{eqn:inequality1-2} by decomposing it as 
\begin{align}\label{eqn:inequality1}
    &\EE\left[\dotp{\nabla \cL(\theta^k), (\epsilon I+\hat V^{k-D})^{-\frac{1}{2}}\left(\nabla\ell(\theta^{l};\xi_m^k)-\nabla\ell(\theta^{l};\xi_m^{k-\tau_m^k})\right)}\right]\nonumber \\
    \stackrel{(c)}{=}&\EE\left[\dotp{\nabla \cL(\theta^k)-\nabla \cL(\theta^{l}), (\epsilon I+\hat V^{k-D})^{-\frac{1}{2}}\left(\nabla\ell(\theta^{l};\xi_m^k)-\nabla\ell(\theta^{l};\xi_m^{k-\tau_m^k})\right)}\right]\nonumber\\
     \stackrel{(d)}{\leq}& L\EE\left[\Big\|(\epsilon I+\hat V^{k-D})^{-\frac{1}{4}}\Big\|\Big\|\theta^k-\theta^l\Big\|\Big\|(\epsilon I+\hat V^{k-D})^{-\frac{1}{4}}\left(\nabla\ell(\theta^{l};\xi_m^k)-\nabla\ell(\theta^{l};\xi_m^{k-\tau_m^k})\right)\Big\|\right]\nonumber\\
     \stackrel{(e)}{\leq}&\frac{L\epsilon^{-\frac{1}{2}}}{12  D\alpha_k}\underbracket{\EE\left[\|\theta^k-\theta^{l}\|^2\right]}_{I_1}+\frac{6 D L\alpha_k\epsilon^{-\frac{1}{2}}}{2}\underbracket{\EE\left[\|\nabla\ell(\theta^{l};\xi_m^k)-\nabla\ell(\theta^{l};\xi_m^{k-\tau_m^k})\|^2\right]}_{I_2}
\end{align}
where (c) holds due to \eqref{eq.pf-lm2-1}, (d) uses Assumption \ref{assump:smoothness}, 
and (e) applies the Young's inequality. 

Applying the Cauchy-Schwarz inequality to $I_1$, we have
\begin{align}\label{eq.pf-I1}
    I_1=&\EE\Big[\Big\|\sum\limits_{d=1}^{k-l}(\theta^{k+1-d}-\theta^{k-d})\Big\|^2\Big]\nonumber\\    
    \leq&(k-l)\sum\limits_{d=1}^{k-l}\EE\Big[\|\theta^{k+1-d}-\theta^{k-d}\|^2\Big] \leq D\sum\limits_{d=1}^D\EE\Big[\|\theta^{k+1-d}-\theta^{k-d}\|^2\Big].
\end{align}
Applying Assumption \ref{assump:gradientestimator} to $I_2$, we have
 \begin{align}\label{eq.pf-I2}
    I_2=&\EE\Big[\big\|\nabla\ell(\theta^l;\xi_m^k)-\nabla\ell(\theta^l;\xi_m^{k-\tau_m^k})\big\|^2\Big]\nonumber\\
       =&\EE\Big[\big\|\nabla\ell(\theta^l;\xi_m^k)\big\|^2\Big]+\EE\Big[\big\|\nabla\ell(\theta^l;\xi_m^{k-\tau_m^k})\big\|^2\Big]    \leq 2\sigma_m^2
\end{align}
where the last inequality uses Assumption \ref{assump:gradientestimator}. 
Plugging \eqref{eq.pf-I1} and \eqref{eq.pf-I2} into \eqref{eqn:inequality1}, it leads to \eqref{eqn:inequality1-2}.

Likewise, following the steps to \eqref{eqn:inequality1}, it can be verified that \eqref{eqn:inequality2} also holds true. 
\end{proof}

\begin{lemma}\label{lemma3}
	Under Assumption \ref{assump:gradientestimator}, the parameters $\{h^k, \hat{v}^k\}$ of CADA in Algorithm \ref{alg: CADA} satisfy
\begin{equation}
\|h^k\|\leq \sigma,~~~\forall k;~~~\hat{v}_i^k\leq \sigma^2, ~~~\forall k, i
\end{equation}
where $\sigma:=\frac{1}{M}\sum_{m\in{\cal M}} \sigma_m$.
\end{lemma}
\begin{proof}
Using Assumption 2, it follows that 
\begin{equation}
\|\bm\nabla^k\|=\left\|\frac{1}{M}\sum_{m\in{\cal M}} \nabla \ell(\theta^{k-\tau_m^k};\xi_m^{k-\tau_m^k})\right\|\leq \frac{1}{M}\sum_{m\in{\cal M}} \left\|\nabla \ell(\theta^{k-\tau_m^k};\xi_m^{k-\tau_m^k})\right\|	 \leq \frac{1}{M}\sum_{m\in{\cal M}} \sigma_m=\sigma. 
\end{equation}
Therefore, from the update \eqref{eq.CADA-1}, we have
\begin{align*}
\|h^{k+1}\|\leq \beta_1\|h^{k}\|+(1-\beta_1)\|\bm\nabla^k\|\leq \beta_1\|h^{k}\|+(1-\beta_1) \sigma.
\end{align*}
Since $\|h^{1}\|\leq \sigma$, if follows by induction that $\|h^{k+1}\|\leq \sigma,~\forall k$. 

Using Assumption 2, it follows that 
\begin{align}
(\nabla_i^k)^2&=\left(\frac{1}{M}\sum_{m\in{\cal M}} \nabla_i \ell(\theta^{k-\tau_m^k};\xi_m^{k-\tau_m^k})\right)^2\nonumber \\
	&\leq \frac{1}{M}\sum_{m\in{\cal M}} \left(\nabla_i \ell(\theta^{k-\tau_m^k};\xi_m^{k-\tau_m^k})\right)^2\nonumber\\	
	& \leq \frac{1}{M}\sum_{m\in{\cal M}} \left\|\nabla \ell(\theta^{k-\tau_m^k};\xi_m^{k-\tau_m^k})\right\|^2=\frac{1}{M}\sum_{m\in{\cal M}} \sigma_m^2\leq \sigma^2. 
\end{align}

Similarly, from the update \eqref{eq.CADA-2}, we have
\begin{align*}
\hat v_i^{k+1}\leq\max\{\hat v_i^k,\beta_2 \hat v_i^k+(1-\beta_2)(\nabla_i^k)^2\}\leq \max\{\hat v_i^k,\beta_2 \hat v_i^k+(1-\beta_2)\sigma^2\}. 
\end{align*}
Since $v_i^1=\hat v_i^1\leq \sigma^2$, if follows by induction that $\hat v_i^{k+1}\leq \sigma^2$. 
\end{proof}

\begin{lemma}\label{lemma4}
	Under Assumption \ref{assump:gradientestimator}, the iterates $\{\theta^k\}$ of CADA in Algorithm \ref{alg: CADA} satisfy
	\begin{equation}
    \left\|\theta^{k+1}-\theta^k\right\|^2\leq \alpha_k^2p(1-\beta_2)^{-1}(1-\beta_3)^{-1}
\end{equation}
where $p$ is the dimension of $\theta$, $\beta_1<\sqrt{\beta_2}<1$, and $\beta_3:=\beta_1^2/\beta_2$.
\end{lemma}
\begin{proof}
Choosing $\beta_1<1$ and defining $\beta_3:=\beta_1^2/\beta_2$, it can be verified that
\begin{align}\label{eq.pflem4-1}
    |h_i^{k+1}|&=\left|\beta_1 h^k_i+(1-\beta_1)\nabla^k_i\right|\beta_1|h_i^k|+|\nabla_i^k|\nonumber\\
    &\leq\beta_1\left(\beta_1|h_i^{k-1}|+|\nabla_i^{k-1}|\right)+|\nabla_i^k|\nonumber\\
   &\leq\sum\limits_{l=0}^k\beta_1^{k-l}|\nabla_i^l|=\sum\limits_{l=0}^k\sqrt{\beta_3}^{k-l}\sqrt{\beta_2}^{k-l}|\nabla_i^l|\nonumber\\
   &\stackrel{(a)}{\leq}\left(\sum\limits_{l=0}^k\beta_3^{k-l}\right)^{\frac{1}{2}}\left(\sum\limits_{l=0}^k\beta_2^{k-l}(\nabla_i^l)^2\right)^{\frac{1}{2}}\nonumber\\
   &\leq(1-\beta_3)^{-\frac{1}{2}}\left(\sum\limits_{l=0}^k\beta_2^{k-l}(\nabla_i^l)^2\right)^{\frac{1}{2}}
\end{align}
where (a) follows from the Cauchy-Schwartz inequality. 

For $\hat v_i^k$, first we have that $\hat v_i^1\geq(1-\beta_2)(\nabla_i^1)^2$. Then since
\begin{align*}
    \hat v_i^{k+1}\geq\beta_2\hat v_i^k+(1-\beta_2)(\nabla_i^k)^2
\end{align*}
by induction we have
\begin{equation}\label{eq.pflem4-2}
    \hat v_i^{k+1}\geq(1-\beta_2)\sum\limits_{l=0}^k\beta_2^{k-l}(\nabla_i^l)^2.
\end{equation}
Using \eqref{eq.pflem4-1} and \eqref{eq.pflem4-2}, we have
\begin{align*}
    |h_i^{k+1}|^2\leq & (1-\beta_3)^{-1}\left(\sum\limits_{l=0}^k\beta_2^{k-l}(\nabla_i^l)^2\right)\\
    \leq & (1-\beta_2)^{-1}(1-\beta_3)^{-1}\hat v_i^{k+1}.
\end{align*}
From the update \eqref{eq.CADA-3}, we have
\begin{align}
\|\theta^{k+1}-\theta^k\|^2&=\alpha_k^2\sum_{i=1}^p\left(\epsilon+\hat v^{k+1}_i\right)^{-1}|h_i^{k+1}|^2\nonumber\\
&\leq \alpha_k^2p(1-\beta_2)^{-1}(1-\beta_3)^{-1}
\end{align}
which completes the proof.

\end{proof}

\section{Missing Derivations in Section \ref{sec.CADA}}
The analysis in this part is analogous to that in \cite{ghadimi2013stochastic}. 
We define an auxiliary function as
\begin{align*}
    \psi_m(\theta)={\cal L}_m(\theta)-{\cal L}_m(\theta^{\star})-\dotp{\nabla{\cal L}_m(\theta^{\star}),\theta-\theta^{\star}}
\end{align*}
where $\theta^{\star}$ is a minimizer of ${\cal L}$. Assume that $\nabla\ell(\theta;\xi_m)$ is $\bar{L}$-Lipschitz continuous for all $\xi_m$, we have
\begin{align*}
    \|\nabla\ell(\theta;\xi_m)-\nabla\ell(\theta^{\star};\xi_m)\|^2\leq 2\bar{L}\left(\ell(\theta;\xi_m)-\ell(\theta^{\star};\xi_m)-\dotp{\nabla\ell(\theta^{\star};\xi_m), \theta-\theta^{\star}}\right).
\end{align*}
Taking expectation with respect to $\xi_m$, we can obtain
\begin{align*}
    \EE_{\xi_m}[\|\nabla\ell(\theta;\xi_m)-\nabla\ell(\theta^{\star};\xi_m)\|^2]\leq2\bar{L}\left({\cal L}_m(\theta)-{\cal L}_m(\theta^{\star})-\dotp{\nabla{\cal L}_m(\theta^{\star}),\theta-\theta^{\star}}\right)=2\bar{L}\psi_m(\theta).
\end{align*}
Note that $\nabla{\cal L}_m$ is also $\bar{L}$-Lipschitz continuous and thus
\begin{align*}
    \|\nabla{\cal L}_m(\theta)-\nabla{\cal L}_m(\theta^{\star})\|^2\leq2\bar{L}({\cal L}_m(\theta)-{\cal L}_m(\theta^{\star})-\dotp{\nabla{\cal L}_m(\theta^{\star}),\theta-\theta^{\star}})=2\bar{L}\psi_m(\theta).
\end{align*}

\subsection{Derivations of \eqref{eqn:variance-wk}} 
By \eqref{eqn:cauchy}, we can derive that
\begin{equation*}
    \|\theta_1+\theta_2\|\leq 2\|\theta_1\|^2+2\|\theta_2\|^2
\end{equation*}
which also implies $\|\theta_1\|^2\geq\frac{1}{2}\|\theta_1+\theta_2\|^2-\|\theta_2\|^2$.


As a consequence, we can obtain
{\small\begin{align*}
&\EE\Big[\big\|\nabla\ell(\theta^k;\xi_m^k)-\nabla\ell(\theta^{k-\tau_m^k};\xi_m^{k-\tau_m^k})\big\|^2\Big]\\
\geq&\frac{1}{2}\EE\Big[\big\|\big(\nabla\ell(\theta^k;\xi_m^k)-\nabla{\cal L}_m(\theta^k)\big)+\big(\nabla{\cal L}_m(\theta^{k-\tau_m^k})-\nabla\ell(\theta^{k-\tau_m^k};\xi_m^{k-\tau_m^k})\big)\big\|^2\Big]\\
&\qquad\qquad\qquad\qquad\qquad\qquad\qquad\qquad\qquad\qquad\qquad\qquad\quad~~-\EE\Big[\big\|\nabla{\cal L}_m(\theta^k)-\nabla{\cal L}_m(\theta^{k-\tau_m^k})\big\|^2\Big]\\
=&\frac{1}{2}\EE\Big[\big\|\nabla\ell(\theta^k;\xi_m^k)-\nabla{\cal L}_m(\theta^k)\big\|^2\Big]+\frac{1}{2}\EE\Big[\big[\big\|\nabla\ell(\theta^{k-\tau_m^k};\xi_m^{k-\tau_m^k})-\nabla{\cal L}_m(\theta^{k-\tau_m^k})\big\|^2\big]\Big]\\
&+\underbracket{\EE\Big[\dotp{\nabla\ell(\theta^k;\xi_m^k)-\nabla{\cal L}_m(\theta^k), \nabla{\cal L}_m(\theta^{k-\tau_m^k})-\nabla\ell(\theta^{k-\tau_m^k};\xi_m^{k-\tau_m^k})}\Big]}_{I_3}-\EE\Big[\big\|\nabla{\cal L}_m(\theta^k)-\nabla{\cal L}_m(\theta^{k-\tau_m^k})\big\|^2\Big]
\end{align*}}
where we used the fact that $I_3=0$ to obtain \eqref{eqn:variance-wk}, that is
\begin{align*}
    I_3= \EE\Big[\Big\langle \EE\big[\nabla\ell(\theta^k;\xi_m^k)\big|\Theta^k\big] -\nabla{\cal L}_m(\theta^k), \nabla{\cal L}_m(\theta^{k-\tau_m^k})-\nabla\ell(\theta^{k-\tau_m^k};\xi_m^{k-\tau_m^k})\Big\rangle\Big]=0.
\end{align*}

\subsection{Derivations of \eqref{eqn:variance-wk1}}
Recall that
{\small\begin{align*}
    \tdelta_m^k-\tdelta_m^{k-\tau_m^k}=&\big(\nabla\ell(\theta^k;\xi_m^k)-\nabla\ell(\tilde{\theta};\xi_m^k)+\nabla{\cal L}_m(\tilde{\theta})\big)-\big(\nabla\ell(\theta^{k-\tau_m^k};\xi_m^{k-\tau_m^k})-\nabla\ell(\tilde{\theta};\xi_m^{k-\tau_m^k})+\nabla{\cal L}_m(\tilde{\theta})\big)\\
    =&\underbracket{\big(\nabla\ell(\theta^k;\xi_m^k)-\nabla\ell(\tilde{\theta};\xi_m^k)+\nabla\psi_m(\tilde{\theta})\big)}_{g_m^k}-\underbracket{\big(\nabla\ell(\theta^{k-\tau_m^k};\xi_m^{k-\tau_m^k})-\nabla\ell(\tilde{\theta};\xi_m^{k-\tau_m^k})+\nabla\psi_m(\tilde{\theta})\big)}_{g_m^{k-\tau_m^k}}.
\end{align*}}
And by \eqref{eqn:cauchy}, we have $\|\tdelta_m^k-\tdelta_m^{k-\tau_m^k}\|^2\leq 2\|g_m^k\|^2+2\|g_m^{k-\tau_m^k}\|^2$. We decompose the first term as 
\begin{align*}
    \EE[\|g_m^k\|^2]\leq&2\EE[\|\nabla\ell(\theta^k;\xi_m^k)-\nabla\ell(\theta^{\star};\xi_m^k)\|^2]+2\EE[\|\nabla\ell(\ttheta;\xi_m^k)-\nabla\ell(\theta^{\star};\xi_m^k)-\nabla\psi_m(\ttheta)\|^2]\\
    =&2\EE[\EE[\|\nabla\ell(\theta^k;\xi_m^k)-\nabla\ell(\theta^{\star};\xi_m^k)\|^2|\Theta^k]]\\&+2\EE[\|\nabla\ell(\ttheta;\xi_m^k)-\nabla\ell(\theta^{\star};\xi_m^k)-\EE[\nabla\ell(\ttheta;\xi_m^k)-\nabla\ell(\theta^{\star};\xi_m^k)|\Theta^k]\|^2]\\
    \leq&4\bar{L}\EE\psi_m(\theta^k)+2\EE[\|\nabla\ell(\ttheta;\xi_m^k)-\nabla\ell(\theta^{\star};\xi_m^k)\|^2]\\
    =&4\bar{L}\EE\psi_m(\theta^k)+2\EE[\EE[\|\nabla\ell(\ttheta;\xi_m^k)-\nabla\ell(\theta^{\star};\xi_m^k)\|^2|\Theta^k]]\\
    \leq&4\bar{L}\EE\psi_m(\theta^k)+4\bar{L}\EE\psi_m(\ttheta).
\end{align*}
By nonnegativity of $\psi_m$, we have
\begin{align}\label{eq.b2-1}
    \EE[\|g_m^k\|^2]&\leq4\bar{L}\sum\limits_{m\in{\cal M}}\EE\psi_m(\theta^k)+4\bar{L}\sum\limits_{m\in{\cal M}}\EE\psi_m(\ttheta)\nonumber\\
    &=4M\bar{L}(\EE{\cal L}(\theta^k)-{\cal L}(\theta^{\star}))+4M\bar{L}(\EE{\cal L}(\ttheta)-{\cal L}(\theta^{\star})).
\end{align}
Similarly, we can prove 
\begin{equation}
\EE[\|g_m^{k-\tau_m^k}\|^2]\leq4M\bar{L}(\EE{\cal L}(\theta^{k-\tau_m^k})-{\cal L}(\theta^{\star}))+4M\bar{L}(\EE{\cal L}(\ttheta)-{\cal L}(\theta^{\star})).
\end{equation}
Therefore, it follows that
\begin{align*}
    \EE[\|\tdelta_m^k&-\tdelta_m^{k-\tau_m^k}\|^2]\\
    &\leq 8M\bar{L}(\EE{\cal L}(\theta^k)-{\cal L}(\theta^{\star}))+8M\bar{L}(\EE{\cal L}(\theta^{k-\tau_m^k})-{\cal L}(\theta^{\star}))+16M\bar{L}(\EE{\cal L}(\ttheta)-{\cal L}(\theta^{\star})).
\end{align*}

\subsection{Derivations of \eqref{eqn:variance-wk2}}
The LHS of \eqref{eqn:workerrule2} can be written as
\begin{align*}
    \nabla\ell(\theta^k;\xi_m^k)-\nabla\ell(\theta^{k-\tau_m^k};\xi_m^k)=&\big(\nabla\ell(\theta^k;\xi_m^k)-\nabla\ell(\theta^{k-\tau_m^k};\xi_m^k)+\nabla{\cal L}_m(\theta^{k-\tau_m^k})\big)-\nabla{\cal L}_m(\theta^{k-\tau_m^k})\\
    =&\big(\nabla\ell(\theta^k;\xi_m^k)-\nabla\ell(\theta^{k-\tau_m^k};\xi_m^k)+\nabla\psi_m(\theta^{k-\tau_m^k})\big)-\nabla\psi_m(\theta^{k-\tau_m^k}).
\end{align*}
Similar to \eqref{eq.b2-1}, we can obtain
\begin{align*}
    \EE[\|\nabla\ell(\theta^k;\xi_m^k) -\nabla\ell(\theta^{k-\tau_m^k};\xi_m^k)&+\nabla\psi_m(\theta^{k-\tau_m^k})\|^2]\\
   & \leq
    4M\bar{L}(\EE{\cal L}(\theta^k)-{\cal L}(\theta^{\star}))+4M\bar{L}(\EE{\cal L}(\theta^{k-\tau_m^k})-{\cal L}(\theta^{\star})).
\end{align*}
Combined with the fact
\begin{align*}
	\EE[\|\nabla\psi_m(\theta^{k-\tau_m^k})\|^2]&=\EE[\|\nabla{\cal L}_m(\theta^{k-\tau_m^k})-\nabla{\cal L}_m(\theta^{\star})\|^2]\\	
	&\leq 2\bar{L}\EE\psi_m(\theta^{k-\tau_m^k})\leq 2M\bar{L}(\EE{\cal L}(\theta^{k-\tau_m^k})-{\cal L}(\theta^{\star}))
\end{align*}
we have
\begin{align*}
    \EE[\|\nabla\ell(\theta^k;\xi_m^k)-\nabla\ell(\theta^{k-\tau_m^k};\xi_m^k)\|^2]\leq 8M\bar{L}(\EE{\cal L}(\theta^k)-{\cal L}(\theta^{\star}))+12M\bar{L}(\EE{\cal L}(\theta^{k-\tau_m^k})-{\cal L}(\theta^{\star})).
\end{align*}

\section{Proof of Lemma \ref{lemma:lossdescent}}

Using the smoothness of $\cL(\theta)$ in Assumption \ref{assump:smoothness}, we have
\begin{align}\label{eqn:lossdescent1}
\cL(\theta^{k+1}) \leq &\, \cL(\theta^k)+\dotp{\nabla\cL(\theta^k),\theta^{k+1}-\theta^k}+\frac{L}{2}\big\|\theta^{k+1}-\theta^k\big\|^2 \nonumber\\
    = & \,\cL(\theta^k)-\alpha_k\dotp{\nabla\cL(\theta^k),(\epsilon I+\hat V^{k+1})^{-\frac{1}{2}}h^{k+1}}+\frac{L}{2}\big\|\theta^{k+1}-\theta^k\big\|^2.
\end{align}

We can further decompose the inner product as
\begin{align}\label{eqn:decomp1}
&- \dotp{\nabla\cL(\theta^k), (\epsilon I+\hat V^{k+1})^{-\frac{1}{2}}h^{k+1}}\nonumber\\
	=&-(1-\beta_1) \dotp{\nabla\cL(\theta^k),(\epsilon I+\hat V^k)^{-\frac{1}{2}}\bm\nabla^k}\underbracket{-\beta_1 \dotp{\nabla\cL(\theta^k),(\epsilon I+\hat V^k)^{-\frac{1}{2}}h^k}}_{I_1^k}\nonumber\\
	& \underbracket{-\dotp{\nabla\cL(\theta^k),\left((\epsilon I+\hat V^{k+1})^{-\frac{1}{2}}-(\epsilon I+\hat V^k)^{-\frac{1}{2}}\right)h^{k+1}}}_{I_2^k}
\end{align}
where we again decompose the first inner product as
\begin{align}\label{eqn:decomp2}
- (1-\beta_1) \dotp{\nabla\cL(\theta^k),(\epsilon I&+\hat V^k)^{-\frac{1}{2}}\bm\nabla^k}=\underbracket{-(1-\beta_1) \dotp{\nabla\cL(\theta^k),(\epsilon I+\hat V^{k-D})^{-\frac{1}{2}}\bm\nabla^k}}_{I_3^k}\nonumber\\
	&\underbracket{-(1-\beta_1)\dotp{\nabla\cL(\theta^k),\left((\epsilon I+\hat V^k)^{-\frac{1}{2}}-(\epsilon I+\hat V^{k-D})^{-\frac{1}{2}}\right)\bm\nabla^k}}_{I_4^k}.
\end{align}

Next, we bound the terms $I_1^k, I_2^k, I_3^k, I_4^k$ separately. 

Taking expectation on $I_1^k$ conditioned on $\Theta^k$, we have
 \begin{align}\label{eq.pf.I_1}
 \EE[I_1^k\mid \Theta^k]&=-\EE\left[\beta_1 \dotp{\nabla\cL(\theta^k),(\epsilon I+\hat V^k)^{-\frac{1}{2}}h^k}\mid \Theta^k\right]\nonumber\\
 &=	- \beta_1 \dotp{\nabla\cL(\theta^{k-1}),(\epsilon I+\hat V^k)^{-\frac{1}{2}}h^k} -\beta_1 \dotp{\nabla\cL(\theta^k)-\nabla\cL(\theta^{k-1}),(\epsilon I+\hat V^k)^{-\frac{1}{2}}h^k}\nonumber\\
 &  \stackrel{(a)}{\leq} - \beta_1 \dotp{\nabla\cL(\theta^{k-1}),(\epsilon I+\hat V^k)^{-\frac{1}{2}}h^k}+\alpha_{k-1}^{-1}\beta_1 L\big\|\theta^k-\theta^{k-1}\big\|^2\nonumber\\
 & \stackrel{(b)}{\leq}  \beta_1 \left(I_1^{k-1}+ I_2^{k-1}+ I_3^{k-1}+ I_4^{k-1}\right)+\alpha_{k-1}^{-1}\beta_1 L\big\|\theta^k-\theta^{k-1}\big\|^2
 \end{align}
where follows from the $L$-smoothness of $\cL(\bbtheta)$ implied by Assumption \ref{assump:smoothness}; and (b) uses again the decomposition \eqref{eqn:decomp1} and \eqref{eqn:decomp2}. 

Taking expectation on $I_2^k$ over all the randomness, we have
\begin{align}\label{eq.pf.I_2}
 \EE[I_2^k ]=&\EE\Big[-\dotp{\nabla\cL(\theta^k),\left((\epsilon I+\hat V^{k+1})^{-\frac{1}{2}}-(\epsilon I+\hat V^k)^{-\frac{1}{2}}\right)h^{k+1}}\Big]\nonumber\\
       =&\EE\Big[\sum_{i=1}^p \nabla_i\cL(\theta^k)h_i^{k+1}\Big((\epsilon+\hat v_i^k)^{-\frac{1}{2}}-(\epsilon+\hat v_i^{k+1})^{-\frac{1}{2}}\Big)\Big]\nonumber\\
       \stackrel{(d)}{\leq}&\EE\Big[\|\nabla\cL(\theta^k)\| \|h^{k+1}\|\sum_{i=1}^p\Big((\epsilon+\hat v_i^k)^{-\frac{1}{2}}-(\epsilon+\hat v_i^{k+1})^{-\frac{1}{2}}\Big)\Big]\nonumber\\
        \stackrel{(e)}{\leq}&\sigma^2\EE\Big[\sum_{i=1}^p\Big((\epsilon+\hat v_i^k)^{-\frac{1}{2}}-(\epsilon+\hat v_i^{k+1})^{-\frac{1}{2}}\Big)\Big]
 \end{align}
where (d) follows from the Cauchy-Schwarz inequality and (e) is due to Assumption \ref{assump:gradientestimator}.  

Regarding $I_3^k$, we will bound separately in Lemma \ref{lemma6}. 


Taking expectation on $I_4^k$ over all the randomness, we have
\begin{align}\label{eq.pf.I_4}
 \EE[I_4^k ]=&\EE\Big[-(1-\beta_1)\dotp{\nabla\cL(\theta^k),\left((\epsilon I+\hat V^k)^{-\frac{1}{2}}-(\epsilon I+\hat V^{k-D})^{-\frac{1}{2}}\right)\bm\nabla^k}\Big]\nonumber\\
       =&-(1-\beta_1)\EE\Big[\sum_{i=1}^p \nabla_i\cL(\theta^k)\bm\nabla^k_i\Big((\epsilon+\hat v_i^k)^{-\frac{1}{2}}-(\epsilon+\hat v_i^{k-D})^{-\frac{1}{2}}\Big)\Big]\nonumber\\
       \leq &(1-\beta_1)\EE\Big[\|\nabla\cL(\theta^k)\| \|\bm\nabla^k\|\sum_{i=1}^p\Big((\epsilon+\hat v_i^{k-D})^{-\frac{1}{2}}-(\epsilon+\hat v_i^k)^{-\frac{1}{2}}\Big)\Big]\nonumber\\
        \leq &(1-\beta_1)\sigma^2\EE\Big[\sum_{i=1}^p\Big((\epsilon+\hat v_i^{k-D})^{-\frac{1}{2}}-(\epsilon+\hat v_i^k)^{-\frac{1}{2}}\Big)\Big].
 \end{align}

Taking expectation on \eqref{eqn:lossdescent1} over all the randomness, and plugging \eqref{eq.pf.I_1}, \eqref{eq.pf.I_2}, and \eqref{eq.pf.I_4}, we have
\begin{align}\label{eqn:lossdescent2}
\EE[\cL(\theta^{k+1})]-\EE[\cL(\theta^k)] \leq &\,-\alpha_k\EE\left[\dotp{\nabla\cL(\theta^k),(\epsilon I+\hat V^{k+1})^{-\frac{1}{2}}h^{k+1}}\right]+\frac{L}{2}\EE\left[\big\|\theta^{k+1}-\theta^k\big\|^2\right]\nonumber\\
= &\, \alpha_k\EE\left[I_1^k+I_2^k+I_3^k+I_4^k\right]+\frac{L}{2}\EE\left[\big\|\theta^{k+1}-\theta^k\big\|^2\right]\nonumber\\
\leq &-\alpha_k(1-\beta_1)\EE\left[\dotp{\nabla\cL(\theta^k),(\epsilon I+\hat V^{k-D})^{-\frac{1}{2}}\bm\nabla^k}\right]\nonumber\\
     &-\alpha_k\beta_1\EE\left[ \dotp{\nabla\cL(\theta^{k-1}),(\epsilon I+\hat V^k)^{-\frac{1}{2}}h^k}\right]\nonumber\\
     &+\alpha_k\sigma^2\EE\Big[\sum_{i=1}^p\Big((\epsilon+\hat v_i^k)^{-\frac{1}{2}}-(\epsilon+\hat v_i^{k+1})^{-\frac{1}{2}}\Big)\Big]\nonumber\\
     &+\alpha_k(1-\beta_1)\sigma^2\EE\Big[\sum_{i=1}^p\Big((\epsilon+\hat v_i^{k-D})^{-\frac{1}{2}}-(\epsilon+\hat v_i^k)^{-\frac{1}{2}}\Big)\Big]\nonumber\\
     &+\left(\frac{L}{2}+\alpha_k\alpha_{k-1}^{-1}\beta_1 L\right)\EE\left[\|\theta^{k+1}-\theta^{k}\|^2\right].
\end{align}

Since $(\epsilon+\hat v_i^k)^{-\frac{1}{2}}\leq (\epsilon+\hat v_i^{k-1})^{-\frac{1}{2}}$, we have
\begin{align}\label{eqn:lossdescent2-2}
	     & \sigma^2\EE\Big[\sum_{i=1}^p\!\Big(\!(\epsilon+\hat v_i^k)^{-\frac{1}{2}}\!-\!(\epsilon+\hat v_i^{k+1})^{-\frac{1}{2}}\Big)\!+\!(1-\beta_1)\sum_{i=1}^p\!\Big((\epsilon+\hat v_i^{k-D})^{-\frac{1}{2}}\!-\!(\epsilon+\hat v_i^k)^{-\frac{1}{2}}\Big)\Big]\nonumber\\
     \leq & (2-\beta_1)\sigma^2\EE\Big[\sum_{i=1}^p\Big((\epsilon+\hat v_i^{k-D})^{-\frac{1}{2}}-(\epsilon+\hat v_i^{k+1})^{-\frac{1}{2}}\Big)\Big].
\end{align}
Plugging \eqref{eqn:lossdescent2-2} into \eqref{eqn:lossdescent2} leads to the statement of Lemma \ref{lemma:lossdescent}.

\section{Proof of Lemma \ref{lemma6}}
We first analyze the inner produce under CADA2 and then CADA1. 

First recall that {\small$\bar{\bm\nabla}^k=\frac{1}{M}\!\sum_{m\in{\cal M}}\nabla  \ell(\theta^k; \xi_m^k)$}. Using the law of total probability implies that 
\begin{align}\label{eqn:decomp3-0}
	\EE\left[\dotp{\nabla\cL(\theta^k),(\epsilon I+\hat V^{k-D})^{-\frac{1}{2}}\bar{\bm\nabla}^k}\right]&=\EE\left[\EE\left[\dotp{\nabla\cL(\theta^k),(\epsilon I+\hat V^{k-D})^{-\frac{1}{2}}\bar{\bm\nabla}^k}\mid \Theta^k \right]\right]\nonumber\\
	&=\EE\left[\dotp{\nabla\cL(\theta^k),(\epsilon I+\hat V^{k-D})^{-\frac{1}{2}}\EE\left[\bar{\bm\nabla}^k\mid \Theta^k \right]}\right]\nonumber\\
   &=\EE\left[\left\|\nabla\cL(\theta^k)\right\|^2_{(\epsilon I+\hat V^{k-D})^{-\frac{1}{2}}} \right].
\end{align}

Taking expectation on $\dotp{\nabla\cL(\theta^k),(\epsilon I+\hat V^{k-D})^{-\frac{1}{2}}\bm\nabla^k}$ over all randomness, we have
\begin{align}\label{eqn:decomp3-1}
 &- \EE\left[\dotp{\nabla\cL(\theta^k),(\epsilon I+\hat V^{k-D})^{-\frac{1}{2}}\bm\nabla^k}\right]\nonumber\\
 =& - \EE\left[\dotp{\nabla\cL(\theta^k),(\epsilon I+\hat V^{k-D})^{-\frac{1}{2}}\bar{\bm\nabla}^k}\right]\nonumber\\
& - \EE\Big[\dotp{\nabla\cL(\theta^k),(\epsilon I+\hat V^{k-D})^{-\frac{1}{2}}\frac{1}{M}\sum_{m\in{\cal 	M}} \left(\nabla\ell(\theta^{k-\tau_m^k}; \xi_m^{k-\tau_m^k})-\nabla  \ell(\theta^k; \xi_m^k)\right)}\!\Big] \nonumber\\
  \stackrel{(a)}{=}& -\EE\left[\left\|\nabla\cL(\theta^k)\right\|^2_{(\epsilon I+\hat V^{k-D})^{-\frac{1}{2}}} \right]\nonumber\\
 & -\frac{1}{M}\sum_{m\in{\cal M}}\EE\Big[\dotp{\nabla\cL(\theta^k),(\epsilon I+\hat V^{k-D})^{-\frac{1}{2}} \left(\nabla\ell(\theta^{k-\tau_m^k}; \xi_m^{k-\tau_m^k})-\nabla  \ell(\theta^k; \xi_m^k)\right)}\!\Big]
\end{align}
where (a) uses \eqref{eqn:decomp3-0}.

Decomposing the inner product, for the CADA2 rule \eqref{eqn:workerrule2}, we have 
\begin{align}\label{eqn:decomp3-1-2}
& - \EE\Big[\dotp{\nabla\cL(\theta^k),(\epsilon I+\hat V^{k-D})^{-\frac{1}{2}}\left(\nabla\ell(\theta^{k-\tau_m^k}; \xi_m^{k-\tau_m^k})-\nabla  \ell(\theta^k; \xi_m^k)\right)}\!\Big] \nonumber\\
 =&  - \EE\Big[\dotp{\nabla\cL(\theta^k),(\epsilon I+\hat V^{k-D})^{-\frac{1}{2}}\left(\nabla\ell(\theta^{k-\tau_m^k}; \xi_m^{k-\tau_m^k})-\nabla  \ell(\theta^{k-\tau_m^k}; \xi_m^k)\right)\!}\!\Big] \nonumber\\
 &  - \EE\Big[\dotp{\nabla\cL(\theta^k),(\epsilon I+\hat V^{k-D})^{-\frac{1}{2}}\left(\nabla\ell(\theta^{k-\tau_m^k}; \xi_m^k)-\nabla  \ell(\theta^k; \xi_m^k)\right)}\Big]\nonumber\\
  \stackrel{(b)}{\leq}&  \frac{L\epsilon^{-\frac{1}{2}}}{12\alpha_k} \sum\limits_{d=1}^{D} \EE\left[\|\theta^{k+1-d}-\theta^{k-d}\|^2\right]+ {6 D L\alpha_k \epsilon^{-\frac{1}{2}}} \sigma_m^2\nonumber\\
  & \!-  \EE\Big[\dotp{\!\nabla\cL(\theta^k),(\epsilon I+\hat V^{k-D})^{-\frac{1}{2}} \!\left(\nabla\ell(\theta^{k-\tau_m^k}; \xi_m^k)-\nabla  \ell(\theta^k; \xi_m^k)\right)\!\!}\Big]\!\!
\end{align}
where (b) follows from Lemma \ref{lemma2}. 

Using the Young's inequality, we can bound the last inner product in \eqref{eqn:decomp3-1-2} as
\begin{align}\label{eqn:decomp3-2}
	&-   \EE\Big[\dotp{\nabla\cL(\theta^k),(\epsilon I+\hat V^{k-D})^{-\frac{1}{2}}\!\left(\nabla\ell(\theta^{k-\tau_m^k}; \xi_m^k)-\nabla  \ell(\theta^k; \xi_m^k)\right)\!}\Big]\nonumber\\
	\leq & \frac{1 }{2 }\EE\Big[ \Big\|\nabla\cL(\theta^k)\Big\|^2_{(\epsilon I+\hat V^{k-D})^{-\frac{1}{2}}}\Big] \!+\frac{1 }{2}\EE\Big[\Big\|(\epsilon I+\hat V^{k-D})^{-\frac{1}{2}}\Big\|\Big\| \!\left(\nabla\ell(\theta^{k-\tau_m^k}; \xi_m^k)-\nabla  \ell(\theta^k; \xi_m^k)\right)\!\Big\|^2\Big]\nonumber\\
	 \stackrel{(g)}{\leq} &\frac{1 }{2 }\EE\Big[ \Big\|\nabla\cL(\theta^k)\Big\|^2_{(\epsilon I+\hat V^{k-D})^{-\frac{1}{2}}}\Big] +\frac{1}{2}\EE\Big[ \Big\|(\epsilon I+\hat V^{k-D})^{-\frac{1}{2}}\Big\|\Big\|\nabla\ell(\theta^{k-\tau_m^k}; \xi_m^k)-\nabla  \ell(\theta^k; \xi_m^k)\Big\|^2\Big]\nonumber\\
	 \stackrel{(h)}{\leq}& \frac{1 }{2 }\EE\Big[ \Big\|\nabla\cL(\theta^k)\Big\|^2_{(\epsilon I+\hat V^{k-D})^{-\frac{1}{2}}}\Big] +\frac{c}{2d_{\rm max}} \EE\Big[\Big\|(\epsilon I+\hat V^{k-D})^{-\frac{1}{2}}\Big\|\sum\limits_{d=1}^{d_{\rm max}}\left\|\theta^{k+1-d}\!\!-\theta^{k-d}\right\|^2\Big]\nonumber\\
	 \stackrel{(i)}{\leq} & \frac{1 }{2 }\EE\Big[ \Big\|\nabla\cL(\theta^k)\Big\|^2_{(\epsilon I+\hat V^{k-D})^{-\frac{1}{2}}}\Big] +\frac{c  \epsilon^{-\frac{1}{2}}}{2d_{\rm max}}\sum\limits_{d=1}^D\EE\Big[\left\|\theta^{k+1-d}\!\!-\theta^{k-d}\right\|^2\Big]
\end{align}
where (g) follows from the Cauchy-Schwarz inequality, and (h) uses the adaptive communication condition \eqref{eqn:workerrule2} in CADA2, and (i) follows since $\hat V^{k-D}$ is entry-wise nonnegative and $\left\|\theta^{k+1-d}\!\!-\theta^{k-d}\right\|^2$ is nonnegative. 

Similarly for CADA1's condition \eqref{eqn:workerrule1}, we have
\begin{align}\label{eqn:decomp3-3}
& - \EE\Big[\dotp{\nabla\cL(\theta^k),(\epsilon I+\hat V^{k-D})^{-\frac{1}{2}}\left(\nabla\ell(\theta^{k-\tau_m^k}; \xi_m^{k-\tau_m^k})-\nabla  \ell(\theta^k; \xi_m^k)\right)}\!\Big] \nonumber\\
 =&  - \EE\Big[\dotp{\nabla\cL(\theta^k),(\epsilon I+\hat V^{k-D})^{-\frac{1}{2}}\left(\nabla\ell(\tilde\theta; \xi_m^{k-\tau_m^k})-\nabla\ell(\tilde\theta; \xi_m^k)\right)\!}\!\Big] \nonumber\\
 &  - \EE\Big[\dotp{\nabla\cL(\theta^k),(\epsilon I+\hat V^{k-D})^{-\frac{1}{2}}\left(\tilde\delta^{k-\tau_m^k}_m-\tilde\delta^k_m)\right)}\Big]\nonumber\\
  \stackrel{(j)}{\leq}&  \frac{L\epsilon^{-\frac{1}{2}}}{12\alpha_k} \sum\limits_{d=1}^{D} \EE\left[\|\theta^{k+1-d}-\theta^{k-d}\|^2\right]+ {6 D L\alpha_k \epsilon^{-\frac{1}{2}}} \sigma_m^2  \!-  \EE\Big[\dotp{\!\nabla\cL(\theta^k),(\epsilon I+\hat V^{k-D})^{-\frac{1}{2}} \!\left(\tilde\delta_m^{k-\tau_m^k}-\tilde\delta_m^k\right)\!\!}\Big]
\end{align}
where (j) follows from Lemma \ref{lemma2} since $\tilde\theta$ is a snapshot among $\{\theta^k, \cdots, \theta^{k-D}\}$. 

And the last product in \eqref{eqn:decomp3-3} is bounded by
\begin{align}\label{eqn:decomp3-4}
	&-   \EE\Big[\dotp{\nabla\cL(\theta^k),(\epsilon I+\hat V^{k-D})^{-\frac{1}{2}}\!\left(\tilde\delta_m^{k-\tau_m^k}-\tilde\delta_m^k\right)\!}\Big]\nonumber\\
	 \leq& \frac{1 }{2 }\EE\Big[ \Big\|\nabla\cL(\theta^k)\Big\|^2_{(\epsilon I+\hat V^{k-D})^{-\frac{1}{2}}}\Big] +\frac{c}{2} \EE\Big[\Big\|(\epsilon I+\hat V^{k-D})^{-\frac{1}{2d_{\rm max}}}\Big\|\sum\limits_{d=1}^{d_{\rm max}}\left\|\theta^{k+1-d}\!\!-\theta^{k-d}\right\|^2\Big]\nonumber\\
	 \stackrel{(i)}{\leq} & \frac{1 }{2 }\EE\Big[ \Big\|\nabla\cL(\theta^k)\Big\|^2_{(\epsilon I+\hat V^{k-D})^{-\frac{1}{2}}}\Big] +\frac{c  \epsilon^{-\frac{1}{2}}}{2d_{\rm max}}\sum\limits_{d=1}^D\EE\Big[\left\|\theta^{k+1-d}\!\!-\theta^{k-d}\right\|^2\Big].
\end{align}
Combining \eqref{eqn:decomp3-1}-\eqref{eqn:decomp3-4} leads to the desired statement for CADA1 and CADA2.


\section{Proof of Lemma \ref{lemma:lyapunovdescent}}
For notational brevity, we re-write the Lyapunov function \eqref{eqn:Lyapunov} as 
\begin{align}\label{eqn:Lyapunov3}
    {\cal V}^k:= \cL(\theta^k)-\cL(\theta^{\star})&-c_k\left\langle \nabla \cL(\theta^{k-1}), (\epsilon I+\hat V^k)^{-\frac{1}{2}}h^k\right\rangle\nonumber\\
    &+b_k\sum\limits_{d=0}^{D}\sum_{i=1}^p(\epsilon+\hat v_i^{k-d})^{-\frac{1}{2}}+\sum\limits_{d=1}^{D}\rho_d\|\theta^{k+1-d}-\theta^{k-d}\|^2	
\end{align}
where $\{c_k\}$ are some positive constants.

Therefore, taking expectation on the difference of ${\cal V}^k$ and ${\cal V}^{k+1}$ in \eqref{eqn:Lyapunov3}, we have (with $\rho_{D+1}=0$)
\begin{align}\label{eqn:lossdescent3}
   \EE[ {\cal V}^{k+1}]-\EE[{\cal V}^k]
    =&\EE[\cL(\theta^{k+1})]-\EE[\cL(\theta^k)]-c_{k+1}\EE\left[\left\langle \nabla \cL(\theta^k), (\epsilon I+\hat V^{k+1})^{-\frac{1}{2}}h^{k+1}\right\rangle\right]\nonumber\\
 &   +c_k\EE\left[\left\langle \nabla \cL(\theta^{k-1}), (\epsilon I+\hat V^k)^{-\frac{1}{2}}h^k\right\rangle\right]\nonumber\\
 &+b_{k+1}\sum\limits_{d=0}^{D}\sum_{i=1}^p(\epsilon+\hat v_i^{k+1-d})^{-\frac{1}{2}}-b_k\sum\limits_{d=0}^{D}\sum_{i=1}^p(\epsilon+\hat v_i^{k-d})^{-\frac{1}{2}}\nonumber\\
     &+\rho_1\EE\left[\|\theta^{k+1}-\theta^{k}\|^2\right]+\sum\limits_{d=1}^D(\rho_{d+1}-\rho_d)\EE\left[\|\theta^{k+1-d}-\theta^{k-d}\|^2\right]\nonumber\\
     \stackrel{(a)}{\leq} & (\alpha_k+c_{k+1})\EE\left[I_1^k+I_2^k+I_3^k+I_4^k\right]   -c_k\EE\left[I_1^{k-1}+I_2^{k-1}+I_3^{k-1}+I_4^{k-1}\right]\nonumber\\
          &+b_{k+1} \sum_{i=1}^p\EE\Big[(\epsilon+\hat v_i^{k+1})^{-\frac{1}{2}}\Big]-b_k \sum_{i=1}^p\EE\Big[(\epsilon+\hat v_i^{k-D})^{-\frac{1}{2}}\Big]\nonumber\\
     &+\sum\limits_{d=1}^{D}(b_{k+1}-b_k)\sum_{i=1}^p\EE\Big[(\epsilon+\hat v_i^{k+1-d})^{-\frac{1}{2}}\Big]+\left(\frac{L}{2}+\rho_1\right)\EE\left[\|\theta^{k+1}-\theta^{k}\|^2\right]\nonumber\\
     &+\sum\limits_{d=1}^D(\rho_{d+1}-\rho_d)\EE\left[\|\theta^{k+1-d}-\theta^{k-d}\|^2\right]
\end{align}
where (a) uses the smoothness in Assumption \ref{assump:smoothness} and the definition of $I_1^k, I_2^k, I_3^k, I_4^k$ in \eqref{eqn:decomp1} and \eqref{eqn:decomp2}.

Note that we can bound $(\alpha_k+c_{k+1})\EE\left[I_1^k+I_2^k+I_3^k+I_4^k\right]$ the same as \eqref{eqn:decomp1} in the proof of Lemma \ref{lemma:lossdescent}.
In addition, Lemma \ref{lemma6} implies that
\begin{align}\label{eq.pf.I_3}
	\EE[I_3^k]\leq &- \frac{1-\beta_1}{2} \EE\left[\left\|\nabla\cL(\theta^k)\right\|^2_{(\epsilon I+\hat V^{k-D})^{-\frac{1}{2}}} \right]\nonumber\\
 & +(1-\beta_1)\epsilon^{-\frac{1}{2}}\left(\frac{L}{12\alpha_k}+\frac{c }{2d_{\rm max}}\right) \sum\limits_{d=1}^{D} \EE\left[\|\theta^{k+1-d}-\theta^{k-d}\|^2\right]\!+\! (1-\beta_1)\frac{6 D L\alpha_k \epsilon^{-\frac{1}{2}}}{M}\!\!\sum_{m\in{\cal 	M}}\sigma_m^2.
 \end{align}
 

Hence, plugging Lemma \ref{lemma:lossdescent} with $\alpha_k$ replaced by $\alpha_k+c_{k+1}$ into \eqref{eqn:lossdescent3}, together with \eqref{eq.pf.I_3}, leads to
\begin{align}\label{eqn:lossdescent4}
   \EE[ {\cal V}^{k+1}]-\EE[{\cal V}^k]
     \leq &-(\alpha_k+c_{k+1})\left(\frac{1-\beta_1}{2}\right)\EE\left[\left\|\nabla\cL(\theta^k)\right\|^2_{(\epsilon I+\hat V^{k-D})^{-\frac{1}{2}}} \right]\nonumber\\
     & +(\alpha_k+c_{k+1})(1-\beta_1)\epsilon^{-\frac{1}{2}}\left(\frac{L}{12\alpha_k}+\frac{c }{2d_{\rm max}}\right) \sum\limits_{d=1}^{D} \EE\left[\|\theta^{k+1-d}-\theta^{k-d}\|^2\right]\!\nonumber\\
     &+(\alpha_k+c_{k+1})(1-\beta_1)\frac{6 D L\alpha_k \epsilon^{-\frac{1}{2}}}{M}\sum_{m\in{\cal 	M}}\sigma_m^2\nonumber\\
     &+((\alpha_k+c_{k+1})\beta_1-c_k)\EE\left[I_1^{k-1}+I_2^{k-1}+I_3^{k-1}+I_4^{k-1}\right]\nonumber\\
     &+(\alpha_k+c_{k+1})(2-\beta_1)\sigma^2\EE\Big[\sum_{i=1}^p\Big((\epsilon+\hat v_i^{k-D})^{-\frac{1}{2}}-(\epsilon+\hat v_i^{k+1})^{-\frac{1}{2}}\Big)\Big]\nonumber\\
    &+b_{k+1} \sum_{i=1}^p\EE\Big[(\epsilon+\hat v_i^{k+1})^{-\frac{1}{2}}\Big]-b_k \sum_{i=1}^p\EE\Big[(\epsilon+\hat v_i^{k-D})^{-\frac{1}{2}}\Big]\nonumber\\
     &+\sum\limits_{d=1}^{D}(b_{k+1}-b_k)\sum_{i=1}^p\EE\Big[(\epsilon+\hat v_i^{k+1-d})^{-\frac{1}{2}}\Big] +\sum\limits_{d=1}^D(\rho_{d+1}-\rho_d)\EE\left[\|\theta^{k+1-d}-\theta^{k-d}\|^2\right]\nonumber\\
     &+\left(\frac{L}{2}+\rho_1+(\alpha_k+c_{k+1})\alpha_{k-1}^{-1}\beta_1 L\right)\EE\left[\|\theta^{k+1}-\theta^{k}\|^2\right].
\end{align}

Select $\alpha_k\leq \alpha_{k-1}$ and $c_k:=\sum\limits_{j=k}^{\infty}\alpha_j\beta_1^{j-k+1}\leq (1-\beta_1)^{-1}\alpha_k$ so that
$(\alpha_k+c_{k+1})\beta_1=c_k$ and 
\begin{align*}
(\alpha_k+c_{k+1})(1-\beta_1)&\leq (\alpha_k+ (1-\beta_1)^{-1}\alpha_{k+1})(1-\beta_1)\\
&\leq \alpha_k(1+ (1-\beta_1)^{-1})(1-\beta_1)=\alpha_k(2-\beta_1).
\end{align*}
In addition, select $b_k$ to ensure that $b_{k+1}\leq  b_k$. Then it follows from \eqref{eqn:lossdescent4} that
\begin{align}\label{eqn:lossdescent5}
  \EE[ {\cal V}^{k+1}]-\EE[{\cal V}^k]
     \leq & -\frac{\alpha_k(1-\beta_1)}{2}\EE\left[\left\|\nabla\cL(\theta^k)\right\|^2_{(\epsilon I+\hat V^{k-D})^{-\frac{1}{2}}} \right]+(2-\beta_1)\alpha_k^2\frac{6 D L \epsilon^{-\frac{1}{2}}}{M}\sum_{m\in{\cal 	M}}\sigma_m^2\nonumber\\
     & +(2-\beta_1)\alpha_k\epsilon^{-\frac{1}{2}}\left(\frac{L}{12\alpha_k}+\frac{c }{2d_{\rm max}}\right) \sum\limits_{d=1}^{D} \EE\left[\|\theta^{k+1-d}-\theta^{k-d}\|^2\right]\nonumber\\
          &+\left( \frac{(2-\beta_1)^2}{(1-\beta_1)}\alpha_k\sigma^2-b_k\right)\EE\Big[\sum_{i=1}^p\Big((\epsilon+\hat v_i^{k-D})^{-\frac{1}{2}}-(\epsilon+\hat v_i^{k+1})^{-\frac{1}{2}}\Big)\Big]\nonumber\\
     &+\left(\frac{L}{2}+\rho_1+(1-\beta_1)^{-1}  L\right)\EE\left[\|\theta^{k+1}-\theta^{k}\|^2\right]\nonumber\\
     &+\sum\limits_{d=1}^D(\rho_{d+1}-\rho_d)\EE\left[\|\theta^{k+1-d}-\theta^{k-d}\|^2\right]
\end{align}
where we have also used the fact that $-(\alpha_k+c_{k+1})\left(\frac{1-\beta_1}{2}\right)\leq -\frac{\alpha_k(1-\beta_1)}{2}$ since $c_{k+1}\geq 0$.

If we choose $\alpha_k\leq \frac{1}{L}$ for $k=1,2\ldots, K$, then it follows from \eqref{eqn:lossdescent5} that
 \begin{align}\label{eqn:thm1-1}
&   \EE[ {\cal V}^{k+1}]-\EE[{\cal V}^k]\nonumber\\
     \leq & -\frac{\alpha_k(1-\beta_1)}{2}\!\left(\epsilon+\frac{\sigma^2}{1-\beta_2}\right)^{-\frac{1}{2}}\!\!\EE\left[\left\|\nabla\cL(\theta^k)\right\|^2 \right]+(2-\beta_1)\frac{6 \alpha_k^2D L \epsilon^{-\frac{1}{2}}}{M}\!\!\sum_{m\in{\cal M}}\!\!\sigma_m^2\nonumber\\
  &+\underbracket{\left(\frac{(2-\beta_1)^2}{(1-\beta_1)}\alpha_k\sigma^2-b_k\right)}_{A^k}\EE\Big[\sum_{i=1}^p\Big((\epsilon+\hat v_i^{k-D})^{-\frac{1}{2}}-(\epsilon+\hat v_i^{k+1})^{-\frac{1}{2}}\Big)\Big]\nonumber\\
       &+\left(\frac{L}{2}+\rho_1+(1-\beta_1)^{-1}  L\right)\EE\left[\|\theta^{k+1}-\theta^{k}\|^2\right]\nonumber\\
     &+\sum\limits_{d=1}^D\underbracket{\left((2-\beta_1)\epsilon^{-\frac{1}{2}}\left(\frac{L}{12}+\frac{c \alpha_k}{2d_{\rm max}}\right)+\rho_{d+1}-\rho_d\right)}_{B_d^k}\EE\left[\|\theta^{k+1-d}-\theta^{k-d}\|^2\right].
\end{align}

To ensure $A^k\leq 0$ and $B_d^k\leq 0$, it is sufficient to choose $\{b_k\}$ and $\{\rho_d\}$ satisfying (with $\rho_{D+1}=0$)
\begin{align}
&\frac{(2-\beta_1)^2}{(1-\beta_1)}\alpha_k\sigma^2-b_k\leq 0,\quad k=1,\cdots,K\\
&(2-\beta_1) \epsilon^{-\frac{1}{2}} \left(\frac{L}{12}+\frac{c \alpha_k}{2d_{\rm max}}\right)+\rho_{d+1}-\rho_d\leq 0,\quad d=1,\cdots,D.
\end{align}

Solve this system of linear equations and get
\begin{align}
&b_k=\frac{(2-\beta_1)^2}{(1-\beta_1)L} \sigma^2,\quad k=1,\cdots,K\\
    &\rho_d=(2-\beta_1)\epsilon^{-\frac{1}{2}}\left(\frac{L}{12}+\frac{c}{2Ld_{\rm max}}\right)(D-d+1),~~~ d=1,\cdots,D
\end{align}
plugging which into \eqref{eqn:thm1-1} leads to the conclusion of Lemma \ref{lemma:lyapunovdescent}.
%

\section{Proof of Theorem \ref{thm:nonconvex}}
From the definition of ${\cal V}^k$, we have for any $k$, that 
\begin{align}\label{eqn:thm1-2}
    \EE[{\cal V}^k]& \geq \cL(\theta^k)-\cL(\theta^*)-c_k\left\langle \nabla \cL(\theta^{k-1}), (\epsilon I+\hat V^k)^{-\frac{1}{2}}h^k\right\rangle+\sum\limits_{d=1}^{D}\rho_d\|\theta^{k+1-d}-\theta^{k-d}\|^2	\nonumber\\
    & \geq -|c_k|\left\|\nabla \cL(\theta^{k-1})\right\| \left\|(\epsilon I+\hat V^k)^{-\frac{1}{2}}h^k\right\|\nonumber \\ &\geq -(1-\beta_1)^{-1}\alpha_k\sigma^2 \epsilon^{-\frac{1}{2}}
\end{align}
where we use Assumption \ref{assump:gradientestimator} and Lemma \ref{lemma3}.

By taking summation on \eqref{eqn:thm1-1} over $k=0,\cdots, K-1$, it follows from that
\begin{align}\label{eqn:thm1-3}
&\frac{\alpha(1-\beta_1)}{2}\left(\epsilon+\frac{\sigma^2}{1-\beta_2}\right)^{-\frac{1}{2}}\frac{1}{K}\sum\limits_{k=1}^K\EE\left[\left\|\nabla\cL(\theta^k)\right\|^2 \right]  	\nonumber\\
     \leq & \frac{\EE[ {\cal V}^1]-\EE[{\cal V}^{K+1}]}{K}+(2-\beta_1)\frac{6 \alpha^2D L \epsilon^{-\frac{1}{2}}}{M}\sum_{m\in{\cal 	M}}\sigma_m^2+\frac{(2-\beta_1)^2}{(1-\beta_1)}  \sigma^2pD\epsilon^{-\frac{1}{2}}\frac{\alpha}{K}\nonumber\\
     &+\left(\frac{L}{2}+\rho_1+(1-\beta_1)^{-1}  L\right)\frac{1}{K}\sum\limits_{k=1}^K\EE\left[\|\theta^{k+1}-\theta^{k}\|^2\right]\nonumber\\
    \stackrel{(a)}{\leq} & \frac{\EE[ {\cal V}^1]}{K}+(2-\beta_1)\frac{6 \alpha^2D L \epsilon^{-\frac{1}{2}}}{M}\sum_{m\in{\cal 	M}}\sigma_m^2+(1-\beta_1)^{-1}\sigma^2 \epsilon^{-\frac{1}{2}}\frac{\alpha}{K}+ \frac{(2-\beta_1)^2}{(1-\beta_1)}  \sigma^2pD\epsilon^{-\frac{1}{2}}\frac{\alpha}{K}\nonumber\\
     &+\left(\frac{L}{2}+\rho_1+(1-\beta_1)^{-1}  L\right) p(1-\beta_2)^{-1}(1-\beta_3)^{-1}\alpha^2
\end{align}
where (a) follows from \eqref{eqn:thm1-2} and Lemma \ref{lemma4}.

Specifically, if we choose a constant stepsize $\alpha:= \frac{\eta}{\sqrt{K}}$, 
where $\eta>0$ is a constant, and define 
\begin{equation}
\tilde{C}_1:=(2-\beta_1)6D  L \epsilon^{-\frac{1}{2}}
\end{equation}
and
\begin{equation}
\tilde{C}_2:=	(1-\beta_1)^{-1} \epsilon^{-\frac{1}{2}}+ \frac{(2-\beta_1)^2}{(1-\beta_1)}D\epsilon^{-\frac{1}{2}}
\end{equation}
and 
\begin{equation}
\tilde{C}_3:=	\left(\frac{L}{2}+\rho_1+(1-\beta_1)^{-1}  L\right)(1-\beta_2)^{-1}(1-\beta_3)^{-1}
\end{equation}
and 
\begin{equation}\label{eq.def-tc4}
\tilde{C}_4:=\frac{1}{2}(1-\beta_1)\left(\epsilon+\frac{\sigma^2}{1-\beta_2}\right)^{-\frac{1}{2}}
\end{equation}
we can obtain from \eqref{eqn:thm1-3} that
\begin{align*}
    \frac{1}{K}\sum\limits_{k=0}^{K-1}\EE\left[\|\nabla\cL(\theta^{k})\|^2\right]\leq&\frac{\frac{\cL(\theta^0)-\cL(\theta^*)}{K}+\frac{\tilde{C}_1 }{M}\sum_{m\in{\cal 	M}}\sigma_m^2 \alpha^2+\tilde{C}_2p\sigma^2\frac{\alpha}{K}+\tilde{C}_3p\alpha^2}{ \alpha \tilde{C}_4}\\
    \leq&\frac{\cL(\theta^0)-\cL(\theta^*)}{K\alpha \tilde{C}_4}+\frac{\tilde{C}_1 \alpha}{\tilde{C}_4M}\sum_{m\in{\cal 	M}}\sigma_m^2 +\tilde{C}_2p\frac{\sigma^2}{K\tilde{C}_4}+\frac{\tilde{C}_3p\alpha}{\tilde{C}_4}\\
   = &\frac{(\cL(\theta^0)-\cL(\theta^*))C_4}{\sqrt{K}\eta}+\frac{C_1 \eta}{ \sqrt{K} M}\sum_{m\in{\cal 	M}}\sigma_m^2 +\frac{C_2p\sigma^2}{K}+\frac{C_3p\eta}{\sqrt{K}}
\end{align*}
where we define $C_1:=\tilde{C}_1/\tilde{C}_4$, $C_2:=\tilde{C}_2/\tilde{C}_4$, $C_3:=\tilde{C}_3/\tilde{C}_4$, and $C_4:=1/\tilde{C}_4$. 

\section{Proof of Theorem \ref{thm:stronglyconvex}}

By the PL-condition of $\cL(\theta)$, we have
{\small\begin{align}\label{eqn:strongconvexity}
&-\frac{\alpha_k(1-\beta_1)}{2}\!\left(\epsilon+\frac{\sigma^2}{1-\beta_2}\right)^{-\frac{1}{2}}\!\!\EE\left[\left\|\nabla\cL(\theta^k)\right\|^2 \right]\nonumber\\
\leq & -\alpha_k\mu(1-\beta_1)\!\left(\epsilon+\frac{\sigma^2}{1-\beta_2}\right)^{-\frac{1}{2}}\!\!\EE\left[  \cL(\theta^k)-\cL(\theta^{\star}) \right]\nonumber\\
\stackrel{(a)}{\leq}  &\!-2\alpha_k\mu\tilde{C}_4\Big(\EE[{\cal V}^k] \!+\!c_k\left\langle \nabla \cL(\theta^{k-1}), (\epsilon I+\hat V^k)^{-\frac{1}{2}}h^k\right\rangle \!-\!b_k\sum\limits_{d=0}^{D}\sum_{i=1}^p\!(\epsilon+\hat v_i^{k-d})^{-\frac{1}{2}}-\sum\limits_{d=1}^{D}\rho_d\|\theta^{k+1-d}-\theta^{k-d}\|^2 \Big)\nonumber\\
\stackrel{(b)}{\leq}  &-2\alpha_k\mu\tilde{C}_4\EE[{\cal V}^k]+2\alpha_k^2\mu\tilde{C}_4(1-\beta_1)^{-1}\sigma^2 \epsilon^{-\frac{1}{2}}+2\alpha_k\mu\tilde{C}_4b_k\sum\limits_{d=0}^{D}\sum_{i=1}^p\EE\left[(\epsilon+\hat v_i^{k-d})^{-\frac{1}{2}}\right]\nonumber\\
&+2\alpha_k\mu\tilde{C}_4\sum\limits_{d=1}^{D}\rho_d\EE[\|\theta^{k+1-d}-\theta^{k-d}\|^2]
\end{align}}
where (a) uses the definition of $\tilde{C}_4$ in \eqref{eq.def-tc4}, and (b) uses Assumption \ref{assump:gradientestimator} and Lemma \ref{lemma3}. 

Plugging \eqref{eqn:strongconvexity} into \eqref{eqn:lossdescent5}, we have
 \begin{align}\label{eqn:thm2-2}
 \EE[ {\cal V}^{k+1}]&-\EE[{\cal V}^k] 
     \leq -2\alpha_k\mu\tilde{C}_4\EE[{\cal V}^k]+(2-\beta_1)\frac{6 \alpha_k^2D L \epsilon^{-\frac{1}{2}}}{M}\!\!\sum_{m\in{\cal M}}\!\!\sigma_m^2\\
     &+\frac{(2-\beta_1)^2}{(1-\beta_1)}\alpha_k\sigma^2\EE\Big[\sum_{i=1}^p\Big((\epsilon+\hat v_i^{k-D})^{-\frac{1}{2}}-(\epsilon+\hat v_i^{k+1})^{-\frac{1}{2}}\Big)\Big]\nonumber\\
     &+b_{k+1} \sum_{i=1}^p\EE\Big[(\epsilon+\hat v_i^{k+1})^{-\frac{1}{2}}\Big]-(b_k-2\alpha_k\mu\tilde{C}_4b_k) \sum_{i=1}^p\EE\Big[(\epsilon+\hat v_i^{k-D})^{-\frac{1}{2}}\Big]\nonumber\\
     &+\sum\limits_{d=1}^{D}(b_{k+1}-b_k+2\alpha_k\mu\tilde{C}_4b_k)\sum_{i=1}^p\EE\Big[(\epsilon+\hat v_i^{k+1-d})^{-\frac{1}{2}}\Big]\nonumber\\
     &+\left(\frac{L}{2}+\rho_1+(1-\beta_1)^{-1}  L\right)p(1-\beta_2)^{-1}(1-\beta_3)^{-1}\alpha_k^2+2\alpha_k^2\mu\tilde{C}_4(1-\beta_1)^{-1}\sigma^2 \epsilon^{-\frac{1}{2}}\nonumber\\
     &+\sum\limits_{d=1}^D\!\left((2-\beta_1)\epsilon^{-\frac{1}{2}}\!\left(\frac{L}{12}\!+\!\frac{c \alpha_k}{2d_{\rm max}}\right)+\rho_{d+1}-\rho_d+2\alpha_k\mu\tilde{C}_4\rho_d\right)\!\EE\left[\|\theta^{k+1-d}-\theta^{k-d}\|^2\right].\nonumber
\end{align}
If we choose $b_k$ to ensure that $b_{k+1}\leq (1-2\alpha_k\mu\tilde{C}_4) b_k$, then we can obtain from \eqref{eqn:thm2-2} that
 \begin{align}\label{eqn:thm2-3}
&\EE[ {\cal V}^{k+1}]-\EE[{\cal V}^k]\\
     \leq & -2\alpha_k\mu\tilde{C}_4\EE[{\cal V}^k]+\frac{\tilde{C}_1 }{M}\sum_{m\in{\cal 	M}}\sigma_m^2 \alpha_k^2+\tilde{C}_3p\alpha_k^2+2\mu\tilde{C}_4(1-\beta_1)^{-1}\sigma^2 \epsilon^{-\frac{1}{2}}\alpha_k^2\nonumber\\
     &+\left( \frac{(2-\beta_1)^2}{(1-\beta_1)}\alpha_k\sigma^2-(1-2\alpha_k\mu\tilde{C}_4) b_k\right)\EE\Big[\sum_{i=1}^p\Big((\epsilon+\hat v_i^{k-D})^{-\frac{1}{2}}-(\epsilon+\hat v_i^{k+1})^{-\frac{1}{2}}\Big)\Big]\nonumber\\
     &+\sum\limits_{d=1}^D\!\left((2-\beta_1)\epsilon^{-\frac{1}{2}}\!\left(\frac{L}{12}\!+\!\frac{c \alpha_k}{2d_{\rm max}}\right)+\rho_{d+1}-\rho_d+2\alpha_k\mu\tilde{C}_4\rho_d\right)\!\EE\left[\|\theta^{k+1-d}-\theta^{k-d}\|^2\right].\nonumber
\end{align}
If $\alpha_k\leq \frac{1}{L}$, we choose parameters $\{b_k, \rho_d\}$ to guarantee that
\begin{align}
&	 \frac{(2-\beta_1)^2}{(1-\beta_1)L}\sigma^2-\Big(1-\frac{2\mu\tilde{C}_4}{L}\Big) b_k\leq 0,~~~\forall k\\
	&(2-\beta_1)\!\left(\frac{L}{12}\!+\!\frac{c}{2Ld_{\rm max}}\right)\epsilon^{-\frac{1}{2}}+\rho_{d+1}-\Big(1-\frac{2\mu\tilde{C}_4}{L}\Big)\rho_d\leq 0,~~~d=1,\cdots, D
\end{align}
and choose $\beta_1, \beta_2, \epsilon$ to ensure that $1-\frac{2\mu\tilde{C}_4}{L}\geq 0$.

Then we have
\begin{align}\label{eq.pf-thm2-4}
\EE[ {\cal V}^{k+1}]\leq &\left(1-2\alpha_k\mu\tilde{C}_4\right)\EE[{\cal V}^k]+\Bigg(\underbracket{\frac{\tilde{C}_1 }{M}\sum_{m\in{\cal 	M}}\sigma_m^2+\tilde{C}_3p+2\mu\tilde{C}_4(1-\beta_1)^{-1}\sigma^2 \epsilon^{-\frac{1}{2}}}_{\tilde{C}_5}\Bigg)\alpha_k^2\nonumber\\
    \leq &\prod\limits_{j=0}^k(1-2\alpha_j\mu\tilde{C}_4)\EE[{\cal V}^0]+\sum\limits_{j=0}^k \alpha_j^2\prod\limits_{i=j+1}^k(1-2\alpha_i\mu\tilde{C}_4)\tilde{C}_5.
\end{align}

%

If we choose $\alpha_k=\frac{1}{\mu(k+K_0)\tilde{C}_4}\leq \frac{1}{L}$, where $K_0$ is a sufficiently large constant to ensure that $\alpha_k$ satisfies the aforementioned conditions, then we have
\begin{align*}
\EE[{\cal V}^{K}]\leq&\EE[ {\cal V}^0]\prod\limits_{k=0}^{K-1}(1-2\alpha_k\mu\tilde{C}_4)+\tilde{C}_5\sum\limits_{k=0}^{K-1}\alpha_k^2\prod\limits_{j=k+1}^{K-1}(1-2\alpha_j\mu\tilde{C}_4)\\
    \leq&\EE[{\cal V}^0]\prod\limits_{k=0}^{K-1}\frac{k+K_0-2}{k+K_0}+\frac{\tilde{C}_5}{\mu^2\tilde{C}_4^2}\sum\limits_{k=0}^{K-1}\frac{1}{(k+K_0)^2}\prod\limits_{j=k+1}^{K-1}\frac{j+K_0-2}{j+K_0}\\
    \leq&\frac{(K_0-2)(K_0-1)}{(K+K_0-2)(K+K_0-1)}\EE[{\cal V}^0]+\frac{\tilde{C}_5}{\mu^2\tilde{C}_4^2}\sum\limits_{k=0}^{K-1} \frac{(k+K_0-1) }{(k+K_0)(K+K_0-2)(K+K_0-2)}\\
    \leq&\frac{(K_0-1)^2}{(K+K_0-1)^2}\EE[{\cal V}^0]+\frac{\tilde{C}_5K}{\mu^2\tilde{C}_4^2(K+K_0-1)^2}\\
    =&\frac{(K_0-1)^2}{(K+K_0-1)^2}(\cL(\theta^0)-\cL(\theta^{\star}))+\frac{\tilde{C}_5K}{\mu^2\tilde{C}_4^2(K+K_0-2)^2}
\end{align*}
from which the proof is complete. 

\section{Additional Numerical Results}
\subsection{Simulation setup}
In order to verify our analysis and show the empirical performance of CADA, we conduct experiments in the logistic regression and training neural network tasks, respectively. 

In logistic regression, we tested the \textbf{covtype} and \textbf{ijcnn1} in the main paper, and \textbf{MNIST} in the supplementary document. In training neural networks, we tested \textbf{MNIST} dataset in the main paper, and \textbf{CIFAR10} in the supplementary document. To benchmark CADA, we compared it with some state-of-the-art algorithms, namely ADAM \cite{kingma2014adam}, stochastic LAG, local momentum \cite{yu2019lcml,wang2020iclr} and FedAdam \cite{reddi2020adaptive}. 

All experiments are run on a workstation with an Intel i9-9960x CPU with 128GB memory and four NVIDIA RTX 2080Ti GPUs each with 11GB memory using Python 3.6.

\subsection{Simulation details}
\subsubsection{Logistic regression.} 
\textbf{Objective function.}
For the logistic regression task, we use either the logistic loss for the binary case, or the cross-entropy loss for the multi-class class, both of which are augmented with an $\ell_2$ norm regularizer with the coefficient $\lambda=10^{-5}$.

\textbf{Data pre-processing.}
 For \textit{ijcnn1} and \textit{covtype} datasets, they are imported from the popular library LIBSVM (https://www.csie.ntu.edu.tw/~cjlin/libsvm/) without further preprocessing. 
 For \textit{MNIST}, we normalize the data and subtract the mean. 
We uniformly partition \textit{ijcnn1} dataset with 91,701 samples and \textit{MNIST} dataset with 60,000 samples into $M=10$ workers. 
 To simulate the heterogeneous setting, we partition \textit{covtype} dataset with 581,012 samples randomly into $M=20$ workers with different number of samples per worker.  
 
 For \textit{covtype}, we fix the batch ratio to be 0.001 uniformly across all workers; and for \textit{ijcnn1} and \textit{MNIST}, we fix the batch ratio to be 0.01 uniformly across all workers. 
 
\textbf{Choice of hyperparameters.}
For the logistic regression task, the hyperparameters in each algorithm are chosen by hand to roughly optimize the training loss performance of each algorithm. We list the values of parameters used in each test in Tables \ref{tab:covtype}-\ref{tab:ijcnn}.

\setlength{\tabcolsep}{1mm}{
\begin{table}[H]
\begin{center}
 \begin{tabular}{ c || c |c | c  }
\hline \hline
 \textbf{Algorithm}
     &~~~\textbf{stepsize} $\alpha$~~~&~~~\textbf{momentum weight} $\beta$~~~& ~~~\textbf{averaging interval} $H/D$~~~\\ \hline \hline
    FedAdam& $\alpha_{l}=100$ $\alpha_{s}=0.02$ &$0.9$  & $H=10$    \\ \hline
    Local momentum & $0.1$ & $0.9$ & $H=10$  \\ \hline

   ADAM  & $ 0.005 $ &$\beta _{1}=0.9$ $\beta_{2}=0.999$  & /\\ \hline
    CADA1\&2 & $ 0.005 $  & $\beta _{1}=0.9$ $\beta_{2}=0.999$   &$D=100$,\quad$d_{\rm max} = 10$   \\ \hline
Stochastic    LAG & $0.1$ & /&$d_{\rm max}=10$    \\ 
    \hline\hline
    \end{tabular}
\end{center}
\vspace{-0.2cm}
\caption{Choice of parameters in \textit{covtype}.}
\label{tab:covtype}
\vspace{-0.2cm}
\end{table}
}

\setlength{\tabcolsep}{1mm}{
\begin{table}[H]
\begin{center}
 \begin{tabular}{ c ||c |c |c }
\hline \hline
 \textbf{Algorithm}
     &~~~\textbf{stepsize} $\alpha$~~~&~~~\textbf{momentum weight} $\beta$~~~& ~~~\textbf{averaging interval} $H/D$~~~\\ \hline \hline
    FedAdam &$\alpha_{l}=100$ $\alpha_{s}=0.03$& $0.9$ &$H=10$ \\ \hline
    Local momentum  &$0.1$& $0.9$ &$H=20$\\ \hline

   {ADAM} & $0.01$ & $\beta _{1}=0.9$ $\beta_{2}=0.999$  &/\\ \hline
 {CADA} &$0.01$&$\beta _{1}=0.9$ $\beta_{2}=0.999$  & $D=100$,\quad $d_{\rm max}=10$  \\ \hline
Stochastic    LAG  &$0.1$& / &$d_{\rm max} = 10$\\ 
    \hline\hline
    \end{tabular}
\end{center}
\vspace{-0.1cm}
\caption{Choice of parameters in \textit{ijcnn1}.}
\label{tab:ijcnn}
\vspace{-0.2cm}
\end{table}
}

\subsubsection{Training neural networks.}
For training neural networks, we use the cross-entropy loss but with different network models.

\textbf{Neural network models.}
For \textit{MNIST} dataset, we use a convolutional neural network with two convolution-ELUmaxpooling layers (ELU is a smoothed ReLU) followed by two fully-connected layers. The first convolution layer is $5\times5\times20$ with padding, and the second layer is $5\times5\times50$ with padding. The output of second layer is followed by two fully connected layers with one being $800\times500$ and the other being $500\times10$. The output goes through a softmax function. 
For \textit{CIFAR10} dataset, we use the popular neural network architecture \emph{ResNet20} \footnote{https://github.com/akamaster/pytorch\_resnet\_cifar10} which has 20 and roughly 0.27 million parameters. We do not use a pre-trained model.


\textbf{Data pre-processing.}
We uniformly partition \textit{MNIST} and \textit{CIFAR10} datasets into $M=10$ workers. 
For \textit{MNIST}, we use the raw data without preprocessing. The minibatch size per worker is 12. 
For \textit{CIFAR10}, in addition to normalizing the data and subtracting the mean, we randomly flip and crop part of the original image every time it is used for training. This is a standard technique of data augmentation to avoid over-fitting.  The minibatch size for \textit{CIFAR10} is 50 per worker. 

\textbf{Choice of hyperparameters.}
For \textit{MNIST} dataset which is relatively easy, the hyperparameters in each algorithm are chosen by hand to optimize the performance of each algorithm. We list the values of parameters used in each test in Table \ref{tab:mMNIST}. 

\begin{table}[H]
\begin{center}
 \begin{tabular}{ c || c |c | c  }
\hline \hline
 \textbf{Algorithm}
 &~~~\textbf{stepsize} $\alpha$~~~&~~~\textbf{momentum weight} $\beta$~~~& ~~~\textbf{averaging interval} $H/D$~~~\\ \hline \hline
     FedAdam& $\alpha_{l}=0.1$ $\alpha_{s}=0.001$ &$0.9$  & $H=8$  \\ \hline
    Local momentum & $0.001$ & $0.9$ & $H=8$  \\ \hline
     {ADAM} & ${0.0005}$ &$\beta _{1}=0.9$ $\beta_{2}=0.999$ & / \\ \hline
   {CADA1\&2}& ${0.0005}$  & $\beta _{1}=0.9$ $\beta_{2}=0.999$   &$D=50, d_{\rm max}=10$  \\ \hline
Stochastic    LAG & $0.1$ & / & $d_{\rm max} = 10$ \\ \hline\hline
    \end{tabular}
\end{center}
\caption{Choice of parameters in multi-class \textit{MNIST}.}
\label{tab:mMNIST}
\end{table}

For \textit{CIFAR10} dataset, we search the best values of hyperparameters from the following search grid on a per-algorithm basis to optimize the testing accuracy versus the number of communication rounds. The chosen values of parameter are listed in Table \ref{tab:cifar10param}. 


 \textbf{FedAdam:}   $\alpha_s\in \left\{0.1,0.01, 0.001\right\}$; $\alpha_l\in \left\{1, 0.5, {0.1}\right\}$; $H \in \left\{1, 4, 6,8, 16\right\}$.

 \textbf{Local momentum:}   $\alpha\in \left\{ {0.1},0.01, 0.001\right\}$; $H \in \left\{1, 4, 6,8, 16\right\}$.


\textbf{CADA1:}
    $\alpha\in \left\{ 0.1,{0.01}, 0.001\right\}$; $c \in \left\{ 0.05,0.1,0.3,0.6,0.9,1.2,{1.5},1.8 \right\}$.

\textbf{CADA2:}
    $\alpha\in \left\{ 0.1,{0.01}, 0.001\right\}$; $c \in \left\{ 0.05,0.1,{0.3},0.6,0.9,1.2,1.5,1.8 \right\}$.

\textbf{LAG:} $\alpha\in \left\{ {0.1},0.01, 0.001\right\}$; $c \in \left\{ {0.05},0.1,0.3,0.6,0.9,1.2,1.5,1.8\right\}$.

\begin{figure*}[t]
\vspace{-0.1cm}
\centering
    \includegraphics[width=.4\textwidth]{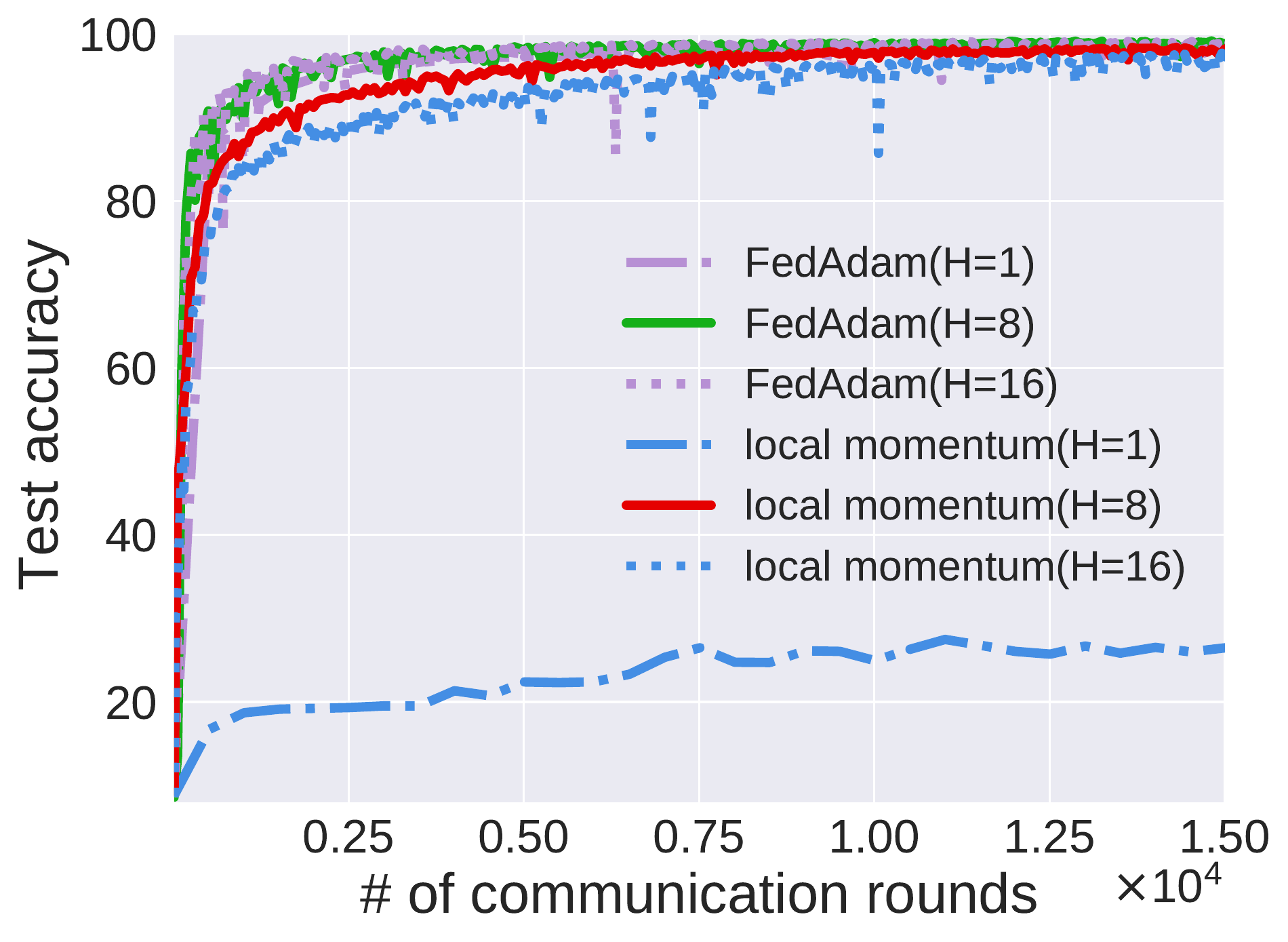}
    \hspace*{2ex}
    \includegraphics[width=.4\textwidth]{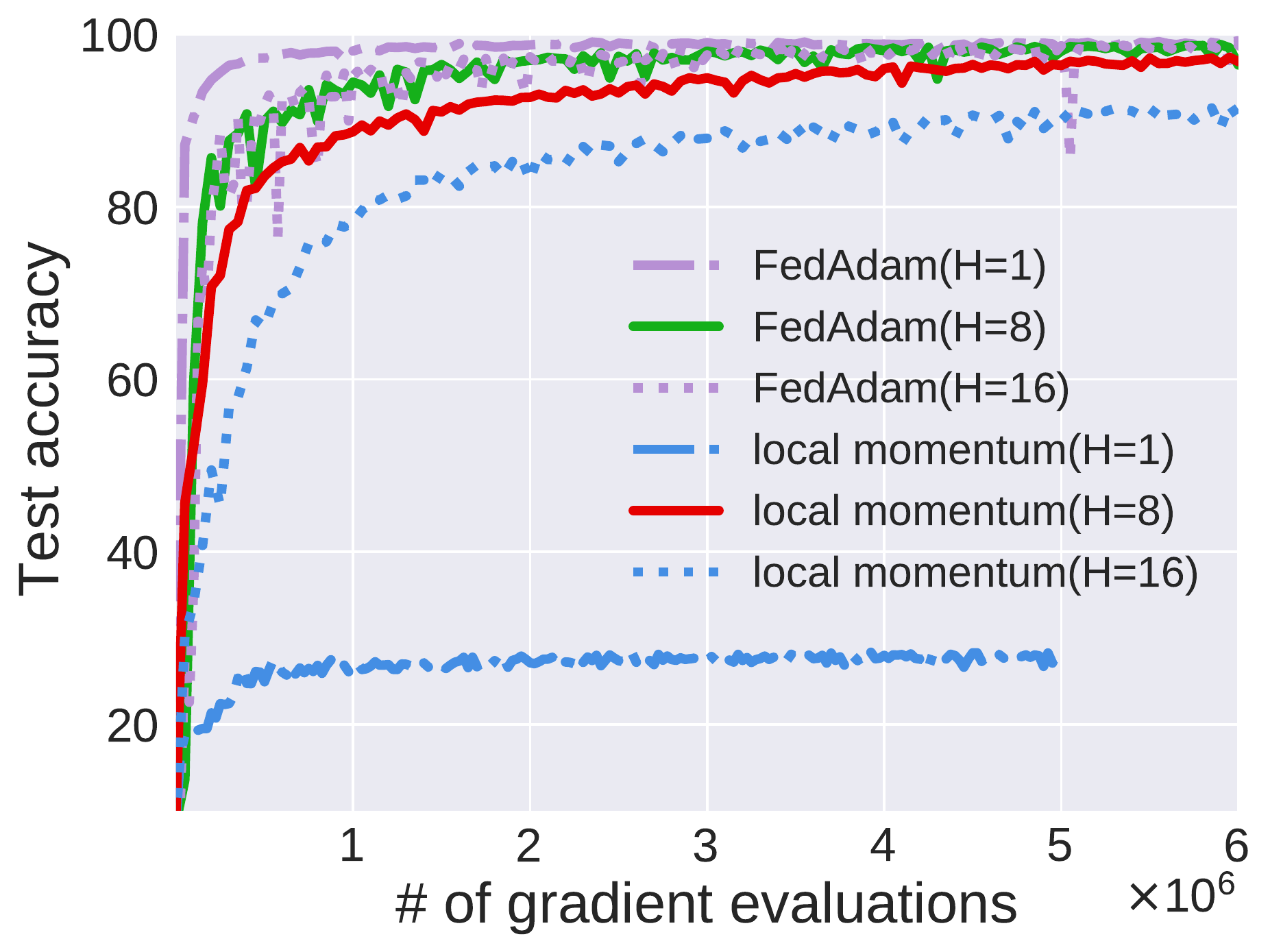}
    \caption{Performance of FedAdam and local momentum on \textit{MNIST} under different $H$.}
    \label{supp-fig:mnist}
\end{figure*}

\begin{figure*}[h!]
\vspace{-0.1cm}
\centering
    \includegraphics[width=.4\textwidth]{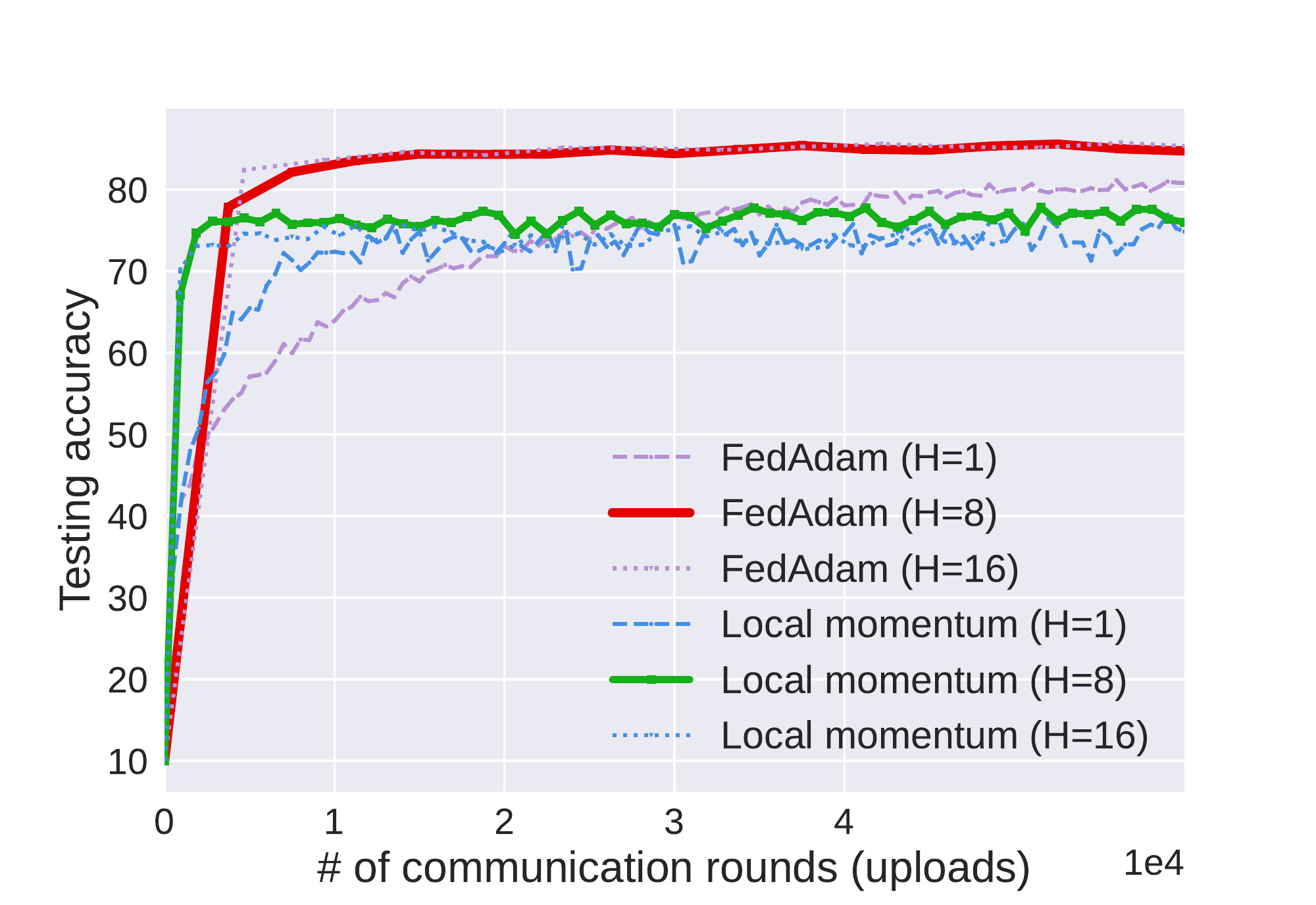}
    \hspace*{2ex}
    \includegraphics[width=.4\textwidth]{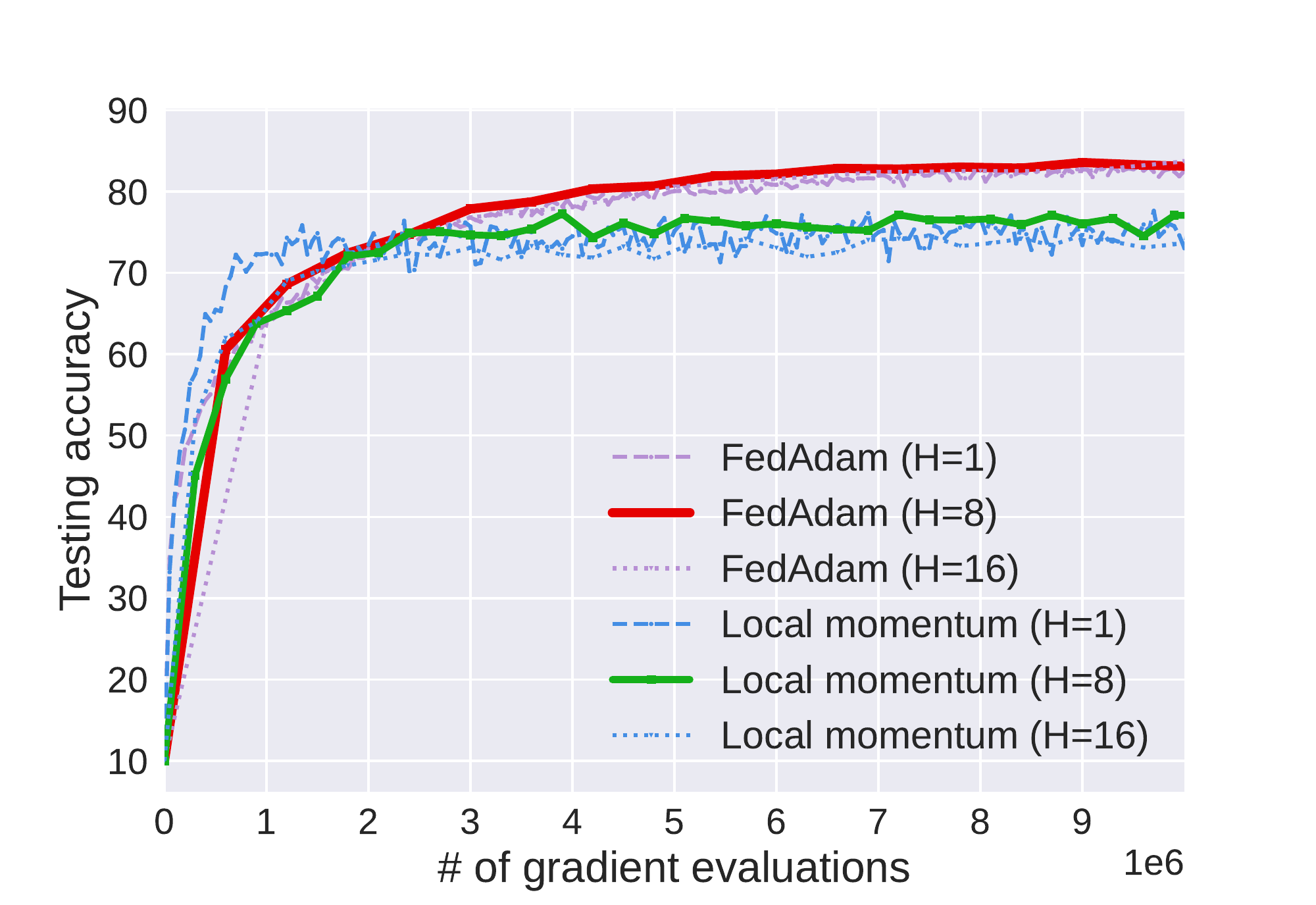}
    \caption{Performance of FedAdam and local momentum on \textit{CIFAR10} under different $H$.}
    \label{supp-fig:cifar10}
\end{figure*}

\begin{table}[H]
\begin{center}
 \begin{tabular}{ c || c |c | c  }
\hline \hline
 \textbf{Algorithm}
 &~~~\textbf{stepsize} $\alpha$~~~&~~~\textbf{momentum weight} $\beta$~~~& ~~~\textbf{averaging interval} $H/D$~~~\\ \hline \hline
     FedAdam& $\alpha_{l}=0.1$ $\alpha_{s}=0.1$ &$0.9$  & $H=8$  \\ \hline
    Local momentum & $0.1$ & $0.9$ & $H=8$  \\ \hline
      CADA1 & $ 0.1$ &$\beta _{1}=0.9$ $\beta_{2}=0.99$  & $D=50$, \quad$d_{\rm max}=2$\\ \hline
   CADA2 & $ 0.1$  & $\beta _{1}=0.9$ $\beta_{2}=0.99$   &$D=50$, \quad$d_{\rm max}= 2$   \\ \hline
   Stochastic LAG & $0.1$ & / & $d_{\rm max} = 2$ \\ \hline\hline
    \end{tabular}
\end{center}
\vspace{-0.2cm}
\caption{Choice of parameters in \textit{CIFAR10}.}
\label{tab:cifar10param}
\end{table}


\textbf{Additional results.}
In addition to the results presented in the main paper, we report a new set of simulations on the performance of local update based algorithms under different averaging interval $H$. Since algorithms under $H=4, 6$ do not perform as good as $H=8$, we only plot $H=1, 8, 16$ in Figures \ref{supp-fig:mnist} and \ref{supp-fig:cifar10} to ease the comparison.
Figure \ref{supp-fig:mnist}  compares the performance of FedAdam and local momentum on \textit{MNIST} dataset under different averaging interval $H$. Figure \ref{supp-fig:cifar10} compares the performance of FedAdam and local momentum on \textit{CIFAR10} dataset under different $H$. 

Figure \ref{supp-fig:cifar10}  compares the performance of FedAdam and local momentum on \textit{CIFAR10} dataset under different averaging interval $H$. 
FedAdam and local momentum under a larger averaging interval $H$ have faster convergence speed at the initial stage, but they reach slightly lower testing accuracy.  This reduced test accuracy is common among local SGD-type methods, which has also been studied theoretically; see e.g., \cite{haddadpour2019local}.



\begin{thebibliography}{61}
\providecommand{\natexlab}[1]{#1}
\providecommand{\url}[1]{\texttt{#1}}
\expandafter\ifx\csname urlstyle\endcsname\relax
  \providecommand{\doi}[1]{doi: #1}\else
  \providecommand{\doi}{doi: \begingroup \urlstyle{rm}\Url}\fi

\bibitem[Aji and Heafield(2017)]{aji2017sparse}
Alham~Fikri Aji and Kenneth Heafield.
\newblock Sparse communication for distributed gradient descent.
\newblock In \emph{Proc. Conf. Empirical Methods Natural Language Process.},
  pages 440--445, Copenhagen, Denmark, Sep 2017.

\bibitem[Alistarh et~al.(2017)Alistarh, Grubic, Li, Tomioka, and
  Vojnovic]{alistarh2017qsgd}
Dan Alistarh, Demjan Grubic, Jerry Li, Ryota Tomioka, and Milan Vojnovic.
\newblock {QSGD: Communication-efficient SGD} via gradient quantization and
  encoding.
\newblock In \emph{Proc. Conf. Neural Info. Process. Syst.}, pages 1709--1720,
  Long Beach, CA, Dec 2017.

\bibitem[Alistarh et~al.(2018)Alistarh, Hoefler, Johansson, Konstantinov,
  Khirirat, and Renggli]{alistarh2018}
Dan Alistarh, Torsten Hoefler, Mikael Johansson, Nikola Konstantinov, Sarit
  Khirirat, and C{\'e}dric Renggli.
\newblock The convergence of sparsified gradient methods.
\newblock In \emph{Proc. Conf. Neural Info. Process. Syst.}, pages 5973--5983,
  Montreal, Canada, Dec 2018.

\bibitem[Bernstein et~al.(2018)Bernstein, Wang, Azizzadenesheli, and
  Anandkumar]{bernstein2018icml}
Jeremy Bernstein, Yu-Xiang Wang, Kamyar Azizzadenesheli, and Animashree
  Anandkumar.
\newblock {SignSGD: C}ompressed optimisation for non-convex problems.
\newblock In \emph{Proc. Intl. Conf. Machine Learn.}, pages 559--568,
  Stockholm, Sweden, Jul 2018.

\bibitem[Bottou et~al.(2018)Bottou, Curtis, and Nocedal]{bottou2016}
L{\'e}on Bottou, Frank~E Curtis, and Jorge Nocedal.
\newblock Optimization methods for large-scale machine learning.
\newblock \emph{Siam Review}, 60\penalty0 (2):\penalty0 223--311, 2018.

\bibitem[Chen et~al.(2018)Chen, Giannakis, Sun, and Yin]{chen2018lag}
Tianyi Chen, Georgios Giannakis, Tao Sun, and Wotao Yin.
\newblock {LAG: L}azily aggregated gradient for communication-efficient
  distributed learning.
\newblock In \emph{Proc. Conf. Neural Info. Process. Syst.}, pages 5050--5060,
  Montreal, Canada, Dec 2018.

\bibitem[Chen et~al.(2020)Chen, Sun, and Yin]{chen2020lasg}
Tianyi Chen, Yuejiao Sun, and Wotao Yin.
\newblock {LASG: L}azily aggregated stochastic gradients for
  communication-efficient distributed learning.
\newblock \emph{arXiv preprint:2002.11360}, 2020.

\bibitem[Chen et~al.(2019)Chen, Liu, Sun, and Hong]{chen2019adam}
Xiangyi Chen, Sijia Liu, Ruoyu Sun, and Mingyi Hong.
\newblock On the convergence of a class of {Adam}-type algorithms for
  non-convex optimization.
\newblock In \emph{Proc. Intl. Conf. Learn. Representations}, New Orleans, LA,
  May 2019.

\bibitem[D{\'e}fossez et~al.(2020)D{\'e}fossez, Bottou, Bach, and
  Usunier]{defossez2020convergence}
Alexandre D{\'e}fossez, L{\'e}on Bottou, Francis Bach, and Nicolas Usunier.
\newblock On the convergence of {Adam and Adagrad}.
\newblock \emph{arXiv preprint:2003.02395}, March 2020.

\bibitem[Duchi et~al.(2011)Duchi, Hazan, and Singer]{duchi2011adaptive}
John Duchi, Elad Hazan, and Yoram Singer.
\newblock Adaptive subgradient methods for online learning and stochastic
  optimization.
\newblock \emph{J. Machine Learning Res.}, 12\penalty0 (Jul):\penalty0
  2121--2159, 2011.

\bibitem[Ghadimi and Lan(2013)]{ghadimi2013stochastic}
Saeed Ghadimi and Guanghui Lan.
\newblock Stochastic first-and zeroth-order methods for nonconvex stochastic
  programming.
\newblock \emph{SIAM Journal on Optimization}, 23\penalty0 (4):\penalty0
  2341--2368, 2013.

\bibitem[Ghadimi and Lan(2016)]{ghadimi2016accelerated}
Saeed Ghadimi and Guanghui Lan.
\newblock Accelerated gradient methods for nonconvex nonlinear and stochastic
  programming.
\newblock \emph{Mathematical Programming}, 156\penalty0 (1-2):\penalty0 59--99,
  2016.

\bibitem[Haddadpour et~al.(2019)Haddadpour, Kamani, Mahdavi, and
  Cadambe]{haddadpour2019local}
Farzin Haddadpour, Mohammad~Mahdi Kamani, Mehrdad Mahdavi, and Viveck Cadambe.
\newblock Local sgd with periodic averaging: Tighter analysis and adaptive
  synchronization.
\newblock In \emph{Proc. Conf. Neural Info. Process. Syst.}, pages
  11080--11092, Vancouver, Canada, December 2019.

\bibitem[Horv{\'a}th et~al.(2019)Horv{\'a}th, Kovalev, Mishchenko, Stich, and
  Richt{\'a}rik]{horvath2019stochastic}
Samuel Horv{\'a}th, Dmitry Kovalev, Konstantin Mishchenko, Sebastian Stich, and
  Peter Richt{\'a}rik.
\newblock Stochastic distributed learning with gradient quantization and
  variance reduction.
\newblock \emph{arXiv preprint:1904.05115}, April 2019.

\bibitem[Ivkin et~al.(2019)Ivkin, Rothchild, Ullah, Stoica, Arora,
  et~al.]{ivkin2019communication}
Nikita Ivkin, Daniel Rothchild, Enayat Ullah, Ion Stoica, Raman Arora, et~al.
\newblock Communication-efficient distributed {SGD} with sketching.
\newblock In \emph{Proc. Conf. Neural Info. Process. Syst.}, pages
  13144--13154, Vancouver, Canada, December 2019.

\bibitem[Jaggi et~al.(2014)Jaggi, Smith, Tak{\'a}c, Terhorst, Krishnan,
  Hofmann, and Jordan]{jaggi2014}
Martin Jaggi, Virginia Smith, Martin Tak{\'a}c, Jonathan Terhorst, Sanjay
  Krishnan, Thomas Hofmann, and Michael~I Jordan.
\newblock Communication-efficient distributed dual coordinate ascent.
\newblock In \emph{Proc. Advances in Neural Info. Process. Syst.}, pages
  3068--3076, Montreal, Canada, December 2014.

\bibitem[Jiang and Agrawal(2018)]{jiang2018linear}
Peng Jiang and Gagan Agrawal.
\newblock A linear speedup analysis of distributed deep learning with sparse
  and quantized communication.
\newblock In \emph{Proc. Conf. Neural Info. Process. Syst.}, pages 2525--2536,
  Montreal, Canada, Dec 2018.

\bibitem[Johnson and Zhang(2013)]{johnson2013accelerating}
Rie Johnson and Tong Zhang.
\newblock Accelerating stochastic gradient descent using predictive variance
  reduction.
\newblock In \emph{Proc. Conf. Neural Info. Process. Syst.}, pages 315--323,
  2013.

\bibitem[Kairouz et~al.(2019)Kairouz, McMahan, Avent, Bellet, Bennis, Bhagoji,
  Bonawitz, Charles, Cormode, Cummings, et~al.]{kairouz2019advances}
Peter Kairouz, H~Brendan McMahan, Brendan Avent, Aur{\'e}lien Bellet, Mehdi
  Bennis, Arjun~Nitin Bhagoji, Keith Bonawitz, Zachary Charles, Graham Cormode,
  Rachel Cummings, et~al.
\newblock Advances and open problems in federated learning.
\newblock \emph{arXiv preprint:1912.04977}, December 2019.

\bibitem[Kamp et~al.(2018)Kamp, Adilova, Sicking, H{\"u}ger, Schlicht, Wirtz,
  and Wrobel]{kamp2018}
Michael Kamp, Linara Adilova, Joachim Sicking, Fabian H{\"u}ger, Peter
  Schlicht, Tim Wirtz, and Stefan Wrobel.
\newblock Efficient decentralized deep learning by dynamic model averaging.
\newblock In \emph{Euro. Conf. Machine Learn. Knowledge Disc. Data.,}, pages
  393--409, Dublin, Ireland, 2018.

\bibitem[Karimi et~al.(2016)Karimi, Nutini, and Schmidt]{karimi2016}
Hamed Karimi, Julie Nutini, and Mark Schmidt.
\newblock Linear convergence of gradient and proximal-gradient methods under
  the polyak-{\l}ojasiewicz condition.
\newblock In \emph{Proc. Euro. Conf. Machine Learn.}, pages 795--811, Riva del
  Garda, Italy, September 2016.

\bibitem[Karimireddy et~al.(2019)Karimireddy, Rebjock, Stich, and
  Jaggi]{karimireddy2019icml}
Sai~Praneeth Karimireddy, Quentin Rebjock, Sebastian Stich, and Martin Jaggi.
\newblock Error feedback fixes signsgd and other gradient compression schemes.
\newblock In \emph{Proc. Intl. Conf. Machine Learn.}, pages 3252--3261, Long
  Beach, CA, June 2019.

\bibitem[Karimireddy et~al.(2020)Karimireddy, Kale, Mohri, Reddi, Stich, and
  Suresh]{karimireddy2019scaffold}
Sai~Praneeth Karimireddy, Satyen Kale, Mehryar Mohri, Sashank~J Reddi,
  Sebastian~U Stich, and Ananda~Theertha Suresh.
\newblock {SCAFFOLD: S}tochastic controlled averaging for on-device federated
  learning.
\newblock In \emph{Proc. Intl. Conf. Machine Learn.}, July 2020.

\bibitem[Kingma and Ba(2014)]{kingma2014adam}
Diederik~P Kingma and Jimmy Ba.
\newblock Adam: {A} method for stochastic optimization.
\newblock \emph{arXiv preprint:1412.6980}, December 2014.

\bibitem[Li et~al.(2018)Li, Sahu, Zaheer, Sanjabi, Talwalkar, and
  Smith]{li2018federated}
Tian Li, Anit~Kumar Sahu, Manzil Zaheer, Maziar Sanjabi, Ameet Talwalkar, and
  Virginia Smith.
\newblock Federated optimization in heterogeneous networks.
\newblock \emph{arXiv preprint arXiv:1812.06127}, 2018.

\bibitem[Li and Orabona(2019)]{li2019adapt}
Xiaoyu Li and Francesco Orabona.
\newblock On the convergence of stochastic gradient descent with adaptive
  stepsizes.
\newblock In \emph{Proc. Intl. Conf. on Artif. Intell. and Stat.}, pages
  983--992, Okinawa, Japan, April 2019.

\bibitem[Lian et~al.(2016)Lian, Zhang, Hsieh, Huang, and Liu]{lian2016nips}
Xiangru Lian, Huan Zhang, Cho-Jui Hsieh, Yijun Huang, and Ji~Liu.
\newblock A comprehensive linear speedup analysis for asynchronous stochastic
  parallel optimization from zeroth-order to first-order.
\newblock In \emph{Proc. Conf. Neural Info. Process. Syst.}, pages 3054--3062,
  Barcelona, Spain, December 2016.

\bibitem[Lin et~al.(2020)Lin, Stich, Patel, and Jaggi]{lin2018don}
Tao Lin, Sebastian~U Stich, Kumar~Kshitij Patel, and Martin Jaggi.
\newblock Don't use large mini-batches, use local {SGD}.
\newblock In \emph{Proc. Intl. Conf. Learn. Representations}, Addis Ababa,
  Ethiopia, April 2020.

\bibitem[Lin et~al.(2018)Lin, Han, Mao, Wang, and Dally]{lin2017deep}
Yujun Lin, Song Han, Huizi Mao, Yu~Wang, and William~J Dally.
\newblock Deep gradient compression: Reducing the communication bandwidth for
  distributed training.
\newblock In \emph{Proc. Intl. Conf. Learn. Representations}, Vancouver,
  Canada, Apr 2018.

\bibitem[Ma et~al.(2017)Ma, Kone{\v{c}}n{\`y}, Jaggi, Smith, Jordan,
  Richt{\'a}rik, and Tak{\'a}{\v{c}}]{ma2017}
Chenxin Ma, Jakub Kone{\v{c}}n{\`y}, Martin Jaggi, Virginia Smith, Michael~I
  Jordan, Peter Richt{\'a}rik, and Martin Tak{\'a}{\v{c}}.
\newblock Distributed optimization with arbitrary local solvers.
\newblock \emph{Optimization Methods and Software}, 32\penalty0 (4):\penalty0
  813--848, July 2017.

\bibitem[McMahan et~al.(2017{\natexlab{a}})McMahan, Moore, Ramage, Hampson, and
  y~Arcas]{mcmahan2017}
Brendan McMahan, Eider Moore, Daniel Ramage, Seth Hampson, and Blaise~Aguera
  y~Arcas.
\newblock Communication-efficient learning of deep networks from decentralized
  data.
\newblock In \emph{Proc. Intl. Conf. Artificial Intell. and Stat.}, pages
  1273--1282, Fort Lauderdale, FL, April 2017{\natexlab{a}}.

\bibitem[McMahan et~al.(2017{\natexlab{b}})McMahan, Moore, Ramage, Hampson, and
  y~Arcas]{mcmahan2017communication}
Brendan McMahan, Eider Moore, Daniel Ramage, Seth Hampson, and Blaise~Aguera
  y~Arcas.
\newblock Communication-efficient learning of deep networks from decentralized
  data.
\newblock In \emph{Proc. Intl. Conf. on Artif. Intell. and Stat.}, pages
  1273--1282, Fort Lauderdale, Florida, Apr 2017{\natexlab{b}}.

\bibitem[Nesterov(1983)]{nesterov1983method}
Yurii~E Nesterov.
\newblock A method for solving the convex programming problem with convergence
  rate $o(1/k^2)$.
\newblock In \emph{Doklady AN USSR}, volume 269, pages 543--547, 1983.

\bibitem[Park et~al.(2019)Park, Samarakoon, Bennis, and
  Debbah]{park2019wireless}
Jihong Park, Sumudu Samarakoon, Mehdi Bennis, and M{\'e}rouane Debbah.
\newblock Wireless network intelligence at the edge.
\newblock \emph{Proc. of the IEEE}, 107\penalty0 (11):\penalty0 2204--2239,
  November 2019.

\bibitem[Polyak(1964)]{polyak1964}
Boris~T Polyak.
\newblock Some methods of speeding up the convergence of iteration methods.
\newblock \emph{Computational Mathematics and Mathematical Physics}, 4\penalty0
  (5):\penalty0 1--17, 1964.

\bibitem[Reddi et~al.(2018)Reddi, Kale, and Kumar]{reddi2019adam}
Sashank Reddi, Satyen Kale, and Sanjiv Kumar.
\newblock On the convergence of adam and beyond.
\newblock In \emph{Proc. Intl. Conf. Learn. Representations}, Vancouver,
  Canada, April 2018.

\bibitem[Reddi et~al.(2020)Reddi, Charles, Zaheer, Garrett, Rush,
  Kone{\v{c}}n{\`y}, Kumar, and McMahan]{reddi2020adaptive}
Sashank Reddi, Zachary Charles, Manzil Zaheer, Zachary Garrett, Keith Rush,
  Jakub Kone{\v{c}}n{\`y}, Sanjiv Kumar, and H~Brendan McMahan.
\newblock Adaptive federated optimization.
\newblock \emph{arXiv preprint:2003.00295}, March 2020.

\bibitem[Reisizadeh et~al.(2019)Reisizadeh, Taheri, Mokhtari, Hassani, and
  Pedarsani]{reisizadeh2019nips}
Amirhossein Reisizadeh, Hossein Taheri, Aryan Mokhtari, Hamed Hassani, and
  Ramtin Pedarsani.
\newblock Robust and communication-efficient collaborative learning.
\newblock In \emph{Proc. Conf. Neural Info. Process. Syst.}, pages 8386--8397,
  Vancouver, Canada, December 2019.

\bibitem[Robbins and Monro(1951)]{robbins1951}
Herbert Robbins and Sutton Monro.
\newblock A stochastic approximation method.
\newblock \emph{Annals of Mathematical Statistics}, 22\penalty0 (3):\penalty0
  400--407, September 1951.

\bibitem[Seide et~al.(2014)Seide, Fu, Droppo, Li, and Yu]{seide20141}
Frank Seide, Hao Fu, Jasha Droppo, Gang Li, and Dong Yu.
\newblock 1-bit stochastic gradient descent and its application to
  data-parallel distributed training of speech {DNN}s.
\newblock In \emph{Proc. Conf. Intl. Speech Comm. Assoc.}, Singapore, Sept
  2014.

\bibitem[Shamir et~al.(2014)Shamir, Srebro, and Zhang]{shamir2014communication}
Ohad Shamir, Nati Srebro, and Tong Zhang.
\newblock Communication-efficient distributed optimization using an approximate
  newton-type method.
\newblock In \emph{Proc. Intl. Conf. Machine Learn.}, pages 1000--1008,
  Beijing, China, Jun 2014.

\bibitem[Stich et~al.(2018)Stich, Cordonnier, and Jaggi]{stich2018nips}
Sebastian~U. Stich, Jean-Baptiste Cordonnier, and Martin Jaggi.
\newblock Sparsified {SGD} with memory.
\newblock In \emph{Proc. Conf. Neural Info. Process. Syst.}, pages 4447--4458,
  Montreal, Canada, Dec 2018.

\bibitem[Stich(2019)]{stich2019local}
Sebastian~Urban Stich.
\newblock Local {SGD} converges fast and communicates little.
\newblock In \emph{Proc. Intl. Conf. Learn. Representations}, New Orleans, LA,
  May 2019.

\bibitem[Strom(2015)]{strom2015scalable}
Nikko Strom.
\newblock Scalable distributed {DNN} training using commodity {GPU} cloud
  computing.
\newblock In \emph{Proc. Conf. Intl. Speech Comm. Assoc.}, Dresden, Germany,
  September 2015.

\bibitem[Sun et~al.(2019)Sun, Chen, Giannakis, and Yang]{sun2019}
Jun Sun, Tianyi Chen, Georgios Giannakis, and Zaiyue Yang.
\newblock Communication-efficient distributed learning via lazily aggregated
  quantized gradients.
\newblock In \emph{Proc. Conf. Neural Info. Process. Syst.}, page to appear,
  Vancouver, Canada, Dec 2019.

\bibitem[Tang et~al.(2018)Tang, Gan, Zhang, Zhang, and
  Liu]{tang2018communication}
Hanlin Tang, Shaoduo Gan, Ce~Zhang, Tong Zhang, and Ji~Liu.
\newblock Communication compression for decentralized training.
\newblock In \emph{Proc. Conf. Neural Info. Process. Syst.}, pages 7652--7662,
  Montreal, Canada, December 2018.

\bibitem[Tang et~al.(2020)Tang, Gan, Rajbhandari, Lian, Zhang, Liu, and
  He]{tang2020apmsqueeze}
Hanlin Tang, Shaoduo Gan, Samyam Rajbhandari, Xiangru Lian, Ce~Zhang, Ji~Liu,
  and Yuxiong He.
\newblock Apmsqueeze: A communication efficient adam-preconditioned momentum
  sgd algorithm.
\newblock \emph{arXiv preprint:2008.11343}, August 2020.

\bibitem[Vogels et~al.(2019)Vogels, Karimireddy, and Jaggi]{vogels2019powersgd}
Thijs Vogels, Sai~Praneeth Karimireddy, and Martin Jaggi.
\newblock Power{SGD: P}ractical low-rank gradient compression for distributed
  optimization.
\newblock In \emph{Proc. Conf. Neural Info. Process. Syst.}, pages
  14236--14245, Vancouver, Canada, December 2019.

\bibitem[Wang et~al.(2020{\natexlab{a}})Wang, Lu, Tu, and Zhang]{wang2020sadam}
Guanghui Wang, Shiyin Lu, Weiwei Tu, and Lijun Zhang.
\newblock {SAdam: A} variant of adam for strongly convex functions.
\newblock In \emph{Proc. Intl. Conf. Learn. Representations},
  2020{\natexlab{a}}.

\bibitem[Wang and Joshi(2019)]{wang2018coop}
Jianyu Wang and Gauri Joshi.
\newblock Cooperative {SGD: A} unified framework for the design and analysis of
  communication-efficient {SGD} algorithms.
\newblock In \emph{ICML Workshop on Coding Theory for Large-Scale ML}, Long
  Beach, CA, June 2019.

\bibitem[Wang et~al.(2020{\natexlab{b}})Wang, Tantia, Ballas, and
  Rabbat]{wang2020iclr}
Jianyu Wang, Vinayak Tantia, Nicolas Ballas, and Michael Rabbat.
\newblock {SlowMo: I}mproving communication-efficient distributed {SGD} with
  slow momentum.
\newblock In \emph{Proc. Intl. Conf. Learn. Representations},
  2020{\natexlab{b}}.

\bibitem[Wangni et~al.(2018)Wangni, Wang, Liu, and Zhang]{wangni2018gradient}
Jianqiao Wangni, Jialei Wang, Ji~Liu, and Tong Zhang.
\newblock Gradient sparsification for communication-efficient distributed
  optimization.
\newblock In \emph{Proc. Conf. Neural Info. Process. Syst.}, pages 1299--1309,
  Montreal, Canada, Dec 2018.

\bibitem[Ward et~al.(2019)Ward, Wu, and Bottou]{ward2019adagrad}
Rachel Ward, Xiaoxia Wu, and Leon Bottou.
\newblock Adagrad stepsizes: Sharp convergence over nonconvex landscapes.
\newblock In \emph{Proc. Intl. Conf. Machine Learn.}, pages 6677--6686, Long
  Beach, CA, June 2019.

\bibitem[Wen et~al.(2017)Wen, Xu, Yan, Wu, Wang, Chen, and Li]{wen2017terngrad}
Wei Wen, Cong Xu, Feng Yan, Chunpeng Wu, Yandan Wang, Yiran Chen, and Hai Li.
\newblock Terngrad: Ternary gradients to reduce communication in distributed
  deep learning.
\newblock In \emph{Proc. Conf. Neural Info. Process. Syst.}, pages 1509--1519,
  Long Beach, CA, Dec 2017.

\bibitem[Wu et~al.(2018)Wu, Huang, Huang, and Zhang]{wu2018error}
Jiaxiang Wu, Weidong Huang, Junzhou Huang, and Tong Zhang.
\newblock Error compensated quantized sgd and its applications to large-scale
  distributed optimization.
\newblock In \emph{Proc. Intl. Conf. Machine Learn.}, pages 5325--5333,
  Stockholm, Sweden, Jul 2018.

\bibitem[Xu et~al.(2020)Xu, Sutcher-Shepard, Xu, and Chen]{xu2020adam}
Yangyang Xu, Colin Sutcher-Shepard, Yibo Xu, and Jie Chen.
\newblock Asynchronous parallel adaptive stochastic gradient methods.
\newblock \emph{arXiv preprint:2002.09095}, February 2020.

\bibitem[Yu et~al.(2019)Yu, Jin, and Yang]{yu2019lcml}
Hao Yu, Rong Jin, and Sen Yang.
\newblock On the linear speedup analysis of communication efficient momentum
  {SGD} for distributed non-convex optimization.
\newblock In \emph{Proc. Intl. Conf. Machine Learn.}, pages 7184--7193, Long
  Beach, CA, June 2019.

\bibitem[Zeiler(2012)]{zeiler2012adadelta}
Matthew~D Zeiler.
\newblock Adadelta: an adaptive learning rate method.
\newblock \emph{arXiv preprint:1212.5701}, December 2012.

\bibitem[Zhang et~al.(2017)Zhang, Li, Kara, Alistarh, Liu, and
  Zhang]{zhang2017zipml}
Hantian Zhang, Jerry Li, Kaan Kara, Dan Alistarh, Ji~Liu, and Ce~Zhang.
\newblock Zipml: Training linear models with end-to-end low precision, and a
  little bit of deep learning.
\newblock In \emph{Proc. Intl. Conf. Machine Learn.}, pages 4035--4043, Sydney,
  Australia, Aug 2017.

\bibitem[Zhang et~al.(2015)Zhang, Choromanska, and LeCun]{zhang2015nips}
Sixin Zhang, Anna~E Choromanska, and Yann LeCun.
\newblock Deep learning with elastic averaging {SGD}.
\newblock In \emph{Proc. Conf. Neural Info. Process. Syst.}, pages 685--693,
  Montreal, Canada, Dec 2015.

\bibitem[Zhang and Lin(2015)]{zhang2015icml}
Yuchen Zhang and Xiao Lin.
\newblock {DiSCO}: {D}istributed optimization for self-concordant empirical
  loss.
\newblock In \emph{Proc. Intl. Conf. Machine Learn.}, pages 362--370, Lille,
  France, June 2015.

\end{thebibliography}
\end{document}